\documentclass[journal]{IEEEtran}
\newtheorem{theorem}{Theorem}[section]
\newtheorem{cor}[theorem]{Corollary}
\newtheorem{lem}[theorem]{Lemma}

\newtheorem{assumption}[theorem]{Assumption}
\newtheorem{problem}{Problem}
\newtheorem{definition}[theorem]{Definition}
\newtheorem{rem}[theorem]{Remark}
\newtheorem{ex}[theorem]{Example}

\usepackage{graphicx} 
\usepackage{epsfig} 
\usepackage{amsmath}
\usepackage{amssymb}
\usepackage{epstopdf}
\usepackage{cite}
\usepackage[noend,ruled,linesnumbered]{algorithm2e}
\SetKwComment{Comment}{$\triangleright$\ } {}
\usepackage{multirow}
\usepackage{rotating}
\usepackage{subfigure} 
\usepackage{color} 
\usepackage{mysymbol}
\usepackage[dvipsnames]{xcolor}
\usepackage[hyphens]{url}

\usepackage[breaklinks=true, colorlinks, bookmarks=true, citecolor=Black, urlcolor=Violet,linkcolor=Black]{hyperref}

\newcommand{\set}[1]{\left\{#1\right\}}

\usepackage{comment}
\usepackage[colorinlistoftodos,prependcaption,textwidth=1.5cm,textsize=tiny]{todonotes}
\IEEEoverridecommandlockouts                              

\begin{document}
\title{Perception-based Temporal Logic Planning \\in Uncertain Semantic Maps}
\author{Yiannis Kantaros, Samarth Kalluraya, Qi Jin, and George J. Pappas\thanks{The authors are with the School of
Engineering and Applied Sciences, GRASP Lab, University of Pennsylvania,
Philadelphia, PA, 19104, USA. $\left\{\text{kantaros,samarth3,qijin1,pappasg}\right\}$@seas.upenn.edu.  
This material is based upon work supported by the Air Force Research Laboratory (AFRL) and the Defense Advanced Research Projects Agency (DARPA) under Contract No. FA8750-18-C-0090 and the ARL Grant DCIST CRA W911NF-17-2-0181.}}
\maketitle 
\begin{abstract}
This paper addresses a multi-robot planning problem in \textcolor{black}{environments with partially unknown semantics}. The environment is assumed to have known geometric structure (e.g., walls) and to be occupied by static labeled landmarks with uncertain positions and classes. This modeling approach gives rise to an uncertain semantic map generated by semantic SLAM algorithms. Our goal is to design control policies for robots equipped with noisy perception systems so that they can accomplish collaborative tasks captured by global temporal logic specifications. \textcolor{black}{To specify missions that account for environmental and perceptual uncertainty, we employ a fragment of Linear Temporal Logic (LTL), called co-safe LTL, defined over perception-based atomic predicates modeling probabilistic satisfaction requirements.} The perception-based LTL planning problem gives rise to an optimal control problem, solved by a novel sampling-based algorithm, that generates open-loop control policies that are updated online to adapt to a continuously learned semantic map. We provide extensive experiments to demonstrate the efficiency of the proposed planning architecture. 
\end{abstract}
\IEEEpeerreviewmaketitle
   
\section{Introduction} \label{sec:Intro}
\IEEEPARstart{A}{utonomous} robot navigation is a fundamental problem that has received considerable research attention \cite{lavalle2006planning}. The basic motion planning problem consists of generating robot trajectories that reach a known desired goal region starting from an initial configuration while avoiding known obstacles. More recently, a new class of planning approaches have been developed that can handle a richer class of tasks than the classical point-to-point navigation, and can capture temporal and Boolean specifications. Such tasks can be, e.g., sequencing or coverage \cite{fainekos2005temporal}, data gathering \cite{guo2017distributed}, or persistent surveillance \cite{leahy2016persistent}, and can be modeled using formal languages, such as Linear Temporal Logic (LTL) \cite{baier2008principles,clarke1999model}. Several control synthesis methods have been proposed to design correct-by-construction controllers that satisfy temporal logic specifications assuming robots with known dynamics that operate in known environments \cite{fainekos2005hybrid,kress2007s,smith2011optimal,tumova2016multi,chen2012formal,ulusoy2014optimal,shoukry2017linear,vasile2013sampling,kantaros2018text}.
As a result, these methods cannot be applied to scenarios where the environment is initially unknown and, therefore, online re-planning may be required as environmental maps are constructed;  resulting in limited applicability. To mitigate this issue, learning-based approaches have also been proposed that consider robots with unknown dynamics operating in fully or partially unknown environments. These approaches learn policies that directly map on-board sensor readings to control commands in a trial-and-error fashion; see e.g., \cite{dorsa2014learning,li2017reinforcement,jones,hasanbeig2019reinforcement}. However, learning-based approaches, in addition to being data inefficient, are specialized to the environment and the robotic platform they were trained on. 


\begin{figure}[t]
  \centering
\includegraphics[width=1\linewidth]{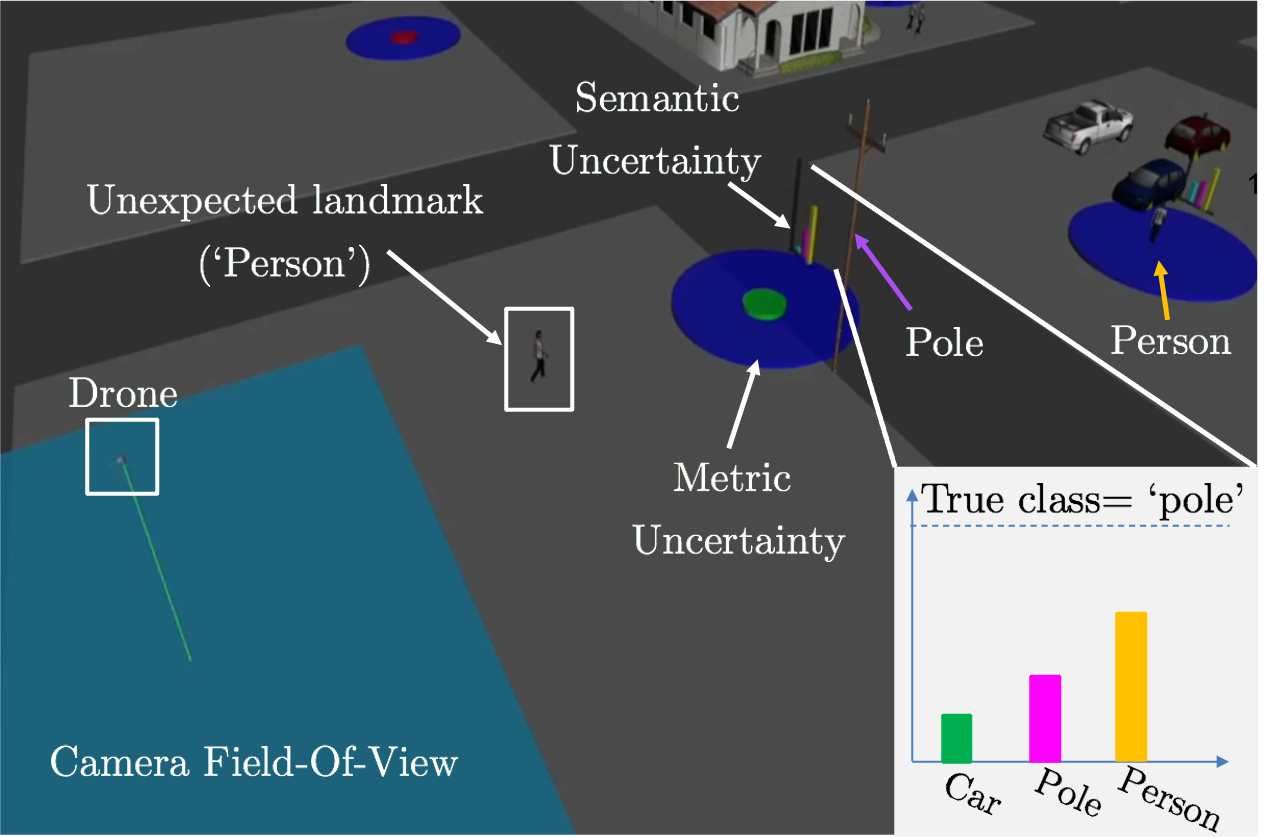}
\caption{\textcolor{black}{A drone equipped with a noisy downward facing camera (blue rectangle) and an object detector is responsible for delivering supplies to uncertain ground units (e.g., people and cars) in a specific order while always avoiding colliding with uncertain poles and flying over known residential areas. The environment is modeled as an uncertain semantic map with known geometry consisting of static labeled landmarks (e.g., poles, people, or parked cars) with uncertain positions (metric uncertainty) and labels (semantic uncertainty). Each landmark is associated with a prior Gaussian and a discrete distribution modeling its positional and class uncertainty, respectively while unknown landmarks without prior information may exist, as well.}}
  \label{fig:PF1}
\end{figure}


In this paper, we address a motion planning problem for a team of mobile sensing robots with known dynamics which are responsible for accomplishing collaborative tasks, expressed using a fragment of LTL, called co-safe LTL, in the presence of environmental and sensing uncertainty. 
Specifically, the environment is assumed to have known geometry while being occupied by static landmarks with uncertain labels/classes (\textit{semantic} uncertainty) located at uncertain positions (\textit{metric} uncertainty) giving rise to an uncertain semantic map with known geometry; see e.g., Figure \ref{fig:PF1}. The semantic map is determined by Gaussian and arbitrary discrete distributions over the positions and labels of the landmarks, respectively. The robots are equipped with noisy/imperfect perception systems (e.g., cameras and object detectors and classifiers) allowing them to take positional and class measurements about the landmarks. To define mission requirements that account for perceptual and environmental uncertainty, similar to \cite{jones2013distribution,haesaert2018temporal}, we specify temporal logic tasks defined over perception-based probabilistic predicates. Specifically, \textcolor{black}{we incorporate probabilistic satisfaction requirements} at the atomic-predicate level as opposed to probabilistic or metric/signal temporal logic languages that evaluate the probabilistic or the robust satisfaction of temporal logic formulas \cite{bianco1995model,donaldson2008monte,fainekos2009robustness,yoo2015control,lindemann2020barrier}. This way, we allow a user to put different weights on the individual components/requirements comprising an LTL formula (e.g., staying away from a hostile patrolling agent may be more important than reaching a goal destination, which cannot be captured by probabilistic temporal logic languages straightforwardly). Then, our goal is to design perception-based control policies so that the multi-robot system satisfies the co-safe LTL specification. 

The problem of synthesizing controllers in the presence of mapping or localization uncertainty typically gives rise to partially observable Markov decision processes (POMDPs) \cite{haesaert2018temporal,sun2020stochastic} that suffer from  the curse of dimensionality and history. To avoid the need for computationally expensive POMDP methods, we formulate the planning problem as a deterministic optimal control problem that designs open-loop perception-based control policies that satisfy the assigned specification. \textcolor{black}{To solve this problem, building upon existing sampling-based planners \cite{hsu1997path,sucan2011sampling,kantaros2019optimal}, we present a sampling-based approach} which, given an uncertain prior belief about the semantic map, explores the robot motion space, a (sub-)space of map uncertainty, and an automaton space corresponding to the assigned mission. During the offline control synthesis process, the robots explore the metric environmental uncertainty, i.e., predict how the metric uncertainty will change based on their future control actions under linearity and Gaussian assumptions on the sensor models \cite{atanasov2014information,atanasov2015decentralized}. Prediction of the evolution of semantic uncertainty in the offline control synthesis phase is not possible, as it requires label measurements (i.e., images) which are impossible to obtain a priori/offline.
To ensure that the proposed sampling-based approach can efficiently explore this large hybrid space, we adopt the biased sampling strategy, developed in our previous work  \cite{kantaros2019optimal}, that biases exploration towards regions that are expected to be informative and contribute to satisfaction of the assigned specification. We show that the proposed sampling-based algorithm is probabilistically complete and asymptotically optimal under Gaussian and linearity assumptions on the sensor models.

As the robots execute the designed control policies, they take measurements that allow them to update their belief about both the metric and the semantic uncertainty of the environment using recently proposed semantic Simultaneous Localization And Mapping (SLAM) algorithms \cite{bowman2017probabilistic, rosinol2020kimera}. To adapt to the online learned map, the robots design new paths by re-applying the sampling-based approach. Re-planning occurs only at time instants where the offline predicted map differs significantly from the online constructed map. Formally, the re-planning time instants are determined online by an automaton corresponding to the assigned co-safe LTL formula. The mission terminates when an accepting condition of this automaton is satisfied. The proposed planning architecture over continuously learned semantic maps is also summarized in Figure \ref{fig:overview}. 
%
Extensions of the proposed algorithm to account for nonlinear sensor models and environments with no prior metric/semantic  information are discussed, as well. 
We provide extensive experiments that corroborate the theoretical analysis and show that the proposed algorithm can address large-scale planning tasks under environmental and sensing uncertainty.
\begin{figure}[t]
  \centering
\includegraphics[width=1\linewidth]{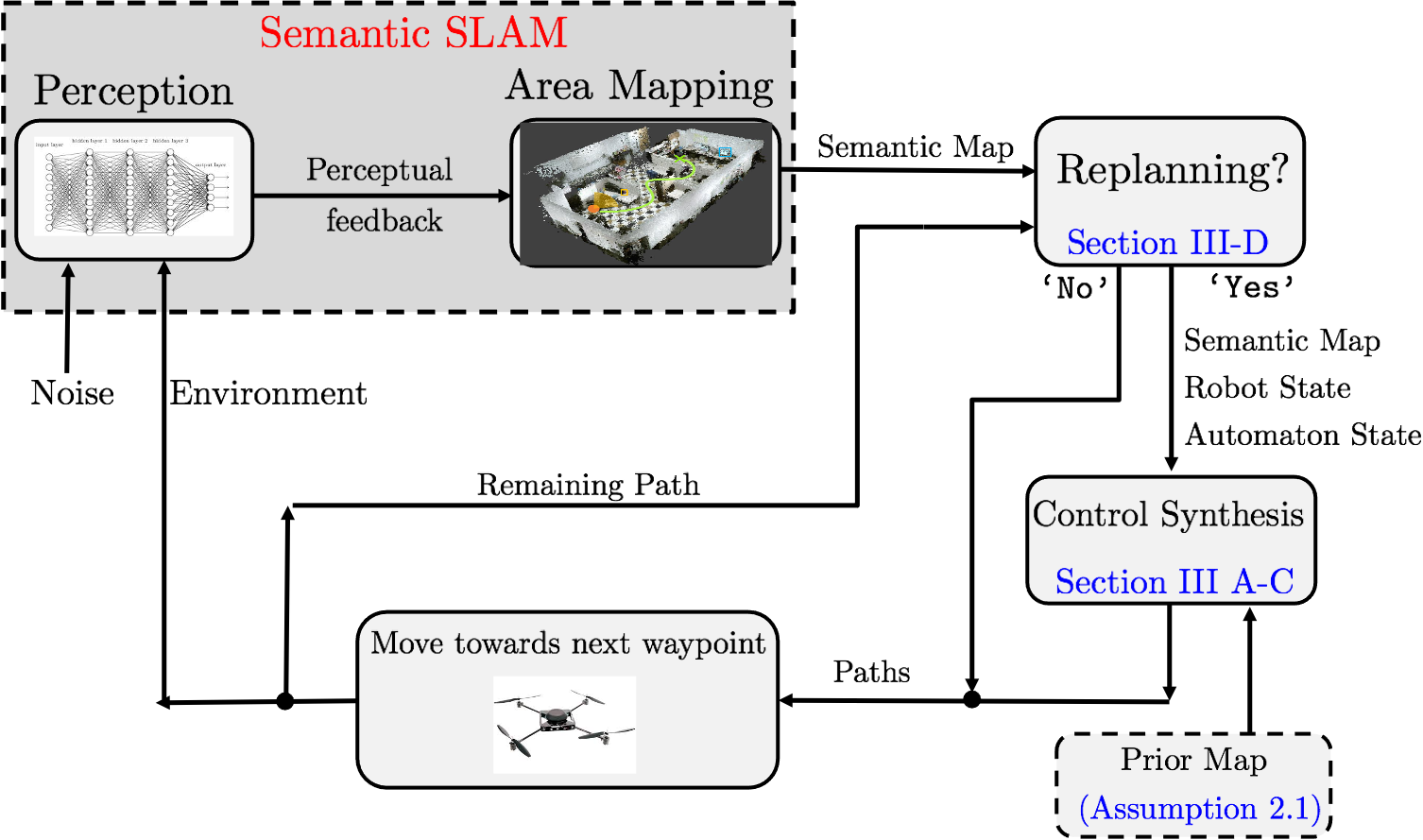}
\caption{Graphical depiction of the proposed  architecture for planning over continuously learned uncertain semantic maps. First, given a mission specification $\phi$ and an uncertain prior belief about the environment, initial paths are constructed. The robots follow the designed paths while updating the semantic map by fusing the collected measurements. Re-planning is triggered occasionally to adapt planning to the continuously learned map. } 
  \label{fig:overview}
\end{figure}
%
%
%
\subsection{Related Research}
\subsubsection{Sampling-based Temporal Logic Planning}
Sampling-based approaches for temporal logic planning have also been proposed in \cite{karaman2012sampling,vasile2013sampling,bhatia2010sampling,bhatia2011motion,luo2021abstraction,kantaros2018text} but they consider  \textit{known} environments. Typically, these works build upon sampling-based planning algorithms originally developed for reachability tasks \cite{karaman2011sampling}.
Specifically, in \cite{karaman2012sampling}, a sampling-based algorithm is proposed which incrementally builds a
Kripke structure until it is expressive enough to generate a path that satisfies a task specification expressed in deterministic $\mu$-calculus. In fact, \cite{karaman2012sampling} builds upon the RRG algorithm \cite{karaman2011sampling} to construct an expressive enough abstraction of the environment. To enhance scalability of \cite{karaman2012sampling} a sampling-based temporal logic path planning algorithm is proposed in \cite{vasile2013sampling}, that also builds upon the RRG algorithm, but constructs incrementally sparse graphs to build discrete abstractions of the workspace and the robot dynamics. These abstractions are then composed with an automaton, capturing the task specification, which yields a product automaton. Then, correct-by-construction paths are synthesized by applying graph search methods on the product automaton. However, similar to \cite{karaman2012sampling}, as the number of samples increases, the sparsity of the constructed graph is lost and, as a result, this method does not scale well to large planning problems either. 
Similar abstraction-free temporal logic planning algorithms have been presented in \cite{bhatia2010sampling,bhatia2011motion,luo2021abstraction}, as well.
A highly scalable sampling-based approach for LTL planning problems is also  presented in \cite{kantaros2018text} that can handle hundreds of robots which, unlike the previous works, requires user-specified discrete abstractions of the robot dynamics \cite{belta2005discrete}. 
Note that unlike these works, as discussed earlier, our sampling-based approach explores a map uncertainty space, as well, to account for environmental and sensing uncertainty.

\subsubsection{Temporal Logic Planning under Map Uncertainty}

Temporal logic planning in the presence of map uncertainty in terms of incomplete environment models is considered in \cite{guo2013revising,guo2015multi,maly2013iterative,lahijanian2016iterative,livingston2012backtracking,livingston2013patching,kantaros2020reactive}. 
Common in these works is that they consider static environments with unknown geometry but with known metric/semantic structure. As a result, in these works, the regions of interest over which the LTL missions are defined are a priori known. Additionally, to model these partially unknown environments, transition systems \cite{baier2008principles} or occupancy grid maps \cite{grisetti2007improved} are used that are continuously updated assuming perfect sensor feedback; as a result, LTL tasks are defined using deterministic atomic predicates. To the contrary, in this work we consider environments with known geometry but with uncertain metric and semantic structure, i.e., the locations and the labels of regions/objects of interest are uncertain. To account for this novel environmental uncertainty, we model the environment as an uncertain semantic map - which is continuously learned through uncertain perceptual
feedback - and we introduce probabilistic atomic predicates over the uncertain map to define robot missions. 
For instance, \cite{guo2013revising,guo2015multi} model the environment as transition systems which are partially known. Then discrete controllers are designed by applying graph search methods on a product automaton. As the environment, i.e., the transition system, is updated assuming perfect sensor feedback, the product automaton is locally updated, as well, and new paths are re-designed by applying graph search approaches on the revised automaton. 
Common in all these works is that, unlike our approach, they rely on discrete abstractions of the robot dynamics  \cite{belta2005discrete,pola2008approximately}. Thus, correctness of the generated motion plans is guaranteed with respect to the discrete abstractions resulting in a gap between  the generated discrete motion plans and their physical low-level implementation \cite{desaisoter}. An abstraction-free reactive method for temporal logic planning over continuously learned occupancy grid maps is proposed in \cite{kantaros2020reactive}. Note that \cite{kantaros2020reactive} can handle sensing and environmental uncertainty, as shown experimentally, but completeness guarantees are provided assuming perfect sensors, which is not the case in this paper. 

The most relevant works are presented by the authors in \cite{fu2016optimal, kantaros2019optimal} that also address temporal logic motion planning problems in  uncertain environments modeled as uncertain semantic maps. In particular, \cite{fu2016optimal} considers sequencing tasks, captured by co-safe LTL formulas, under the assumption of landmarks with uncertain locations but a priori known labels. This work is extended in \cite{kantaros2019optimal} by considering landmarks with uncertain locations and labels and task specifications captured by arbitrary co-safe temporal logic formulas. Common in both works is that control policies are designed based on a given semantic map - which is never updated - without considering the sensing capabilities of the robots. \textcolor{black}{In contrast, in this work, the proposed algorithm designs \textit{perception-based} control policies that actively decrease the metric uncertainty in the environment to satisfy probabilistic satisfaction requirements that are embedded into the mission specification. In other words, the proposed method may find feasible paths even if such paths do not exist for the initial uncertain semantic map.} Also, unlike \cite{fu2016optimal, kantaros2019optimal}, we show that the proposed algorithm can be extended to environments with no prior metric/semantic information.  Finally, we provide formal completeness and optimality guarantees that do not exist in \cite{fu2016optimal, kantaros2019optimal}. 


\subsection{Contribution}
The contribution of this paper can be summarized as follows. \textit{First}, we formulate the first perception-based LTL planning problem over uncertain semantic maps that are continuously learned by semantic SLAM methods. In fact, we propose the first formal bridge between temporal logic planning and semantic SLAM methods.
%
\textcolor{black}{\textit{Second}, we present a sampling-based approach that can synthesize perception-based control policies for robot teams that are tasked with accomplishing temporal logic missions in uncertain semantic environments.}
%
\textit{Third}, we show that the proposed sampling-based approach is probabilistically complete and asymptotically optimal under linearity and Gaussian assumptions. \textit{Fourth}, we show that the proposed algorithm can be applied to environments with no prior semantic/metric information. \textit{Fifth}, we provide extensive experiments that illustrate the efficiency of the proposed algorithm and corroborate its theoretical guarantees. 
\section{Problem Definition} \label{sec:PF}

\subsection{Multi-Robot System}
Consider $N$ mobile robots governed by the following dynamics: 
\begin{equation}\label{eq:rdynamics1}
\bbp_{j}(t+1)=\bbf_j(\bbp_{j}(t),\bbu_{j}(t)),
\end{equation}
for all $j\in\ccalN:=\{1,\dots,N\}$. In \eqref{eq:rdynamics1}, $\bbp_{j}(t)\in\mathbb{R}^n$ stands for the state (e.g., position and orientation) of robot $j$ in an environment $\Omega\subseteq\mathbb{R}^m$, $m\in\{2,3\}$, at discrete time $t$.\footnote{For simplicity of notation, hereafter, we also denote $\bbp_{j}(t)\in\Omega$.} Also, $\bbu_{j}(t) \in\ccalU_j$ stands for a control input applied at time $t$ by robot $j$, where $\ccalU_j$ denotes a \textit{finite} space of admissible control inputs. 
Hereafter, we compactly model the team of robots, each equipped with a finite library of motion primitives, as
\begin{equation}\label{eq:rdynamics}
\bbp(t+1)=\bbf(\bbp(t),\bbu(t)),
\end{equation}
where $\bbp(t)\in \Omega^N$, $\forall t\geq 0$, and $\bbu(t)\in\ccalU:=\ccalU_1\times\dots\times\ccalU_N$.
\begin{figure}[t]
  \centering
\includegraphics[width=1\linewidth]{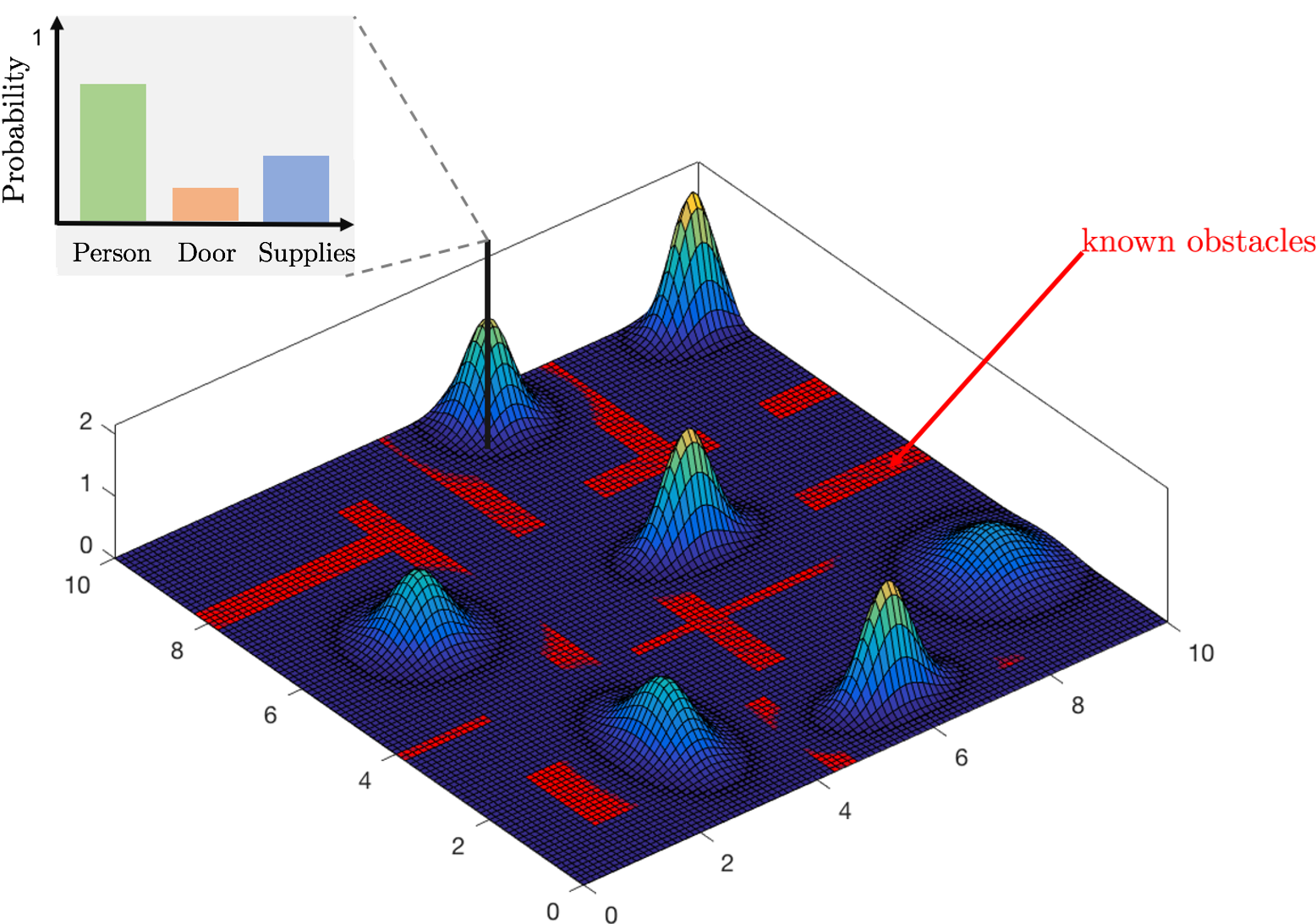}
\caption{An environment with known geometric structure modeled as a uncertain semantic map consisting of $M=7$ landmarks with a set of semantic classes defined as $\ccalC=\{\text{`person'}, \text{`door'}, \text{`supplies'}\}$. Each landmark is determined by a Gaussian distribution and a discrete distribution modeling uncertainty in its position and class, respectively. } 
  \label{fig:seMmap}
\end{figure}
\subsection{Uncertain Semantic Maps}
The robots operate in an environment with known geometric structure (e.g., walls) but with uncertain metric and semantic structure. Specifically, the obstacle-free space $\Omega_{\text{free}}\subseteq\Omega$ is assumed to be known while $\Omega_{\text{free}}$ is cluttered with $M>0$ uncertain, static, and labeled landmarks $\ell_i$ giving rise to an uncertain semantic map $\ccalM=\{\ell_1,\ell_2,\dots,\ell_M\}$. Each landmark $\ell_i=\{\bbx_i,c_i\}\in\ccalM$ is defined by its position $\bbx_i\in \Omega$ and its class $c_i\in\ccalC$, where $\ccalC$ is a finite set of semantic classes. 
The robots do not know the true landmark positions and classes but instead they have access to a probability distribution over the space of all possible maps. Such a distribution can be produced by recently proposed semantic SLAM algorithms \cite{bowman2017probabilistic,rosinol2020kimera} and typically consists of a Gaussian distribution over the landmark positions, modeling metric uncertainty in the environment, and an arbitrary discrete distribution over the landmark classes modeling semantic uncertainty. 
%
Specifically, we assume that $\ccalM$ is determined by parameters $(\hat{\bbx}_i,\Sigma_i, d_i)$, for all landmarks $\ell_i$, such that $\bbx_i\sim\ccalN(\hat{\bbx}_i,\Sigma_i)$, where $\hat{\bbx}_i$ and $\Sigma_i$ denote the mean and covariance matrix of the position of landmark $\ell_i$, and $c_i\sim d_i$, where $d_i:\ccalC\rightarrow [0,1]$ is an arbitrary discrete distribution over the finite set of possible classes. Hereafter, we compactly denote by $(\hat{\bbx},\Sigma, d)$ the parameters that determine all landmarks in the semantic map $\ccalM$. 

In what follows, we assume that an uncertain prior belief about the metric/semantic structure of the environment is available by a user to the robots which can be updated online using SLAM methods. A discussion on relaxing this assumption is provided in Section \ref{sec:extensions}.
\begin{assumption}[Uncertain Prior Map]\label{as:prior}
We assume that initially the robots have access to Gaussian and discrete distributions denoted by $\bbx_i\sim\ccalN(\hat{\bbx}_i(0),\Sigma_i(0))$ and $c_i\sim d_i(0)$, for all landmarks $\ell_i\in\ccalM$, where the number of landmarks and the set of available classes is a priori known.
\end{assumption}
\begin{ex}[Semantic Map]
An example of an uncertain semantic map is illustrated in Figure \ref{fig:seMmap} that consists of $M=7$ landmarks with set of classes $\ccalC=\{\text{`person'}, \text{`door'}, \text{`supplies'}\}$. Uncertain semantic maps provide a novel means to define mission objectives in terms of meaningful semantic objects in uncertain environments. For instance, in Figure \ref{fig:seMmap}, consider a task that requires the robots to collect all available supplies and deliver them in a specific order to injured people. 
\end{ex}

\subsection{Perceptual Capabilities}
The robots are equipped with sensors (e.g., cameras) to collect measurements associated with the landmark positions $\bbx$ as per the following \textit{linear observation model}: $\bby_j(t) = \bbM_j(\bbp_j(t))\bbx + \bbv_j(t)$, where $\bby_j(t)$ is the measurement signal at discrete time $t$ taken by robot $j\in\ccalN$. For instance, linear observation models with respect to the hidden state have been considered in \cite{freundlich2018distributed}. Also, $\bbv_j(t) \sim \ccalN(\bb0, \bbR_j)$ is sensor-state-dependent Gaussian noise with covariance $\bbR_j$. 
Hereafter, we compactly denote the observation models of all robots as 
\begin{equation}\label{eq:measModelR}
\bby(t)= \bbM(\bbp(t))\bbx+\bbv(t),~\bbv(t) \sim \ccalN(\bb0, \bbR),
\end{equation}
where $\bby(t)$ collects measurements taken at time $t$ by all robots  associated with any landmark, and $\bbM$ and $\bbR$ are diagonal matrices where their $j$-th diagonal block is $\bbM_j$ and $\bbR_j$, respectively.\footnote{Note that the observation model \eqref{eq:measModelR} is assumed to be linear with respect to the landmark positions as this allows us to employ separation principles developed in estimation theory; see Section \ref{sec:updateMap}. Nonlinear models (e.g., range sensors) are considered in Section \ref{sec:extensions}.}

Moreover, we assume that the robots are equipped with object recognition systems that allow them to collect measurements associated with the landmark classes $\bbc$. These semantic measurements typically consist of label measurements along with label probabilities \cite{redmon2016you,ren2016faster, guo2017calibration}. Since such algorithms are scale- and orientation-invariant, the class measurement is independent of the robot/sensor position (assuming that the landmark is within the sensing range of the robot). Thus, once a landmark $\ell_i$ is detected, a semantic measurement is generated by the following observation model: $[\bby_{j}^c, \bbs_{j}^c]=\bbg_j(L_i)$ (see also \cite{atanasov2016localization}) where $\bby_{j}^c$ and $\bbs_{j}^c$ represent a class/label measurement and the corresponding probability scores over all available classes, respectively, and $L_i$ stands for the true class of the detected landmark $\ell_i$. More details about modeling sensors generating semantic measurements can be found in \cite{atanasov2016localization}; see also Remark \ref{rem:objRec}. Hereafter, we compactly denote the object recognition model of all robots as  
\begin{equation}\label{eq:measModelC}
[\bby^c(t), \bbs^c(t)]=\bbg(\bbL),
\end{equation}
where $\bbL$ denote the true classes of all landmarks detected at time $t$.
%
\subsection{Separation Principle: Offline Update of Semantic Maps}\label{sec:updateMap}
Given measurements generated by the sensors \eqref{eq:measModelR}-\eqref{eq:measModelC}  until a time instant $t\geq0$, the semantic map can be updated online by fusing the collected measurements 
using recently proposed semantic SLAM algorithms \cite{bowman2017probabilistic,rosinol2020kimera}. Note that updating the map requires knowledge of the measurements that the robots may collect in the field if they apply a sequence of control actions, denoted $\bbu_{0:t}$, which is impossible to obtain during the offline control synthesis process. 

To mitigate this challenge and to design control policies while accounting for their effect on the map uncertainty, we rely on separation principles that have recently been employed in estimation theory \cite{atanasov2014information,atanasov2015decentralized}. Specifically, given a linear observation model with respect to the landmark positions, as  defined in  \eqref{eq:measModelR}, and given a sequence of control inputs until a time instant $t$, denoted by $\bbu_{0:t}$, the a-posteriori covariance matrix can be computed offline using the Kalman filter Riccati equation, i.e., without requiring knowledge of any measurements \cite{atanasov2014information,atanasov2015decentralized}. 
Hereafter, with slight abuse of notation, we denote the Kalman filter Riccati map by $\Sigma(t+1)=\rho(\Sigma(t),\bbp(t+1))$ where $\Sigma(t+1)$ and $\Sigma(t)$ denote the a posteriori covariance matrix at consecutive time instants $t$ and $t+1$. Observe that $\Sigma(t+1)$ can be computed directly from $\Sigma(t)$ given the next robot position $\bbp(t+1)$ (or equivalently the control input $\bbu(t)$ applied at $\bbp(t)$).
Note that the a posteriori estimates $\hat{\bbx}(t)$ for the positions of the landmarks cannot be updated offline, i.e., without sensor measurements generated by \eqref{eq:measModelR}. Similarly, the semantic uncertainty, captured by the discrete distribution $d$, cannot be updated offline as it requires measurements (i.e., images) that will be processed by the object recognition algorithm \eqref{eq:measModelC}.

Finally, throughout the paper we make the following assumption that is required for applying a Kalman filter and, consequently, for offline computation of a posteriori covariance matrices. 
\begin{assumption} \label{as:assumption}
The observation model \eqref{eq:measModelR} and the covariance matrix $\bbR$ of the measurement noise are known.  
\end{assumption}

Assumption \ref{as:assumption} is common in related works and, in practice, it can be satisfied by learning the covariance matrix offline using training data \cite{atanasov2014information,atanasov2015decentralized,bowman2017probabilistic}.

\subsection{Temporal Logic Tasks over Uncertain Semantic Maps}\label{sec:task}
The goal of the robots is to accomplish a complex collaborative task captured by a global co-safe Linear Temporal Logic (LTL) specification $\phi$. To reason about both environmental and perceptual uncertainty, similar to \cite{jones2013distribution,haesaert2018temporal}, we define perception-based atomic predicates that depend on both the multi-robot state and the perceived uncertain semantic map over which the LTL formula is defined. 
To define such predicates, first, with slight abuse of notation, we denote by $\ccalM(t)$ the map distribution at time $t$ obtained after the robots apply a sequence of control actions $\bbu_{0:t}$; the offline construction of $\ccalM(t)$ was discussed in Section \ref{sec:updateMap}. 

\subsubsection{Perception-based Predicates}

We define perception-based predicates, denoted by $\pi_p(\bbp(t),\ccalM(t),\Delta)$, as follows:
\begin{equation}\label{eq:pip}
  \pi_p(\bbp(t),\ccalM(t),\Delta)=
    \begin{cases}
      \text{true}, & \text{if $p(\bbp(t),\ccalM(t),\Delta)\geq 0$}\\
      \text{false}, & \text{otherwise}
    \end{cases}       
\end{equation}
where $\Delta$ is set of user-specified and case-specific parameters (e.g., probabilistic thresholds or robots indices) and $p(\bbp(t),\ccalM(t),\Delta): \Omega^N\times\ccalM(t)\times\Delta\rightarrow\mathbb{R}$. Hereafter, when it is clear from the context, we simply denote a perception-based predicate by $\pi_p$.

In what follows, we define  functions $p(\bbp(t),\ccalM(t),\Delta)$ that will be used throughout this paper. 
First, we define the following two functions reasoning about the metric uncertainty of the map:
\begin{align}
&p(\bbp(t),\ccalM(t), \{j,\ell_i,r,\delta\})=\nonumber\\&\mathbb{P}(||\bbp_j(t)-\bbx_i||\leq r)- (1-\delta),\label{ap1} 
\end{align}
and
\begin{equation}\label{apS}  
p(\bbp(t),\ccalM(t),\{j,\ell_i,\delta\})= \delta - \det(\Sigma_i(t)).
\end{equation}
 %
The predicate associated with \eqref{ap1} is true at time $t$ if the probability of robot $j$ being within distance less than $r$ from landmark $\ell_i$ is greater than $1-\delta$, after applying control actions $\bbu_{0:t}$, for some user-specified parameters $r,\delta>0$. 
Similarly, \eqref{apS} is true if the determinant of the a-posteriori covariance matrix associated with landmark $\ell_i$ is less than a user-specified threshold $\delta$ after applying control actions $\bbu_{0:t}$. 
Note that in \eqref{apS}, robot $j$ is responsible for actively localizing landmark $\ell_i$, although the global multi-robot state $\bbp(t)$ is used to compute the a posteriori covariance matrix $\Sigma(t)$ using the Kalman Filter Ricatti map. The latter predicate allows us to reason about how confident the robots are about the position of landmark $\ell_i$; the smaller the $\det(\Sigma_i(t))$ is, the smaller the metric uncertainty about landmark $i$ is. 
%

In a similar way, we can define predicates that reason about both the metric and the semantic uncertainty. For instance, consider the following function:
\begin{align}
    &p(\bbp(t),\ccalM(t),\{j,r,\delta,c\})=\nonumber\\&\max_{\ell_i}[\mathbb{P}(||\bbp_j(t)-\bbx_i||\leq r)d_i(c)]-( 1-\delta),\label{apMS}
\end{align}
In words, the predicate associated with \eqref{apMS} is true at time $t$ if the probability of robot $j$ being within distance less than $r$ from at least one landmark with class $c$ is greater than $1-\delta$. Additional predicates can be defined using e.g., the entropy of the discrete distributions or risk measures defined over the semantic map. For simplicity, in what follows, we restrict our attention on predicates of the form \eqref{ap1}-\eqref{apMS}. We can also introduce atomic predicates defined over the robots states and the known geometric structure that have been widely used in the literature (e.g., atomic predicates that are true if a robot is within a known region of interest). For simplicity of presentation, we abstain from defining such atomic propositions. Finally, given a finite set of available parameters $\Delta$, we define atomic predicates of the form \eqref{ap1}-\eqref{apMS} that are collected in a set $\mathcal{AP}$.

\subsubsection{Syntax \& Semantics}
Given a set $\mathcal{AP}$ of perception-bases atomic predicates, we construct a mission specification expressed in co-safe LTL, which is satisfied by a finite-length robot trajectory. The syntax of co-safe LTL over sequences of robot states and uncertain semantic maps is defined as follows.

\begin{definition}[Syntax]\label{syntax}
The syntax of co-safe LTL over a sequence of multi-robot states $\bbp(t)$ and uncertain semantic maps $\ccalM(t)$ is defined as $\phi::=\text{true}~|~\pi_p~|~\neg\pi_p~|~\phi_1\wedge\phi_2~|~\phi_1\vee\phi_2~|~\phi_1~\mathcal{U}~\phi_2,$
where 
(i) $\pi_p$ is a perception-based predicate defined in \eqref{eq:pip} and (ii) $\wedge$, $\vee$, $\neg$, and $\mathcal{U}$, denote the conjunction, disjunction, negation, and until operator, respectively.
\end{definition}
Note that using the `until' operator, the eventually operator, denoted by $\Diamond$, can be defined, as well \cite{baier2008principles}. 

Next, we define the semantics of co-safe LTL over sequences of multi-robot states and semantic maps. We first define a labeling function $L:\Omega^N\times\ccalM(t)\rightarrow 2^{\mathcal{AP}}$ determining which atomic propositions are true given the current robot system state $\bbp(t)\in\Omega^N$ and the current map distribution $\ccalM(t)$ obtained after applying a sequence of control actions $\bbu_{0:t}$.

\begin{definition}[Semantics]\label{semantics}
Let $\sigma=\sigma(0),\sigma(1)\dots$ be an finite sequence of multi-robot states and uncertain semantic maps, where $\sigma(t)=L([\bbp(t),\ccalM(t)])$. The semantics of temporal logic over robot states and semantic maps is defined recursively as:
\begin{align}
&\sigma\models\text{true} \nonumber\\
&\sigma\models \pi_p(\bbp,\ccalM,\Delta) \Leftrightarrow p(\underbrace{\bbp(0),\ccalM(0)}_{\text{from~}\sigma(0)},\Delta)\geq 0 \nonumber\\
&\sigma\models \phi_1\wedge\phi_2 \Leftrightarrow (\sigma\models\phi_1)\wedge(\sigma\models\phi_2) \nonumber\\
&\sigma\models \phi_1\vee\phi_2 \Leftrightarrow (\sigma\models\phi_1)\vee(\sigma\models\phi_2) \nonumber\\
&\sigma\models \phi_1\ccalU\phi_2 \Leftrightarrow \exists~t\geq 0~\text{s.t.}~\{\sigma(t)\sigma(t+1)\dots\models\phi_2\}\nonumber\\& \wedge\{\sigma(t')\sigma(t'+1)\dots\models\phi_1\},~\forall 0\leq t'<t.\nonumber
\end{align}
\end{definition}

\begin{ex}[Mission Specification]\label{ex:LTL}
Consider a single-robot mission defined over a semantic map modeled as two landmarks with available classes $\ccalC=\{\text{`Pole'},\text{`Person'}\}$. The robot has to eventually reach a person of interest to deliver supplies while always avoiding colliding with a pole; see e.g., Figure \ref{fig:PF1}. This task can be captured by the following co-safe LTL formula:
\begin{align}\label{eq:taskEx}
    \phi = &[\Diamond\pi_p(\bbp(t),\ccalM(t),\{1,r_1,\delta_1,\text{`Person'}\})]\\&\wedge[\neg\pi_p(\bbp(t),\ccalM(t),\{1,r_2,\delta_2,\text{`Pole'}\})\nonumber\\&\ccalU \pi_p(\bbp(t),\ccalM(t),\{1,r_1,\delta_1,\text{`Person'}\})]\nonumber
\end{align}
where the atomic predicates are defined as in \eqref{apMS}.
\end{ex}
\subsection{Safe Planning over Uncertain Semantic Maps}
Given a task specification $\phi$, an uncertain prior belief about the semantic map (see Assumption \ref{as:prior}), the observation model \eqref{eq:measModelR}, and the robot dynamics \eqref{eq:rdynamics}, our goal is to select a stopping horizon $H$ and a sequence $\bbu_{0:H}$ of control inputs $\bbu(t)$, for all $t\in\set{0,\dots,H}$, that satisfy $\phi$ while minimizing a user-specified motion cost function. This gives rise to the following optimal control problem:
\begin{subequations}
\label{eq:Prob2}
\begin{align}
& \min_{\substack{H, \bbu_{0:H}}} \left[J(H,\bbu_{0:H}) = \sum_{t=0}^{H}  c(\bbp(t),\bbp(t+1)) \right] \label{obj2}\\
& \ \ \ \ \ \ \  [\bbp_{0:H},\ccalM_{0:H}]\models\phi  \label{constr1b} \\
& \ \ \ \ \ \ \   \bbp(t) \in \Omega_{\text{free}}^N, \label{obsFree}\\
& \ \ \ \ \ \ \   \bbp(t+1) = \bbf(\bbp(t),\bbu(t)) \label{constr3b}\\
& \ \ \ \ \ \ \ \hat{\bbx}(t+1)=\hat{\bbx}(0) \label{constr6b}\\
&  \ \ \ \ \ \ \  \Sigma(t+1)=\rho(\Sigma(t),\bbp(t+1)) \label{constr7b}\\
&  \ \ \ \ \ \ \  d(t+1)=d(0) \label{constr8}
\end{align}
\end{subequations}
where the constraints \eqref{constr3b}-\eqref{constr8} hold for all time instants $t\in[0,H]$. In \eqref{obj2}, any motion cost function $c(\bbp(t),\bbp(t+1))$ can be used associated with the transition cost from $\bbp(t)$ to $\bbp(t+1))$ as long as it is positive (e.g., traveled distance). If non-positive summands are selected, then \eqref{eq:Prob2} is not well-defined, since the  optimal terminal horizon $H$ is infinite. The constraints \eqref{constr1b}-\eqref{obsFree} require the robots to accomplish the mission specification $\phi$ and always avoid the known obstacles/walls respectively. With slight abuse of notation, in \eqref{constr1b}, $[\bbp_{0:H},\ccalM_{0:H}]$ denotes a finite sequence of length/horizon $H$ of multi-robot states and semantic maps while $[\bbp_{0:H},\ccalM_{0:H}]\models\phi$ means that the symbols generated along this finite sequence satisfy $\phi$ (see also Definition \ref{semantics}). The constraint \eqref{constr3b} requires the robots to move according to their known dynamics \eqref{eq:rdynamics}. Also, the constraints in \eqref{constr6b}-\eqref{constr7b} require the metric uncertainty to be updated as per the Kalman filter Ricatti map, discussed in Section \ref{sec:updateMap}, while the semantic uncertainty is not updated during the control synthesis process captured by \eqref{constr8}; instead, as it will be discussed in Section \ref{sec:samplingAlg}, it is updated online as semantic measurements are collected as per \eqref{eq:measModelC}.
Then, the problem addressed in this paper can be summarized as follows:
\begin{problem} \label{prob}
Given (i) an initial robot configuration $\bbp(0)$, (ii) a prior distribution of maps $\ccalM(0)$ (see Assumption \ref{as:prior}), (iii) a linear sensor model \eqref{eq:measModelR}, and (iv) a task  captured by a co-safe formula $\phi$, compute a horizon $H$ and compute control inputs $\bbu(t)$ for all time instants $t\in\{0,\dots,H\}$ as per \eqref{eq:Prob2}.
\end{problem}

In Section \ref{sec:extensions}, we provide a discussion on relaxing the assumptions of (i) a linear observation model \eqref{eq:measModelR}; and (ii) availability of prior information about the metric/semantic environmental uncertainty discussed in Assumption \ref{as:prior}.  

\begin{rem}[Online Re-planning \& Object Recognition]\label{rem:objRec}
The object recognition method modeled in \eqref{eq:measModelC} is not required to solve \eqref{eq:Prob2} as, therein, the semantic uncertainty is not updated. In fact, \eqref{eq:Prob2} is an offline problem yielding open-loop/offline robot paths that are agnostic to the object recognition model \eqref{eq:measModelC}. A sampling-based algorithm to solve \eqref{eq:Prob2} is presented in Sections \ref{sec:sampling}-\ref{sec:constructTree}. The online execution of these paths is discussed in Section \ref{sec:replan}. Specifically, in that section, we discuss when the offline synthesized paths may need to be revised online to adapt to the semantic map that is continuously learned using perceptual feedback captured by \eqref{eq:measModelR}-\eqref{eq:measModelC}; see Figure \ref{fig:overview}, as well.
\end{rem}

\begin{rem}[Arbitrary LTL formulas]
Recall that the optimal control problem \eqref{eq:Prob2} generates a finite sequence of control actions that satisfies a co-safe LTL specification. Arbitrary LTL formulas can be considered, as well, which, however, are satisfied by infinite-length paths.
%
Nevertheless, infinite-length paths cannot be manipulated in practice. Typically, to mitigate this issue, these paths are expressed in a prefix-suffix form, where the prefix path is executed once followed by the indefinite execution of the suffix path; see e.g., \cite{kantaros2018text}. Following this approach, to account for arbitrary LTL formulas, two separate optimal control problems, which have a similar form to \eqref{eq:Prob2}, can be defined where the first one generates a (finite) prefix path while the second one generates the corresponding (finite) suffix path. Both optimal control problems can be solved by the proposed sampling-based approach.
In this paper, for simplicity, we consider only co-safe LTL formulas that are satisfied by finite-horizon paths.
\end{rem}

\section{Safe Planning In Uncertain Semantic Maps}
\label{sec:samplingAlg}
In this section, first we show that \eqref{eq:Prob2} can be converted into a reachability problem defined over a hybrid and uncertain space; see Section \ref{sec:reach}. Due to the hybrid nature of the state space that needs to be explored, \eqref{eq:Prob2} cannot be solved in a straightforward manner by existing reachability-based planning methods such as \cite{karaman2011sampling,luders2010chance}. To address this challenge, in Sections \ref{sec:sampling}-\ref{sec:constructTree}, we present a sampling-based algorithm that can efficiently explore hybrid and uncertain state-spaces and solve Problem \ref{prob}. Then in Section \ref{sec:replan} we discuss the online execution of the synthesized paths and we present a re-planning method to adapt to the continuously learned semantic map; see also Figure \ref{fig:overview}.
\subsection{Reachability in Uncertain Hybrid Spaces}\label{sec:reach}
To convert \eqref{eq:Prob2} into a reachability problem, we first convert the specification $\phi$, defined over a set of atomic predicates $\mathcal{AP}$, into a Deterministic Finite state Automaton (DFA), defined as follows \cite{baier2008principles}. 

\begin{definition}[DFA]
A Deterministic Finite state Automaton (DFA) $D$ over $\Sigma=2^{\mathcal{AP}}$ is defined as a tuple $D=\left(\ccalQ_{D}, q_{D}^0,\Sigma,\delta_D,q_F\right)$, where $\ccalQ_{D}$ is the set of states, $q_{D}^0\in\ccalQ_{D}$ is the initial state, $\Sigma$ is an alphabet, $\delta_D:\ccalQ_D\times\Sigma\rightarrow\ccalQ_D$ is a deterministic transition relation, and $q_F\in\ccalQ_{D}$ is the accepting/final state. 
\end{definition}

Given a robot trajectory $\bbp_{0:H}$ and a corresponding sequence of maps $\ccalM_{0:H}$, we get the labeled sequence $L(\bbp_{0:H}, \ccalM_{0:H}) =L([\bbp(0), \ccalM(0)])\dots L([\bbp(H), \ccalM(H)])$. This labeled sequence satisfies the specification $\phi$, if starting from the initial state $q_D^0$, each symbol/element in $L(\bbp_{0:H}, \ccalM_{0:H})$ yields a DFA transition so that eventually -after $H$ DFA transitions- the final state $q_F$ is reached \cite{baier2008principles}. 
As a result, we can equivalently re-write \eqref{eq:Prob2} as follows:
\begin{subequations}
\label{eq:ProbR}
\begin{align}
& \min_{\substack{H, \bbu_{0:H}}} \left[J(H,\bbu_{0:H}) = \sum_{t=0}^{H}  c(\bbp(t),\bbp(t+1)) \right] \label{obj4}\\
& \ \ \ \ \ \ \  q_D(t+1) = \delta_D(q_D(t),\sigma(t)),\label{constr1c}\\
& \ \ \ \ \ \ \   \bbp(t+1) = \bbf(\bbp(t),\bbu(t)) \label{constr3c}\\
& \ \ \ \ \ \ \   \bbp(t) \in \Omega_{\text{free}}^N, \label{obsFree2}\\
& \ \ \ \ \ \ \ \hat{\bbx}(t+1)=\hat{\bbx}(0) \label{constr6c}\\
&  \ \ \ \ \ \ \  \Sigma(t+1)=\rho(\Sigma(t),\bbp(t+1)) \label{constr7c}\\
&  \ \ \ \ \ \ \  d(t+1)=d(0) \label{constr8c}\\
&  \ \ \ \ \ \ \  q_D(H)=q_F \label{constr9c}
\end{align}
\end{subequations}
where $q_D(0)=q_D^0$, $\sigma(t)=L([\bbp(t),\ccalM(t)])$ and $\ccalM(t)$ is determined by $\hat{\bbx}(t)$, $\Sigma(t)$, and $d(t)$. Note that the constraint in \eqref{constr1c} captures the automaton dynamics, i.e., the next DFA state that will be reached from the current DFA state under the observation/symbol $\sigma(t)$.
In words, \eqref{eq:ProbR}, is a reachability problem defined over a joint space consisting of the automaton state-space (see \eqref{constr1c}), the robot motion space (see \eqref{constr3c}-\eqref{obsFree2}), and the metric uncertainty space (see \eqref{constr6c}-\eqref{constr8c}) while the terminal constraint requires to reach the final automaton state (see \eqref{constr9c}).


\subsection{\textcolor{black}{Sampling-based Planning in Uncertain Semantic Maps}}\label{sec:sampling}

In this section, we propose a sampling-based algorithm to solve \eqref{eq:ProbR} which incrementally builds trees that explore simultaneously the robot motion space, the  space of covariance matrices (metric uncertainty), and the automaton space that corresponds to the assigned specification. The proposed algorithm is summarized in Algorithm \ref{alg:RRT}. 

In what follows, we provide some intuition for the steps of Algorithm \ref{alg:RRT}. First, we denote the constructed tree by $\mathcal{G}=\{\mathcal{V},\mathcal{E},J_{\ccalG}\}$, where $\ccalV$ is the set of nodes and $\ccalE\subseteq \ccalV\times\ccalV$ denotes the set of edges. The set of nodes $\mathcal{V}$ contains states of the form $\bbq(t)=[\bbp(t), \ccalM(t), q_D(t)]$, where $\bbp(t)\in\Omega$ and $q_D(t)\in\ccalQ_D$.\footnote{Throughout the paper, when it is clear from the context, we drop the dependence of $\bbq(t)$ on $t$.} The function $J_{\ccalG}:\ccalV\rightarrow\mathbb{R}_{+}$ assigns the cost of reaching node $\bbq\in\mathcal{V}$ from the root of the tree. The root of the tree, denoted by $\bbq(0)$, is constructed so that it matches the initial robot state $\bbp(0)$, the initial semantic map $\ccalM(0)$, and the initial DFA state, i.e., $\bbq(0)=[\bbp(0), \ccalM(0), q_D^0]$. 
By convention the cost of the root $\bbq(0)$ is $J_{\ccalG}(\bbq(0)) = 0$, 
while the cost of a node $\bbq(t+1)\in\ccalV$, given its parent node $\bbq(t)\in\ccalV$, is computed as 
\begin{equation}\label{eq:costUpd}
J_{\ccalG}(\bbq(t+1))= J_{\ccalG}(\bbq(t)) +  c(\bbp(t),\bbp(t+1)).
\end{equation}
Observe that by applying \eqref{eq:costUpd} recursively, we get that $J_{\ccalG}(\bbq(t+1)) = J(t+1,\bbu_{0:t+1})$ which is the objective function in \eqref{eq:Prob2}.

\begin{algorithm}[t]
\caption{Mission Planning in Probabilistic Maps}
\LinesNumbered
\label{alg:RRT}
\KwIn{ (i) maximum number of iterations $n_{\text{max}}$, (ii) dynamics \eqref{eq:rdynamics}, (iii) map distribution $\ccalM(0)$, (iv) initial robot configuration $\bbp(0)$, (v) task $\phi$;}
\KwOut{Terminal horizon $H$, and control inputs $\bbu_{0:H}$}
Convert $\phi$ into a DFA\label{rrt:dfa}\;
Initialize $\ccalV = \set{\bbq(0)}$, $\ccalE = \emptyset$, $\ccalV_{1}=\set{\bbq(0)}$, $K_1=1$, and $\ccalX_g = \emptyset$\;\label{rrt:init}
\For{ $n = 1, \dots, n_{\text{max}}$}{\label{rrt:forn}
	 Sample a subset $\ccalV_{k_{\text{rand}}}$ from $f_{\ccalV}$\;\label{rrt:samplekrand}
     \For{$\bbq_{\text{rand}}(t)=[\bbp_{\text{rand}}(t),\ccalM_{\text{rand}}(t),q_D]\in\ccalV_{k_{\text{rand}}}$}{\label{rrt:forq}
     Sample a control input $\bbu_{\text{new}}\in\ccalU$ from $f_{\ccalU}$\;\label{rrt:sampleu}
     $\bbp_{\text{new}}(t+1)=\bbf(\bbp_{\text{rand}}(t),\bbu_{\text{new}})$\;\label{rrt:pnew}
     \If{$\bbp_{\text{new}}(t+1)\in\Omega_{\text{free}}^N$}{\label{rrt:obsFree}
     $\hat{\bbx}_{\text{new}}(t+1)=\hat{\bbx}(0)$\;\label{rrt:xnew}
     $\Sigma_{\text{new}}(t+1)=\rho(\Sigma_{\text{rand}}(t),\bbp_{\text{new}}(t+1))$\;\label{rrt:Sigmanew}
     $d_{\text{new}}(t+1)=d(0)$\;\label{rrt:dnew}
     Construct map: $\ccalM_{\text{new}}(t+1) = (\hat{\bbx}_{\text{new}}(t+1),\Sigma_{\text{new}}(t+1), d_{\text{new}}(t+1))$\;\label{rrt:updmap}
     Compute $q_D^{\text{new}}=\delta_D(q_D^{\text{rand}},L([\bbp_{\text{rand}}(t),\ccalM_{\text{rand}}(t)]))$\label{rrt:feasTrans}\;
     \If{$\exists q_D^{\text{new}}$}{\label{rrt:feasTrans}
     Construct $\bbq_{\text{new}}=[\bbp_{\text{new}},\ccalM_{\text{new}},q_D^{\text{new}}]$\;\label{rrt:qnew}
     Update set of nodes: $\ccalV= \ccalV\cup\{\bbq_{\text{new}}\}$\;\label{rrt:updV}
     Update set of edges: $\ccalE = \ccalE\cup\{(\bbq_{\text{rand}},\bbq_{\text{new}})\}$\;\label{rrt:updE}
     Compute cost of new state: $J_{\ccalG}(\bbq_{\text{new}})=J_{\ccalG}(\bbq_{\text{rand}})+c(\bbp_{\text{rand}},\bbp_{\text{new}})$\;\label{rrt:updCost}
     \If{$q_D^{\text{new}}=q_F$}{\label{rrt:updXg1}
     $\ccalX_g=\ccalX_g\cup\{\bbq_{\text{new}}\}$\;\label{rrt:updXg2}
     }
     Update the sets $\ccalV_{k}$\;\label{rrt:updVk}
     }}}
}
Among all nodes in $\ccalX_g$, find $\bbq_{\text{end}}(t_{\text{end}})$ \; \label{rrt:node}
$H=t_{\text{end}}$ and recover $\bbu_{0:H}$ by computing the path $\bbq_{0:t_{\text{end}}}= \bbq(0), \dots, \bbq(t_{\text{end}}) $\;\label{rrt:solution}
\end{algorithm}

\normalsize
The tree $\ccalG$ is initialized so that $\ccalV=\{\bbq(0)\}$, $\ccalE=\emptyset$, and $J_{\ccalG}(\bbq(0)) = 0$ [line \ref{rrt:init}, Alg. \ref{alg:RRT}]. Also, the tree is built incrementally by adding new states $\bbq_\text{new}$ to $\ccalV$ and corresponding edges to $\ccalE$, at every iteration $n$ of Algorithm \ref{alg:RRT}, based on a \textit{sampling} [lines \ref{rrt:samplekrand}-\ref{rrt:sampleu}, Alg. \ref{alg:RRT}] and \textit{extending-the-tree} operation [lines \ref{rrt:pnew}-\ref{rrt:updVk}, Alg. \ref{alg:RRT}]. 
After taking $n_{\text{max}}\geq 0$ samples, where $n_{\text{max}}$ is user-specified, Algorithm \ref{alg:RRT} terminates and returns a feasible solution to \eqref{eq:ProbR} (if it has been found), i.e., a terminal horizon $H$ and a sequence of control inputs $\bbu_{0:H}$. 

To extract such a solution, we need first to define the set $\ccalX_g\subseteq\ccalV$ that collects all states $\bbq(t)=[\bbp(t), \ccalM(t), q_D(t)]\in\ccalV$ of the tree that satisfy 
$q_D(t)=q_F$, which captures the terminal constraint \eqref{constr9c} [lines \ref{rrt:updXg1}-\ref{rrt:updXg2}, Alg. \ref{alg:RRT}]. Then, among all nodes $\ccalX_g$, we select the node $\bbq(t)\in\ccalX_g$, with the smallest cost $J_{\ccalG}(\bbq(t))$, denoted by $\bbq(t_{\text{end}})$ [line \ref{rrt:node}, Alg. \ref{alg:RRT}]. 
Then, the terminal horizon is $H=t_{\text{end}}$, and the control inputs $\bbu_{0:H}$ are recovered by computing the path $\bbq_{0:t_{\text{end}}}$ in $\ccalG$ that connects $\bbq(t_{\text{end}})$ to the root $\bbq(0)$, i.e., $\bbq_{0:t_{\text{end}}}= \bbq(0), \dots, \bbq(t_{\text{end}})$ [line \ref{rrt:solution}, Alg. \ref{alg:RRT}]; see also Figure \ref{fig:overview}.
Note that satisfaction of the constraints in \eqref{eq:ProbR} is guaranteed by construction of $\ccalG$; see Section \ref{sec:constructTree}. In what follows, we describe the core operations of Algorithm \ref{alg:RRT}, `\textit{sample}' and `\textit{extend}' that are used to construct the tree $\ccalG$.

\subsection{Incremental Construction of Trees}\label{sec:constructTree}
At every iteration $n$ of Algorithm \ref{alg:RRT}, a new state $\bbq_\text{new}(t+1) =[\bbp_{\text{new}}, \ccalM_{\text{new}}, , q_D^{\text{new}}]$ is sampled. The construction of the state $\bbq_\text{new}(t+1)$ relies on three steps. Specifically, first we sample a state $\bbp_{\text{new}}$; see Section \ref{sec:sample}. Second, given $\bbp_{\text{new}}$ and its parent node, we compute the corresponding new semantic map $\ccalM_{\text{new}}$. Third, we append to $\bbp_{\text{new}}$ and $\ccalM_{\text{new}}$ a DFA state $q_D^{\text{new}}$ giving rise to $\bbq_{\text{new}}(t+1)$ which is then added to the tree structure, if possible; see Section \ref{sec:extend}. 


\subsubsection{Sampling Strategy}\label{sec:sample} 
To construct the state $\bbp_{\text{new}}$, we first divide the set of nodes $\ccalV$ into a \textit{finite} number of sets, denoted by $\ccalV_{k}\subseteq\ccalV$, based on the robot state $\bbp$ and the DFA state $q_D$ that comprise the states $\bbq\in\ccalV$. Specifically, $\ccalV_{k}$ collects all states $\bbq\in\ccalV$ that share the same DFA state and the same robot state (or in practice, robot states that very close to each other). 
%
By construction of $\ccalV_{k}$, we get that $\ccalV = \cup_{k=1}^{K_n}\{\ccalV_{k}\}$, where $K_n$ is the number of subsets $\ccalV_{k}$ at iteration $n$.
Also, notice that $K_n$ is finite for all iterations $n$, due to the finite number of available control inputs $\bbu$ and the finite DFA state-space. 
At iteration $n=1$ of Algorithm \ref{alg:RRT}, it holds that $K_1=1$, $\ccalV_1=\ccalV$ [line \ref{rrt:init}, Alg. \ref{alg:RRT}]. 

Second, given the sets $\ccalV_k$, we first sample from a given discrete distribution $f_{\ccalV}(k|\ccalV):\set{1,\dots,K_n}\rightarrow[0,1]$ an index $k\in\set{1,\dots,K_n}$ that points to the set $\ccalV_{k}$ [line \ref{rrt:samplekrand}, Alg. \ref{alg:RRT}]. The mass function $f_{\ccalV}(k|\ccalV)$ defines the probability of selecting the set $\ccalV_{k}$ at iteration $n$ of Algorithm \ref{alg:RRT} given the set $\ccalV$. 


Next, given the set $\ccalV_{k_{\text{rand}}}$ sampled from $f_{\ccalV}$, we perform the following steps for all $\bbq\in \ccalV_{k_{\text{rand}}}$. Specifically, given a state $\bbq_{\text{rand}}$, we sample a control input $\bbu_{\text{new}}\in\ccalU$ from a discrete distribution $f_{\ccalU}(\bbu):\ccalU\rightarrow [0,1]$ [line \ref{rrt:sampleu}, Alg. \ref{alg:RRT}]. 
Given a control input $\bbu_{\text{new}}$ sampled from $f_{\ccalU}$, we construct the state $\bbp_{\text{new}}$ as $\bbp_{\text{new}}=\bbf(\bbp_{\text{rand}},\bbu_{\text{new}})$ [line \ref{rrt:pnew}, Alg. \ref{alg:RRT}]. If $\bbp_{\text{new}}$ belongs to the obstacle-free space, as required by \eqref{obsFree2}, then the `extend' operation follows [line \ref{rrt:obsFree}, Alg. \ref{alg:RRT}].

Any mass functions $f_{\ccalV}$ and $f_{\ccalU}$ can be used to generate $\bbp_{\text{new}}$ as long as they satisfy the following two assumptions. These assumptions will be used in Section \ref{sec:complOpt} to show completeness and optimality of the proposed sampling-based algorithm.
\begin{assumption}[Probability mass function $f_{\ccalV}$]\label{frand}
(i) The probability mass function $f_{\ccalV}(k|\ccalV):\set{1,\dots,K_n}\rightarrow[0,1]$ satisfies $f_{\ccalV}(k|\ccalV)\geq \epsilon$, $\forall$ $k\in\set{1,\dots,K_n}$ and for all $n\geq 0$, for some $\epsilon>0$ that remains constant across all iterations $n$. 
(ii) Independent samples $k_{\text{rand}}$ can be drawn from $f_{\ccalV}$.
\end{assumption}
%
\begin{assumption}[Probability mass function $f_{\ccalU}$]\label{fnew}
(i) The distribution $f_{\ccalU}(\bbu)$ satisfies $f_{\ccalU}(\bbu)\geq \zeta$, for all $\bbu\in\ccalU$, for some  $\zeta>0$ that remains constant across all iterations $n$. (ii) Independent samples $\bbu_{\text{new}}$ can be drawn from the probability mass function $f_{\ccalU}$.
\end{assumption}
\begin{rem}[Mass functions $f_{\ccalV}$ and $f_{\ccalU}$]
Assumptions \ref{frand}(i) and \ref{fnew}(i) also imply that the mass functions $f_{\ccalV}$ and $f_{\ccalU}$ are bounded away from zero on $\set{1,\dots,K_n}$ and $\ccalU$, respectively. Also, observe that 
Assumptions \ref{frand} and \ref{fnew} are very flexible, since they also allow $f_{\ccalV}$ and $f_{\ccalU}$ to change with iterations $n$ of Algorithm \ref{alg:RRT}, as the tree grows.
\end{rem}
\begin{rem}[Sampling Strategy]
An example of a distribution $f_{\ccalV}$ that satisfies Assumption \ref{frand} is the discrete uniform distribution $f_{\ccalV}(k|\ccalV )=\frac{1}{k}$, for all $k\in\set{1,\dots,K_n}$. Observe that the uniform mass function trivially satisfies Assumption \ref{frand}(ii). Also, observe that Assumption \ref{frand}(i) is also satisfied, since there exists a finite $\epsilon>0$ that satisfies Assumption \ref{frand}(i), which is $\epsilon =\frac{1}{|\ccalR|}$, where $\ccalR$ is a set that collects all robot configurations $\bbp$ that can be reached by the initial state $\bbp(0)$. Note that $\ccalR$ is a finite set, since the set $\ccalU$ of admissible control inputs is finite, by assumption. Similarly, uniform mass functions $f_{\ccalU}$ satisfy Assumption \ref{fnew}. Note that any functions $f_{\ccalV}$ and $f_{\ccalU}$ can be employed as long as they satisfy Assumptions \ref{frand} and \ref{fnew}. Nevertheless, the selection of $f_{\ccalV}$ and $f_{\ccalU}$ affects the performance of Algorithm \ref{alg:RRT}; see the numerical experiments in Section \ref{sec:Sim}. In Appendix \ref{sec:biasedSampling}, we design (nonuniform) mass functions $f_{\ccalV}$ and $f_{\ccalU}$ that bias the exploration towards locations that are expected to contribute to the satisfaction of the assigned mission specification. 
\end{rem}

\subsubsection{Extending the tree}\label{sec:extend}

To build incrementally a tree that explores both the robot motion space, the space of metric uncertainty, and the DFA state space, we need to append to $\bbp_\text{new}$ the corresponding semantic map $\ccalM_{\text{new}}$ determined by the parameters $(\hat{\bbx}_{\text{new}},\Sigma_{\text{new}},d_{\text{new}})$ and DFA state $q_D^{\text{new}}$. Particularly, $\ccalM_{\text{new}}$ is constructed so that $\hat{\bbx}_{\text{new}}=\hat{\bbx}(0)$,
as required in \eqref{constr6c}, and computing $\Sigma_{\text{new}}$ as per the Kalman filter Ricatti map, i.e.,
$\Sigma_{\text{new}}=\rho(\Sigma_{\text{rand}},\bbp_{\text{new}})$, as required in \eqref{constr7c}, and $d_{\text{new}}=d(0)$,
as required in \eqref{constr8c} [lines \ref{rrt:xnew}-\ref{rrt:updmap}, Alg. \ref{alg:RRT}]. 
Next, to construct the state $\bbq_{\text{new}}$ we append to $\bbp_\text{new}$ and $\ccalM_{\text{new}}$ the DFA state $q_D^{\text{new}}=\delta_D(q_D^{\text{rand}},L([\bbp_{\text{rand}},\ccalM_{\text{rand}}]))$,
as required by \eqref{constr1c}.
In words, $q_D^{\text{new}}$ is the automaton state that that can be reached from the parent automaton state $q_D^{\text{rand}}$ given the observation $L([\bbp_{\text{rand}},\ccalM_{\text{rand}}])$. 
If such a DFA state does not exist, then this means the observation $L([\bbp_{\text{rand}},\ccalM_{\text{rand}}])$ results in violating the LTL formula and this new sample is rejected. Otherwise, the state $\bbq_{\text{new}}=(\bbp_\text{new},\ccalM_{\text{new}},q_D^{\text{new}})$ is constructed [line \ref{rrt:qnew}, Alg. \ref{alg:RRT}] which is then added to the tree.

Given a state $\bbq_{\text{new}}$, we update the set of nodes and edges of the tree as $\ccalV = \ccalV\cup\{\bbq_{\text{new}}\}$ and $\ccalE = \ccalE\cup\{(\bbq_{\text{rand}},\bbq_{\text{new}})\}$, respectively [lines \ref{rrt:updV}-\ref{rrt:updE}, Alg. \ref{alg:RRT}]. The cost of the new node $\bbq_{\text{new}}$ is computed as in \eqref{eq:costUpd}, i.e., $J_{\ccalG}(\bbq_{\text{new}})=J_{\ccalG}(\bbq_{\text{rand}})+c(\bbp_{\text{rand}},\bbp_{\text{new}})$ [line \ref{rrt:updCost}, Alg. \ref{alg:RRT}]. Finally, the sets $\ccalV_{k}$ are updated, so that if there already exists a subset $\ccalV_k$ associated with both the DFA state  $\bbq_D^{\text{new}}$ and the robot state $\bbp_{\text{new}}$, then $\ccalV_{k}=\ccalV_{k}\cup\set{\bbq_{\text{new}}}$. Otherwise, a new set $\ccalV_{k}$ is created, i.e., $K_n = K_n +1$ and $\ccalV_{K_n}=\set{\bbq_{\text{new}}}$ [line \ref{rrt:updVk}, Alg. \ref{alg:RRT}]. Recall that this process is repeated for all states  $\bbq_{\text{rand}}\in\ccalV_{k_{\text{rand}}}$ [line \ref{rrt:forq}, Alg. \ref{alg:RRT}].

\begin{rem}[Finite Decomposition of $\ccalV$]\label{rem:finite}
In Section \ref{sec:sample}, we decomposed the set of nodes $\ccalV$ into $K_n$ finite sets, where $K_n$ remains finite as $n\to\infty$, by construction. The \textit{finite} decomposition of the set of nodes is required to ensure probabilistic completeness and optimality of the proposed sampling-based algorithm; see Section \ref{sec:complOpt}. Any alternative decomposition of $\ccalV$ that always results in a finite number of subsets $\ccalV_k$ can be employed without affecting correctness of Algorithm \ref{alg:RRT}; see also Section \ref{sec:scale}.
\end{rem}
\subsection{Online Execution and Re-planning}\label{sec:replan}
Algorithm \ref{alg:RRT} generates an open-loop sequence $\bbq_{0:H}=\bbq(0),\bbq(1),\dots,\bbq(H)$, where $\bbq(t)=[\bbp(t),\ccalM(t),q_D(t)]$, so that the resulting robot trajectory $\bbp_{0:H}$ and sequence of maps $\ccalM_{0:H}$ satisfy $\phi$. 
As the robots execute synchronously their respective paths in $\bbp_{0:H}$, synthesized by Algorithm \ref{alg:RRT}, they take measurements (see \eqref{eq:measModelR}-\eqref{eq:measModelC}) which are used to update both the metric and the semantic environmental uncertainty using a semantic SLAM method. In other words, every time the robots move towards their next waypoint, they take measurements which are then used to update (i) the a posteriori mean and covariance matrix associated with the landmark positions and (ii) the discrete distribution associated with the landmark classes yielding a new semantic map denoted by $\ccalM_{\text{online}}(t)$. \textcolor{black}{Note that the online constructed map $\ccalM_{\text{online}}(t)$ may be different from the corresponding offline map $\ccalM(t)$, as the latter is constructed by predicting only the metric environmental uncertainty by computing the a posteriori covariance matrices.}

As a result, given the new map, the previously designed paths may not be feasible and, therefore, re-planning may be required to adapt to the updated map $\ccalM_{\text{online}}(t)$; see also Figure \ref{fig:overview}. Specifically, re-planning is needed only if the observations/symbols that the robots generate online do not allow them to follow the sequence of DFA states that appears in $\bbq_{0:H}$. Formally, re-planning at time $t$ is required only if the observation $L([\bbp(t),\ccalM_{\text{online}}(t)])$ does not enable a transition from $q_D(t)$ to $q_D(t+1)$, where $q_D(t)$ and $q_D(t+1)$ are determined by $\bbq_{0:H}$ computed by Algorithm \ref{alg:RRT}. In this case, Algorithm \ref{alg:RRT} is initialized with the current robot state $\bbp(t)$, the current DFA state $q_D(t)$, and the current map $\ccalM_{\text{online}}(t)$ to generate a new path. This process is repeated until the accepting DFA state is reached. Once this happens, the LTL formula has been satisfied and the mission terminates. We note that it is possible that at time $t$, as the robots follow their paths, the LTL formula may get violated. This happens if $L([\bbp(t),\ccalM_{\text{online}}(t)])$ does not yield any DFA transition from the current DFA state $q_D(t)$ (or if it yields a transition to an absorbing DFA state, in case DFA is complete). In this scenario, least violating paths can be found as e.g., in \cite{lahijanian2016specification}.\footnote{Intuitively, violation of an LTL formula as the robots follow their paths depends on their sensing capabilities. For instance, an LTL formula may get violated during execution of the synthesized paths, if the robot sensing range is significantly smaller than then proximity requirements captured by the atomic predicates. Consider e.g., the specification in Example \ref{ex:LTL}; see also Fig. \ref{fig:PF1}. Assume that a robot has to always keep a distance of $2$m from the pole. If the sensing capabilities allow the robot to detect objects that are within $d<2$ m, then it is possible that the pole avoidance requirement may get violated, as the pole may get detected only when its distance from the robots is less than $d<2$ m.} 
%
%
\textcolor{black}{\begin{rem}[Prior Map \& Re-planning]
First, recall from Assumption \ref{as:prior} that we assume that there exists a known number of landmarks in the environment. This assumption can be relaxed due to the SLAM capability included in the proposed architecture; see Fig. \ref{fig:overview}. In fact, a prior semantic map is used to initialize Algorithm \ref{alg:RRT} and to guide the planning phase. Once the robots follow the designed paths, they may discover/sense new landmarks. Once this happens, the employed SLAM method will generate a probabilistic belief about these landmarks, which will then be included in the semantic map. Second, it is possible that the prior distribution may guide the robots towards locations in the environment where no landmarks exist. For instance, consider a case where, as per the prior distribution, a landmark of interest is expected to be in a certain room. Due to this prior distribution, Algorithm \ref{alg:RRT} will design paths towards this room. If no landmarks exist in this room, i.e., if this prior distribution is wrong, then it can be discarded from the current semantic map, so that design of mission paths is not guided by this wrong prior. For instance, a prior belief about a landmark can be considered wrong, if the robot approaches the expected position of the landmark within some distance $\eta>0$, and no landmarks are sensed. The parameter $\eta$ is user-specified and it can be associated with the sensing range of the robots. Both scenarios may trigger re-planning as discussed above to account for the new map. 
\end{rem}}
\begin{rem}[Re-planning Frequency]
The frequency with which re-planning is triggered depends on how different the offline predicted map is from the online learned map. In Section \ref{sec:Sim}, we illustrate through simulations that the more inaccurate the initial semantic map is, the more often re-planning is triggered. \textcolor{black}{A potential approach towards minimizing the re-planning frequency is to incorporate semantic map prediction methods into the offline planning process \cite{narasimhan2020seeing} or to design paths that are robust to uncertainty \cite{yel2020gp,lindemann2021robust}.} 
Also, throughout the paper, we have assumed that the geometric structure of the environment is known; note that the geometry of the environment can be learned using SLAM methods, as well \cite{grisetti2007improved}. This assumption can be relaxed by re-planning every time the robots detect new obstacles/walls that interfere with the paths designed by Algorithm \ref{alg:RRT}. Nevertheless, this may result in rather frequent re-planning, especially, in cluttered environments, increasing the online computational burden \cite{ryll2019efficient}. 
Our future work will focus on designing more efficient re-planning frameworks in the presence of both geometric and metric/semantic environmental uncertainty by e.g., locally patching the designed paths, or decomposing the global LTL task into smaller sub-tasks \cite{kantaros2020reactive}, or incorporating data-driven map prediction methods in the planning process \cite{elhafsi2020map}.
\end{rem}


\section{Discussion and Extensions}\label{sec:extensions}
In this section, we discuss how assumptions made in Section \ref{sec:PF} can be relaxed.  \textcolor{black}{Specifically, in Section \ref{sec:unexpl} we discuss how the robots should behave in case feasible - either initial or online revised - paths do not exist. This way, we can relax Assumption \ref{as:prior} requiring an informative enough prior  belief so that there exists a feasible solution to \eqref{eq:Prob2}.} Then, in  Section \ref{sec:nonlin}, we show how Algorithm \ref{alg:RRT} can account for sensor models \eqref{eq:measModelR} that may be nonlinear in the positions of landmarks (e.g., range sensors) \textcolor{black}{or subject to  noise of unknown distribution}.
\subsection{Reactive Planning in Unexplored Semantic Environments}\label{sec:unexpl}
\textcolor{black}{In this section, we discuss how the robots should behave in case feasible paths do not exist and, therefore, Alg.\ref{alg:RRT} cannot generate paths. 
%
%
In fact, initial feasible paths may not exist if a prior map is unavailable or highly uncertain; see Assumption \ref{as:prior}. The latter may hold for the online learned map when re-planning is triggered, as well. 
%
%
To mitigate this challenge, we couple Algorithm \ref{alg:RRT} with existing exploration and information gathering methods \cite{leung2012decentralized,atanasov2015decentralized, kantaros2019asymptotically}. 
The key idea is that the robots switch back and forth between an exploration and exploitation mode depending on the mission status. In particular, first the robots follow an exploration strategy aiming to spread in the environment to detect landmarks and decrease the environmental uncertainty. Once a user-specified condition about the exploration mode is satisfied, the robots switch to an exploitation mode by applying Algorithm \ref{alg:RRT} aiming to find paths satisfying the assigned mission by exploiting the semantic map constructed during exploration. If such paths still do not exist, the robots switch back to the exploration mode. This process is repeated until the robots accomplish their mission. }

In what follows, we describe an exploration strategy, adopted from \cite{leung2012decentralized,atanasov2015decentralized}, along with its interaction with Algorithm \ref{alg:RRT} in more detail. First, we define a grid $C=\{c_1,c_2,\dots,c_L\}$ over the environment, where $c_k$ corresponds to the $k$-th grid cell, and we split the grid cells into `explored' (i.e., cells that have been visited by the robots) and 'unexplored'. Then, we introduce dummy `exploration' landmarks at the frontiers of the explored grid map for which uniform discrete and Gaussian distributions are defined for their labels and positions, respectively.  Given the `exploration' and any already detected landmarks, the robots design informative open-loop/offline paths that aim to minimize the metric uncertainty of these landmarks within a pre-determined finite horizon $F>0$; more details about designing these paths can be found in  \cite{atanasov2015decentralized}. Note that the fake uncertainty in the exploration-landmark locations promises information gain to the robots motivating them to move towards the unexplored part of the environment and, therefore, detect new landmarks.
As the robots start following the designed paths while taking measurements and updating the semantic map, they may detect new landmarks. Once a  new landmark $\ell_i$ is detected by a robot $j$, a Gaussian and a discrete distribution associated with this landmark are initialized. 
Once all robots have finished the execution of their paths, they should determine whether they should keep exploring the environment or design mission-based paths over the constructed semantic map using Algorithm \ref{alg:RRT}. 
We require mission-based planning to be triggered once $K(t)\geq 0$ new landmarks are detected, where $K(t)$ is user-specified and it depends on the assigned mission. 
Also, additional requirements can be incorporated to trigger Algorithm \ref{alg:RRT} related to e.g., the map uncertainty.
Note that during the exploration phase, as the robots detect and localize landmarks, they may  satisfy certain atomic propositions enabling a transition from the current DFA state to a new DFA state; see also Remark \ref{rem:safeExpl}. Thus, when Algorithm \ref{alg:RRT} is executed, it is initialized with the DFA state that the exploration mode has ended up. 
\begin{rem}[Safe Exploration]\label{rem:safeExpl}
A particular challenge that arises in the proposed framework is that of safe exploration. In fact, as the robots explore the environment to detect and localize previously unknown landmarks they may satisfy atomic propositions that violate the assigned LTL specification. Specifically, let $q_D(t)$ be the automaton state the robots are at time $t$ and assume that the robots are currently in exploration mode. During exploration, the robots may generate a symbol that does not enable any transition originating from $q_D(t)$ resulting in violation of the LTL task. An empirical approach towards mitigating this issue is to ensure that the robots always satisfy the Boolean formula that enables them to remain in $q_D(t)$.\footnote{By construction of the DFA, there exists a Boolean formula defined over $\mathcal{AP}$ that needs to be satisfied to enable a DFA transition.} For instance, assume that the robots can remain in $q_D(t)$ if with probability greater than $0.9$, they keep a distance of at least $1$m from any landmark with class `table'. Then, we can compute $\epsilon$-confidence intervals, for some $\epsilon>0$, centered at the expected positions of detected landmarks whose most likely class is `table'; $\epsilon$-confidence intervals are then treated as virtual obstacles that should be avoided. 
If the negation operator is excluded from the syntax, then exploration is always safe as there are no observations that can violate the assigned LTL task. 
\end{rem}
\subsection{Nonlinear Sensor Models and Unknown Noise Distribution}\label{sec:nonlin}
In Section \ref{sec:PF}, we assumed that the sensor model \eqref{eq:measModelR} is linear with respect to the landmark positions $\bbx$ allowing us to compute offline the a-posteriori covariance matrix without the need of measurements. Nevertheless, this assumption may not hold in practice e.g., in case of range sensors. In this case, during the execution of Algorithm \ref{alg:RRT}, similar to \cite{atanasov2015decentralized}, we compute the a-posteriori covariance matrices $\Sigma_{\text{new}}$ using the linearized observation model about the estimated landmark positions [line \ref{rrt:Sigmanew}, Alg. \ref{alg:RRT}]. Such case studies are presented in Section \ref{sec:Sim}.

\textcolor{black}{Additionally, in Section \ref{sec:PF}, we assumed that the measurement noise follows a Gaussian distribution. This can be relaxed by leveraging recent distributionally robust planning methods \cite{summers2018distributionally,renganathan2020towards}. Specifically, we can consider measurement noise with unknown distribution assuming that it belongs to an ambiguity set with known mean and covariance matrix. Then, the Kalman filter Ricatti map can be used to update the covariance matrices offline \cite{renganathan2020towards}. A key challenge in this case is to define perception-based predicates capturing probabilistic requirements since the hidden state (e.g., the landmark positions) does not follow a known distribution. To mitigate this challenge, worst-case probabilistic requirements can be defined that are typically translated to linear constraints \cite{renganathan2020towards}.}



\section{Completeness \& Optimality }\label{sec:complOpt}
In this section, we examine, the correctness and optimality of Algorithm \ref{alg:RRT}. First, in Theorem \ref{thm:probCompl},  we show that Algorithm \ref{alg:RRT} is probabilistically complete. In other words, given an initial semantic map $\ccalM(0)$ and a sensor model as per \eqref{eq:measModelR}, if there exists a solution to \eqref{eq:Prob2}, then Algorithm \ref{alg:RRT} will eventually find it.
In Theorem \ref{thm:asOpt}, we also show that Algorithm \ref{alg:RRT} is asymptotically optimal, i.e., Algorithm \ref{alg:RRT} will eventually find the optimal solution to \eqref{eq:Prob2}, if it exists. The proofs can be found in Appendix \ref{sec:prop}.

To show completeness and optimality of Algorithm \ref{alg:RRT}, we assume that the probability mass functions $f_{\ccalV}$ and $f_{\ccalU}$ satisfy Assumptions \ref{frand} and \ref{fnew}, respectively.
\begin{theorem}[Probabilistic Completeness]\label{thm:probCompl}
If there exists a solution to \eqref{eq:Prob2}, then Algorithm \ref{alg:RRT} is probabilistically complete, i.e., the probability of finding a feasible solution, i.e., a feasible horizon $H$ and a feasible sequence of control inputs $\bbu_{0:H}$ for \eqref{eq:Prob2}, goes to $1$ as $n\to\infty$.
\end{theorem}


\begin{theorem}[Asymptotic Optimality]\label{thm:asOpt}
Assume that there exists an optimal solution to  \eqref{eq:Prob2}. Then, Algorithm \ref{alg:RRT} is asymptotically optimal, i.e., the optimal horizon $H$ and the optimal sequence of control inputs $\bbu_{0:H}$ will be found with probability $1$, as $n\to\infty$. In other words, the path generated by Algorithm \ref{alg:RRT} satisfies
$\mathbb{P}\left(\left\{\lim_{n\to\infty} J(H,\bbu_{0:H})=J^*\right\}\right)=1,$
where $J$ is the objective function of \eqref{eq:Prob2} and $J^*$ is the optimal cost.\footnote{Note that the horizon $H$ and $\bbu_{0:H}$ returned by Algorithm \ref{alg:RRT} depend on $n$. For simplicity of notation, we drop this dependence.}
\end{theorem}
%

Recall from Section \ref{sec:replan} that the synthesized paths may need to be revised online by re-running Algorithm \ref{alg:RRT}. Theorems \ref{thm:probCompl}-\ref{thm:asOpt} imply that if re-planning is required at any time $t$, a feasible and the optimal path will eventually be found, if they exist. Nevertheless, the proposed control architecture cannot guarantee recursive feasibility, i.e., existence of feasible paths at future time instants when re-planning is needed. In fact, existence of feasible paths when re-planning is triggered depends on the initialization of the corresponding optimal control problem \eqref{eq:Prob2}, i.e., the online constructed map and, consequently, on the perceptual feedback which is hard to predict a-priori.

\section{Experimental Validation} \label{sec:Sim}


\begin{figure}[t]
  \centering
  \includegraphics[width=0.75\linewidth]{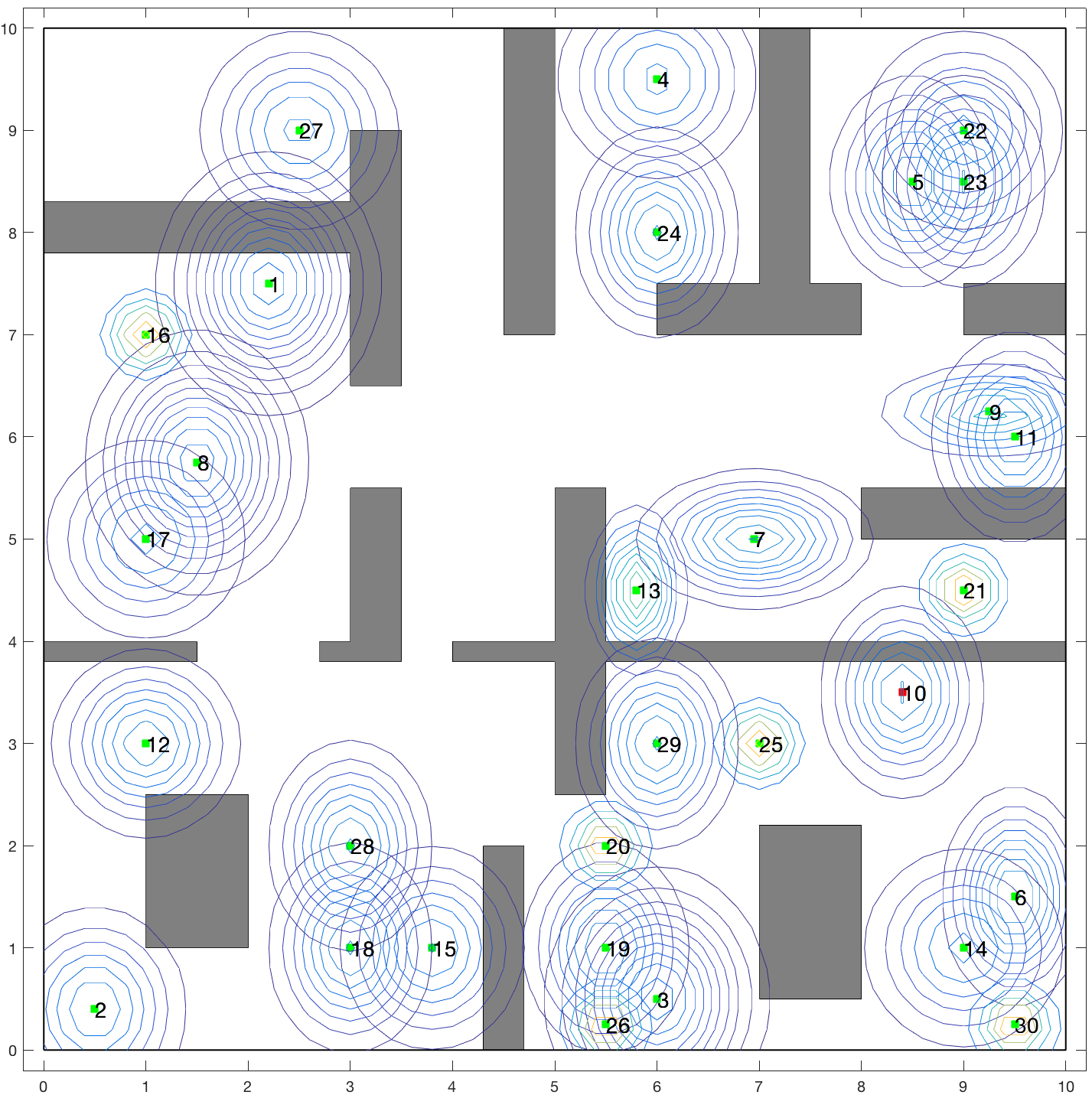}
\caption{Initial semantic map $\ccalM(0)$ with $30$ landmarks considered in Table \ref{tab:scaleSum}. The map with the $15$ landmarks considered in the same table consists of the landmarks $\ell_i$ with indices $i\in\{1,\dots,15\}$ shown in this figure.}
  \label{fig:maps_tables}
\end{figure}


In this section, we present experimental case studies that illustrate the performance of Algorithm \ref{alg:RRT} and show that it can solve large-scale planning tasks in uncertain environments. In Section \ref{sec:scale}, we demonstrate scalability of Algorithm \ref{alg:RRT} with respect to the number of landmarks and robots. In particular, given an initial semantic map and an initial multi-robot state, we report the time required by Algorithm \ref{alg:RRT} to find the first feasible solution along with the corresponding terminal horizon (without considering re-planning). \textcolor{black}{The scalability analysis has been conducted using MATLAB 2016b on a computer with Intel Core i7 3.1GHz and 16Gb RAM.}
In Section \ref{sec:replanSim}, we demonstrate in a Gazebo simulated semantic city the re-planning framework, discussed in Section \ref{sec:replan}, for various initial semantic maps. 
\textcolor{black}{These experiments have been conducted on Gazebo (ROS, python3) on a computer with Intel Core i5 - 8350U 1.7GHz and 16Gb RAM.}
Finally, in Section \ref{sec:noMap}, we show how Algorithm \ref{alg:RRT} can be applied in environments with no prior information about their metric/semantic structure, as per the discussion provided in Section \ref{sec:extensions}. The experimental videos can be found in \cite{SimSemMaps}.
%
%
%
\subsection{Scalability Analysis: Offline Planning over Prior Semantic Maps}\label{sec:scale}
In what follows, we examine the scalability performance of Algorithm \ref{alg:RRT} with respect to the number $N$ of robots, and the number $M$ of landmarks. We first describe the robot dynamics, the sensor model, the setup of the initial semantic map, the structure of the assigned LTL task, and then we conclude with the scalability analysis summarized in Tables \ref{tab:scaleSum}-\ref{tab:scale3}.
\normalsize
\subsubsection{Robot Dynamics} We consider robots with the following differential drive dynamics:

%
\footnotesize{
\begin{equation}\label{eq:nonlinRbt}
\begin{bmatrix}
   p_j^1(t+1) \\
  p_j^2(t+1)\\
  \theta_j(t+1)
  \end{bmatrix}=  
  \begin{bmatrix}
   p_j^1(t) \\
  p_j^2(t)\\
  \theta_j(t)
  \end{bmatrix}+    
   \begin{cases}
        \begin{bmatrix}
   \tau u \cos(\theta_j(t)+\tau\omega/2) \\
   \tau u \sin(\theta_j(t)+\tau\omega/2)\\
     \tau  \omega
     \end{bmatrix}, \text{if} ~\tau\omega<0.001\\
  \\
 \begin{bmatrix}
   \frac{u}{\omega} (\sin(\theta_j(t)+\tau\omega)-\sin(\theta_j(t))) \\
   \frac{u}{\omega} (\cos(\theta_j(t)) - \cos(\theta_j(t)+\tau\omega)) \\
  \tau \omega
  \end{bmatrix}, \text{else},
    \end{cases}      
\end{equation}}
\normalsize
where the robot state $\bbp_j(t)=[p_j^1(t),p_j^2(t),\theta_j]^T$ captures both the position $[p_j^1(t),p_j^2(t)]$ and the orientation $\theta_j(t)$ of robot $j$ and $\tau$ is the sampling period. The available motion primitives are $u\in\{0,1\}\text{m/s}$ and $\omega\in\set{0,\pm 1,\pm 2,\dots, \pm 180}\text{degree/s}$.

\subsubsection{Sensor Model \eqref{eq:measModelR}} All robots are equipped with omnidirectional, range-only, line-of-sight sensors with limited range of $1$m. Every robot can take noisy range measurements, i.e., measure its distance from all landmarks that lie within its sight and range. Specifically, the measurement associated with robot $j$ and landmark $i$ is given by 
\begin{equation}\label{eq:meassim}
y_{j,i} = \ell_{j,i}(t) + v(t) \ ~\mbox{if}~ \ (\ell_{j,i}(t) \leq R_j) \wedge (i\in\texttt{FOV}_j).
\end{equation}
where $\ell_{j,i}(t)$ is the distance between landmark $\ell_i$ and robot $j$, $R_j$ denotes the sensing range of robot $j$ which is equal to $1$m, $\texttt{FOV}_j$ denotes its corresponding field-of-view, and $v(t) \sim \ccalN(0,\sigma(\bbx_i(t),\bbp_j(t)))$ is the measurement noise. 
Also, we model the measurement noise so that $\sigma$ increases linearly with $\ell_{j,i}(t)$, with slope $0.5$, as long as $\ell_{j,i}(t)\leq R_j$; if $\ell_{j,i}(t)> R_j$, then $\sigma$ is infinite. This model captures the fact that range readings become more inaccurate as the distance increases \cite{schlotfeldt2018anytime}; nevertheless, any other model can be used, such as uniform noise within the sensing range. Observe that the considered observation model is nonlinear and, therefore, the a-posteriori covariance matrix, cannot be estimated optimally offline; see \eqref{constr7c}.
In this case, at every iteration of Algorithm \ref{alg:RRT},  we approximately compute the a-posteriori covariance matrices $\Sigma_{\text{new}}$ using the linearized observation model about the predicted positions of the landmarks [line \ref{rrt:Sigmanew}, Alg. \ref{alg:RRT}].\footnote{We avoid defining an object recognition method as per \eqref{eq:measModelC}, as the main focus of this section is to demonstrate scalability of Algorithm \ref{alg:RRT} in terms of solving \eqref{eq:Prob2} without considering re-planning. Recall that \eqref{eq:Prob2} does not depend on the object recognition method. In fact, the latter is needed only for learning online the semantic map and accordingly re-planning; see Section \ref{sec:replan}.}

\subsubsection{Sampling strategy} Hereafter, we employ the mass functions $f_{\ccalV}$ and $f_{\ccalU}$ designed in Appendix \ref{sec:biasedSampling} with $p_{\text{rand}}=0.9$ and $p_{\ccalU}=0.9$. Note that uniform sampling functions failed to solve the considered planning problems within three hours due to the large joint state space that needs to be explored. In fact, using uniform sampling strategies, Algorithm \ref{alg:RRT} was able to solve only small-scale problems e.g., single-agent planning problems in $4\times4$ environments with $3$ landmarks.
Also, notice that the computational complexity of Algorithm \ref{alg:RRT} per iteration depends linearly on the size of the selected subset $\ccalV_{k_{\text{rand}}}$ (see [line \ref{rrt:forq}, Alg. \ref{alg:RRT}]). Therefore, to speed up the construction of a feasible solution, a new set $\ccalV_k$ is constructed for every new sample, for the first $100$ iterations/samples (i.e., the first $100$ sets $\ccalV_k$ are singleton). After that, the sets $\ccalV_k$ are constructed as described in Section \ref{sec:samplingAlg}. This way the number of subsets $\ccalV_k$ remains finite as Algorithm \ref{alg:RRT}; see also Remark \ref{rem:finite}.

\subsubsection{Initial Semantic Map} 
We have evaluated the performance of Algorithm \ref{alg:RRT} on two different initial semantic maps consisting of $M=15$ and $M=30$ landmarks illustrated in Figure \ref{fig:maps_tables}. Particularly, this figure shows the Gaussian distributions determining the initial semantic map $\ccalM(0)$ along with the known geometric structure. We assume that the set of classes is $\ccalC=\set{1,2,\dots,M}$,
while the true class of each landmark is uncertain. In fact, the robots have access to a discrete distribution $d(0)$ over the labels that is designed by randomly selecting a probability within the range $[0.8,1]$ that a landmark has the correct label, and by randomly selecting probabilities for the remaining classes.

\subsubsection{Mission Specification} The robots are responsible for accomplishing a search-and-rescue task captured by the following LTL formula: $\phi=\Diamond\xi_{1}\wedge\Diamond\xi_{2}\wedge\neg \xi_1 \ccalU \xi_3\wedge\Diamond(\xi_{4}\wedge\Diamond(\xi_{5}\wedge(\Diamond \xi_{6})))\wedge\Diamond\xi_{7}$
where $\xi_i$ is a Boolean formula. Specifically, to define $\xi_i$, we decompose the robot team into smaller sub-groups where each sub-group is associated with a specific landmark. Each $\xi_i$ is associated with a finite number of pairs of robot sub-groups and an assigned landmark. We collect such pairs in a set denoted by $\ccalP_i$.
The Boolean formula $\xi_i$ is true if the following is true for all pairs in $\ccalP_i$: given a pair in $\ccalP_i$, all robots within a sub-group should be within $r=0.2$ meters from a specific landmark $\ell_k$ with probability greater than $1-\delta=0.75$ (see \eqref{ap1}) while the uncertainty about the position of this landmark is less than $\delta=0.01$ (see \eqref{apS}). 
%
%
%
In words, the specification $\phi$ requires (i) sub-groups of robots to eventually satisfy $\xi_1$ (see $\Diamond\xi_1$); (ii) sub-task (i) to never happen until certain sub-groups satisfy $\xi_3$ (see $(\neg \xi_1 \ccalU \xi_3)$); (iii) robot sub-groups to eventually satisfy $\xi_2$ (see $\Diamond\xi_2$); (iii) robot sub-groups to eventually satisfy $\xi_7$ (see ($\Diamond\xi_7$)); (iv) robot sub-groups  to eventually satisfy $\xi_4$, $\xi_5$, and $\xi_6$ in this order (see $\Diamond(\xi_{4}\wedge\Diamond(\xi_{5}\wedge(\Diamond \xi_{6})))$); 
The sub-tasks (i)-(iv) should be satisfied according to the probabilistic satisfaction requirements discussed above. 

Observe that the LTL task $\phi$ does not capture the requirement for obstacle avoidance. Recall that obstacle avoidance is considered during the offline path design in [line \ref{rrt:obsFree}, Alg. \ref{alg:RRT}]. Additional safety constraints can be added to line \ref{rrt:obsFree}. For instance, consider a safety property that requires all robots to always avoid approaching landmark with class $10$ within distance of $r=1$m with probability greater than $1-\delta=0.8$. This requirement can be captured by the following Boolean formula $\xi_{\text{safety}}=\wedge_{j\in\ccalN}\neg\pi_p(\bbp(t),\ccalM(t),\{j, 1, 0.2, `10'\})$, where the atomic predicate $\pi_p$ is defined as per \eqref{apMS}. Given such a safety property, every time a new sample $\bbp_{\text{new}}$ is generated, we can compute the corresponding semantic map $\ccalM_{\text{new}}$, and check if the current pair of the multi-robot state $\bbp_{\text{new}}$  and the semantic map $\ccalM_{\text{new}}$ satisfies the Boolean formula $\xi_{\text{safety}}$. If it does, the state $\bbq_{\text{new}}$ can be constructed as discussed in Algorithm \ref{alg:RRT}. Otherwise, the current sample is discarded. In this case study, we execute Alg. \ref{alg:RRT} for the specification $\phi$ along with the additional safety constraint $\xi_{\text{safety}}$ defined above.
\subsubsection{Scalability Results}
In what follows, we examine the performance of Algorithm \ref{alg:RRT} with respect to the number $N$ of robots, and the number $M$ of  landmarks along with the effect of robot velocities (parameter $\tau$ in \eqref{eq:nonlinRbt}). 
In this scalability analysis, for a fixed number $M$ of landmarks, as we increase the number of robots, we associate each $\xi_i$ in $\phi$ with more pairs of robot sub-teams and landmarks. 
The same holds as we increase the number of landmarks given a fixed number $N$ of robots, as well. Also, we note that the mission specification does not necessarily involve all landmarks that may exist in the environment. The results are summarized in Table \ref{tab:scaleSum} for $\tau=0.1$ and $\tau=0.5$ seconds. 
%
First, using available online toolboxes \cite{duret2004spot}, we convert $\phi$ into a deterministic automaton that has $48$ states and $540$ transitions. The pruning process finished in less than a second for all cases studies considered in Table \ref{tab:scaleSum}, which e.g., for the case study $N=100, M=30$, yielded a pruned automaton with $172$ edges.
Observe in Table \ref{tab:scaleSum}, that as the robot velocities increase, the terminal horizon and the runtime of Algorithm \ref{alg:RRT} decrease. Also, observe in Table \ref{tab:scale3} that for a fixed number of number of robots, as the number of landmarks increase, the total runtime of Algorithm \ref{alg:RRT} increases, as the robots have to adapt planning to a `larger' map.
Notice that Algorithm \ref{alg:RRT} can solve large-scale tasks that may involve hundreds of robots in a reasonable time due to the biased sampling process. Recall that the employed biased sampling strategy requires computing the geodesic distance between the position of robot $j$ to the expected position of a landmark, for all robots $j$. In our implementation, to compute the geodesic distances, we first discretize the workspace and, then, we use the Dijkstra algorithm to compute the corresponding shortest paths sequentially across the robots. As a result, the runtimes can significantly decrease if the geodesic distances are computed in parallel.


\begin{table}[]
\caption{Scalability Analysis - $\tau = 0.1$}
\label{tab:scaleSum}
\centering
\begin{tabular}{|l|l|l|l|l|l|}
\hline
\multirow{2}{*}{N} & \multirow{2}{*}{M} & \multicolumn{2}{c|}{\begin{tabular}[c]{@{}c@{}}Matlab Runtimes (mins)\\ (First feasible path)\end{tabular}}                                                                                          & \multicolumn{2}{c|}{H}                            \\ \cline{3-6} 
                   &                    & \multicolumn{1}{c|}{\begin{tabular}[c]{@{}c@{}}$\tau=0.1$\\ \end{tabular}} & \multicolumn{1}{c|}{\begin{tabular}[c]{@{}c@{}}$\tau=0.5$\\ 
                   \end{tabular}} & $\tau=0.1$ & $\tau=0.5$ \\ \hline
1                  & 15                 & 0.66                                                                                          & 0.12                                                                                          & 461                     & 105                     \\ \hline
1                  & 30                 & 0.88                                                                                          & 0.21                                                                                          & 465                     & 107                     \\ \hline
5                  & 15                 & 1.94                                                                                          & 0.48                                                                                          & 389                     & 94                      \\ \hline
5                  & 30                 & 2.68                                                                                          & 0.58                                                                                          & 389                     & 102                     \\ \hline
20                 & 15                 & 8.74                                                                                          & 2.33                                                                                          & 561                     & 132                     \\ \hline
20                 & 30                 & 10.32                                                                                         & 2.65                                                                                          & 565                     & 121                     \\ \hline
40                 & 15                 & 14.68                                                                                         & 7.77                                                                                          & 457                     & 151                     \\ \hline
40                 & 30                 & 16.43                                                                                         & 7.81                                                                                          & 456                     & 151                     \\ \hline
100                & 15                 & 33.32                                                                                         & 15.41                                                                                         & 533                     & 147                     \\ \hline
100                & 30                 & 38.42                                                                                         & 20.56                                                                                         & 543                     & 146                     \\ \hline
\end{tabular}
\end{table}



\begin{table}[t]
\caption{Scalability Analysis - $\tau = 0.5$}
\label{tab:scale3}
\centering
\begin{tabular}{|l|l|l|l|}
\hline
N  & M  & MatLab Runtime & H   \\ \hline
10 & 10 & 0.94 mins                 & 106 \\ \hline
10 & 45 & 2.98 mins                 & 146 \\ \hline
10 & 60 & 3.94 mins                 & 121 \\ \hline
10 & 75 & 5.41 mins                 & 133 \\ \hline
\end{tabular}
\end{table}

\begin{figure}[t]
 \centering
  \includegraphics[width=0.9\linewidth]{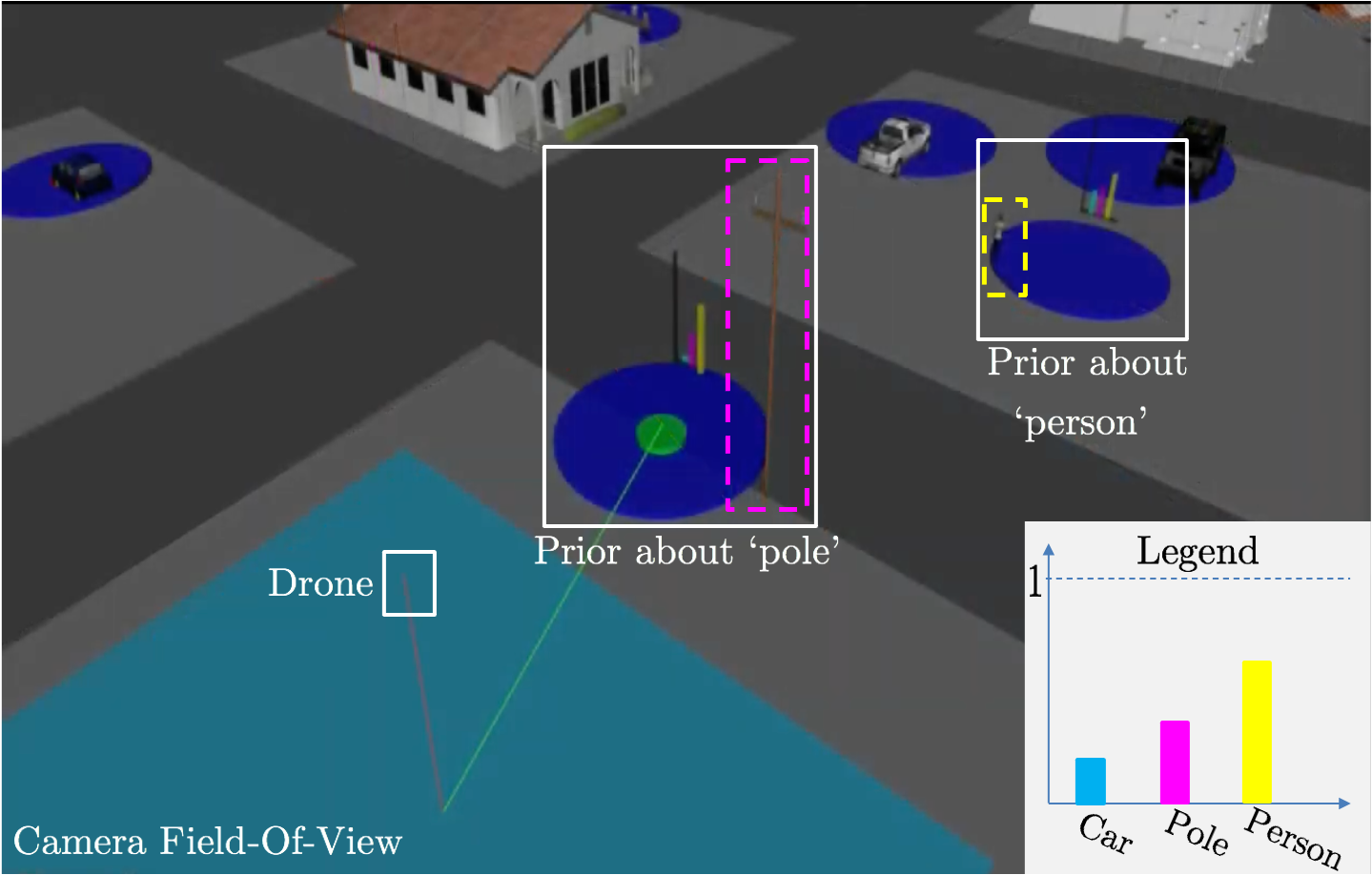}
\caption{A drone is responsible for delivering supplies to a landmark with class `person' while always avoiding hitting landmarks with class `pole' as per \eqref{eq:taskEx}. The prior information about two landmarks is also illustrated. Observe that the prior discrete distribution for the `pole' is inaccurate in the sense that, initially, the most likely class of this landmark is `person'. }
  \label{fig:EnvDrone1}
\end{figure}

\begin{figure}[t]
  \centering
%
  \subfigure[Heading towards the `pole']{
    \label{fig:2Pole}
  \includegraphics[width=0.47\linewidth]{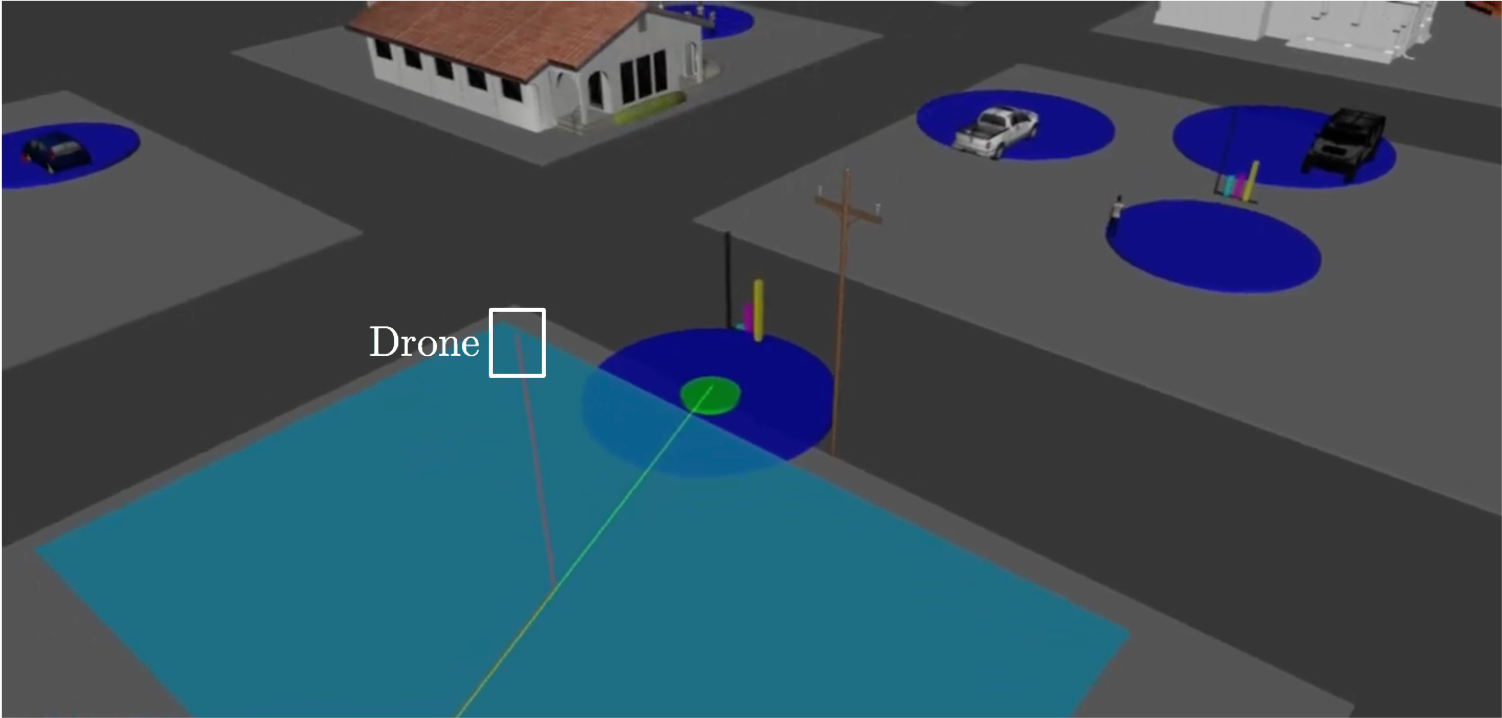}
  }
      \subfigure[Re-planning ]{
    \label{fig:t2drone1}
  \includegraphics[width=0.47\linewidth]{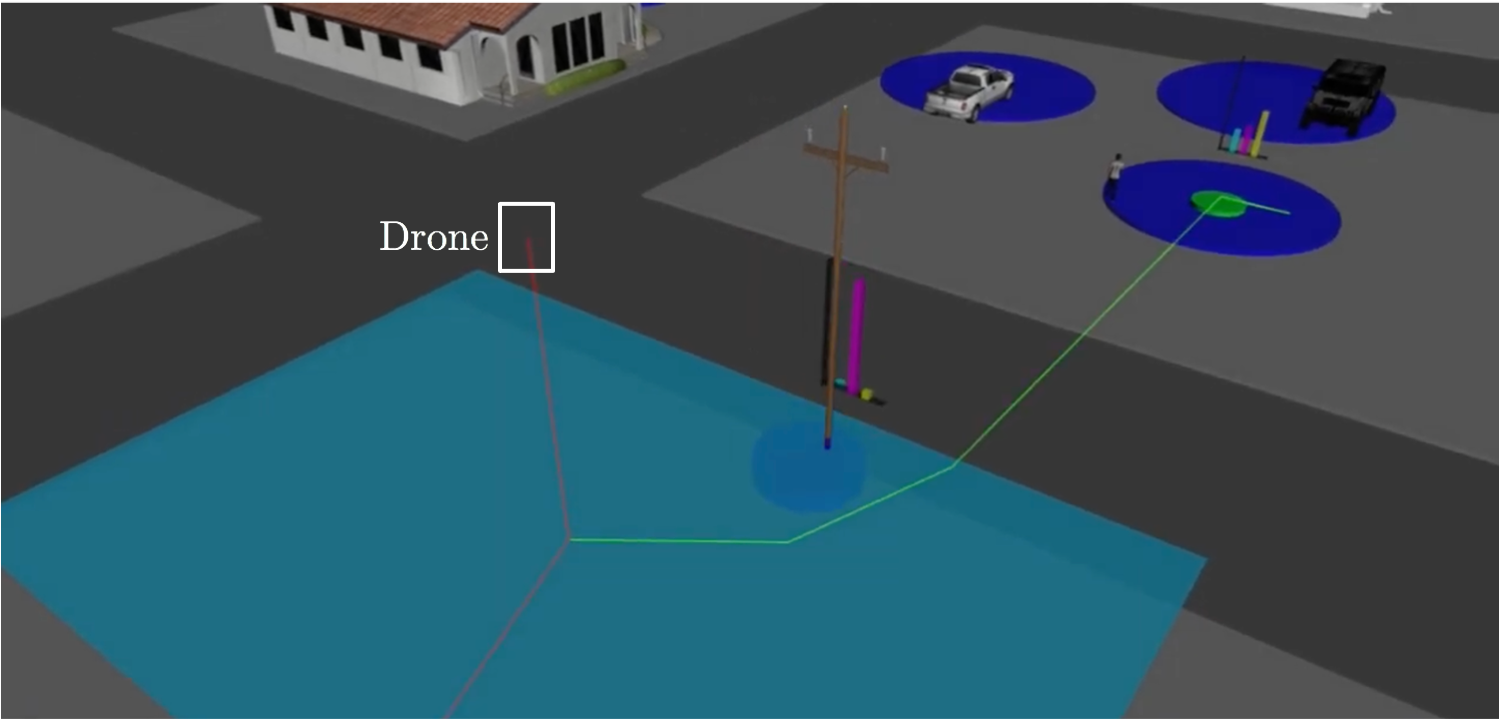}}
        \subfigure[Re-planning ]{
    \label{fig:t3drone1}
  \includegraphics[width=0.47\linewidth]{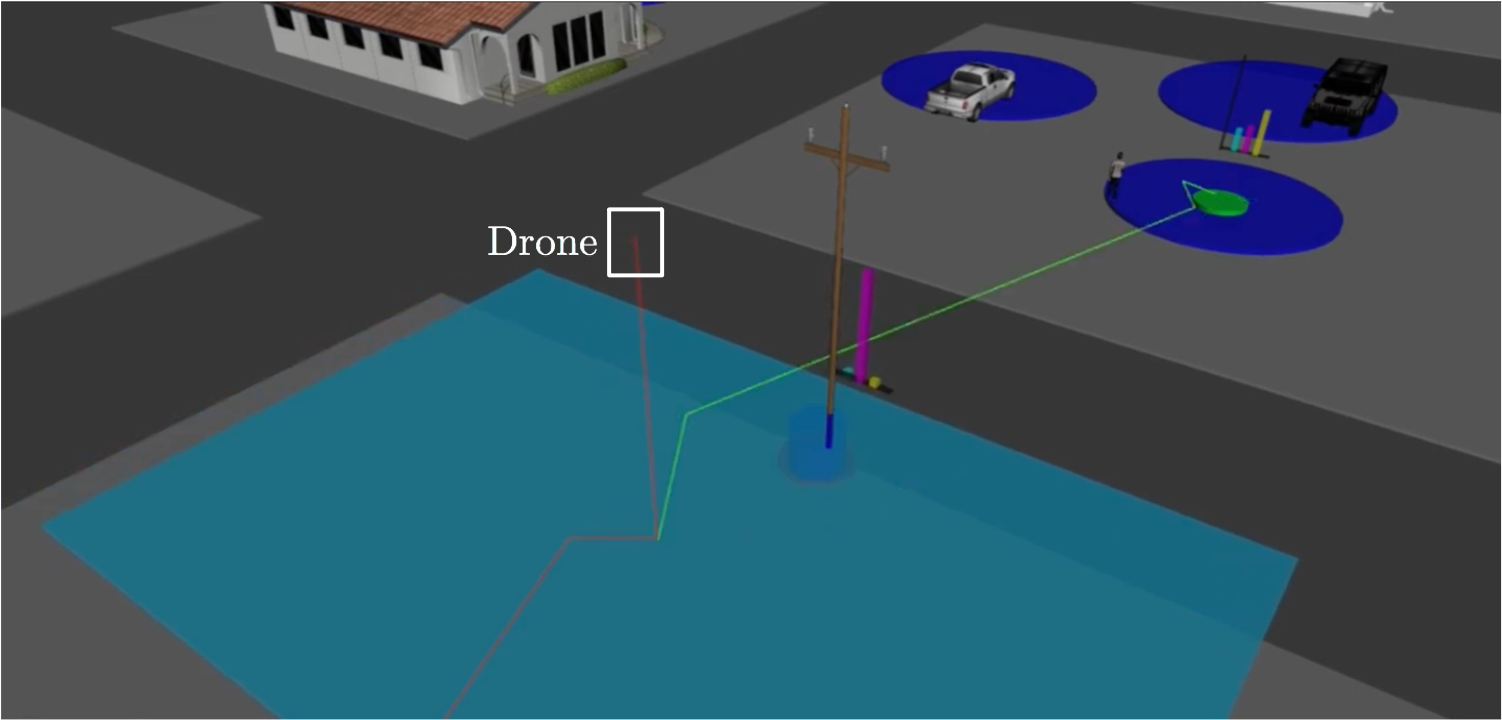}}
          \subfigure[Avoiding `pole' \& heading towards `person']{
    \label{fig:t3drone1replan}
  \includegraphics[width=0.47\linewidth]{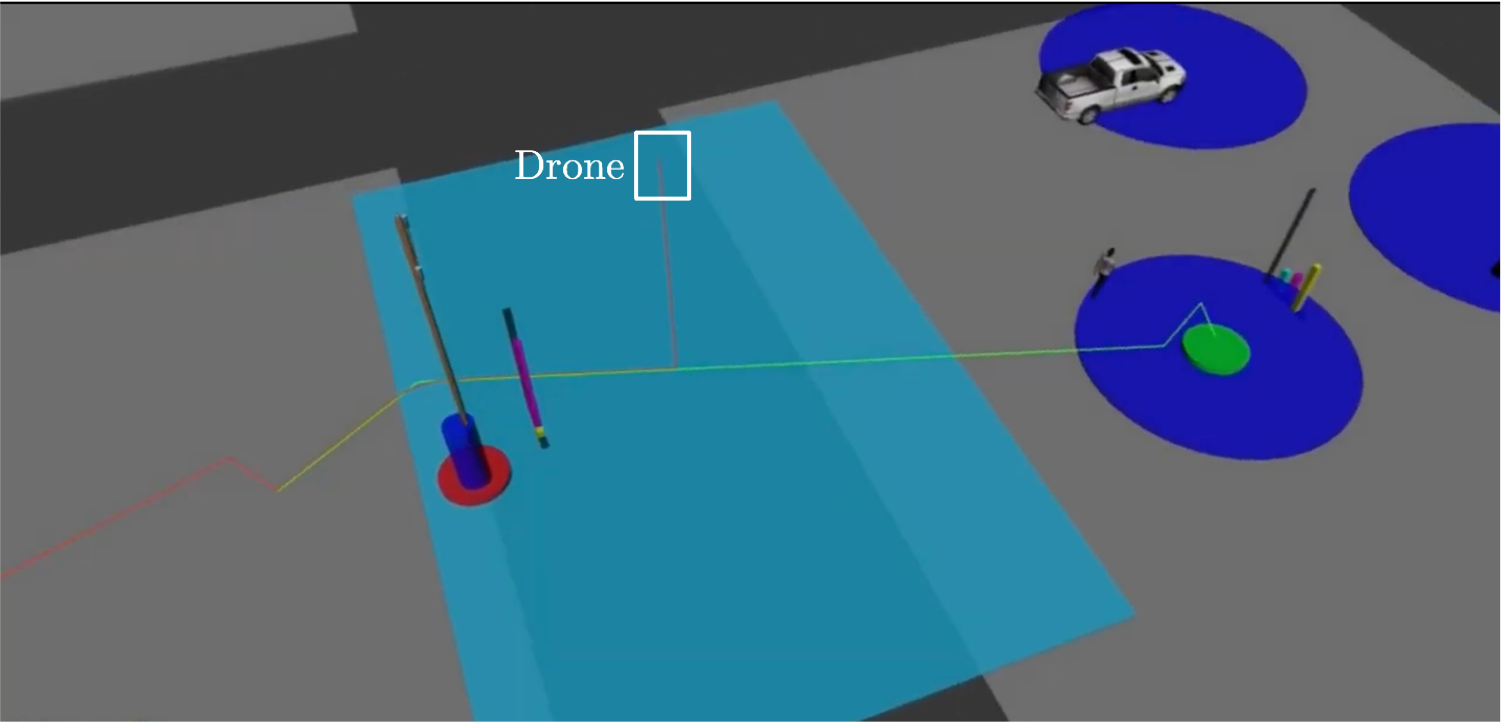}}
\subfigure[Heading towards the `person']{
    \label{fig:t3a}
  \includegraphics[width=0.47\linewidth]{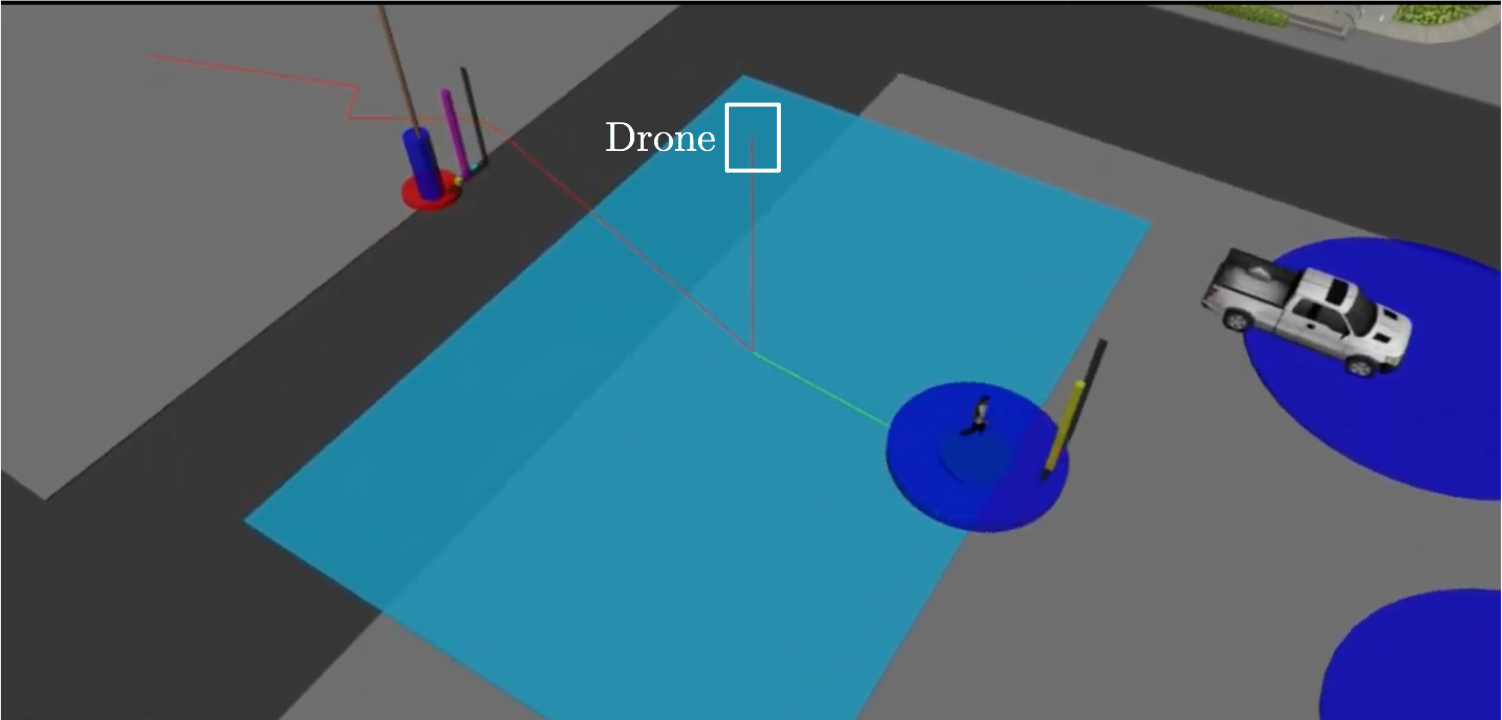}}
   \subfigure[Mission accomplished]{
    \label{fig:trajDrone1}
  \includegraphics[width=0.47\linewidth]{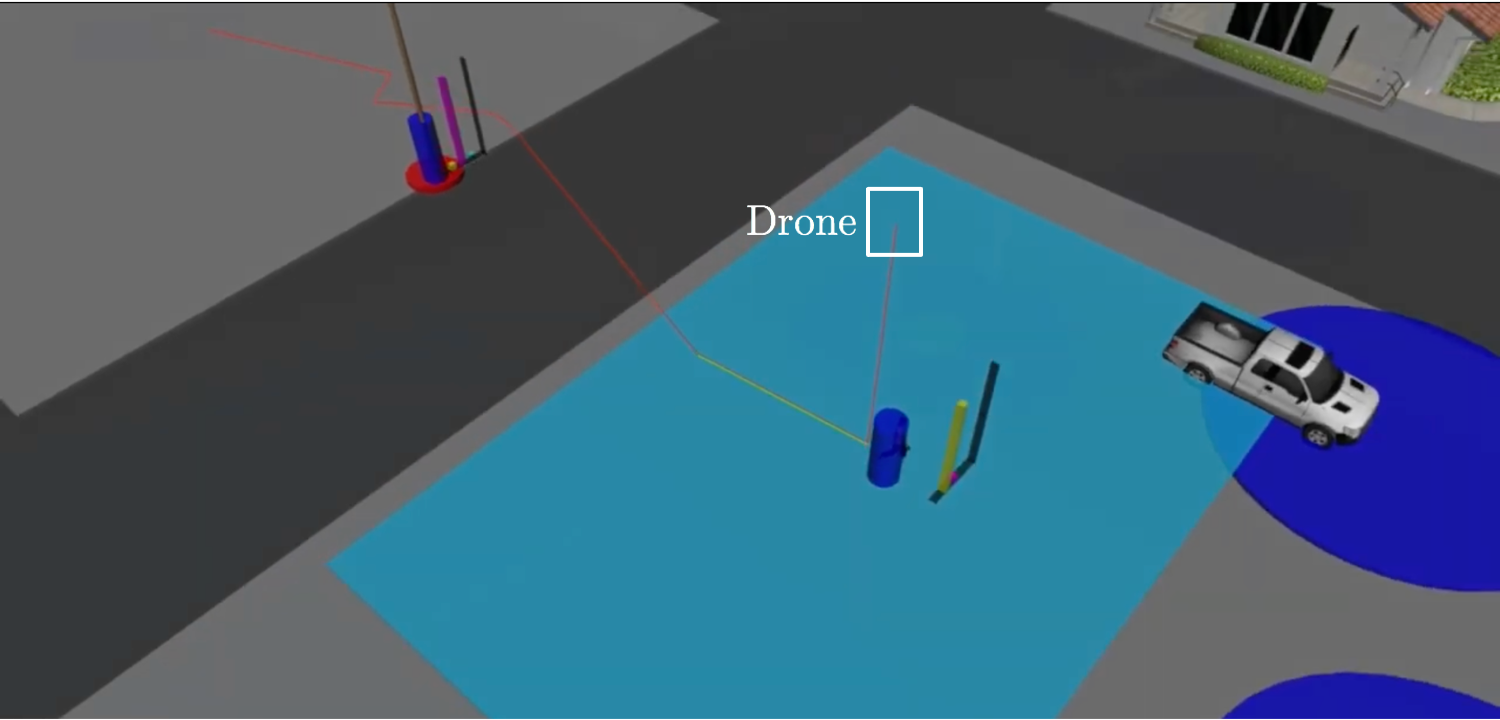}}  
 %
 \caption{Snapshots of a drone navigating an environment towards accomplishing the task in \eqref{eq:taskEx}. The green solid line shows the path the drone currently plans to follow. The drone initially heads towards the `pole' (Fig. \ref{fig:2Pole}) believing that this landmark is a `person' due to its inaccurate prior information (see also Fig. \ref{fig:EnvDrone1}). Once the `pole' is approached, the robot updates its belief about this landmark and it re-plans so that the `pole' is avoided while heading to another landmark that is (correctly) believed to be a `person' (Fig. \ref{fig:t2drone1}-\ref{fig:t3drone1}). Once the person is detected, the belief about the position and the class of this landmark is updated, the drone re-plans, and moves towards the person (Figs. \ref{fig:t3drone1replan}-\ref{fig:t3a}). Once the person is approached as per the imposed probabilistic requirements, the mission terminates (Fig. \ref{fig:trajDrone1}).}
  \label{fig:Drone1}
\end{figure}

\subsection{Effect of Prior Map Uncertainty on Replanning Frequency}\label{sec:replanSim}
In this section, we present experimental studies that involve AsTech Firefly Unmanned Aerial Vehicles (UAVs) that operate over a semantic city with dimensions $150\times 150$ m. The software simulation stack used is based on the Robot Operating System (ROS).

\subsubsection{Robot Dynamics} 
The AsTech Firefly UAV is governed by first order dynamics where the UAV state includes the position, velocity, orientation, and biases in the measured angular velocities and acceleration; more details can be found in \cite{Furrer2016}. 
In general, the more complex the robot dynamics is, the longer it takes for sampling-based methods to generate feasible paths as they have to explore a larger state- and control- space. 
To mitigate this issue, an approach that we investigate in this section, is to generate paths for simple robot dynamics that need to be followed by robots with more complex dynamics. Specifically, in what follows, we use Algorithm \ref{alg:RRT} to synthesize  paths considering the differential drive dynamics defined in \eqref{eq:nonlinRbt} that are simpler than the actual AsTech Firefly UAV dynamics. Given the waypoints, generated by Algorithm \ref{alg:RRT}, we compute minimum snap trajectories that smoothly transition through all waypoints every $T=2$ seconds \cite{mellinger2011minimum}.
%
The UAVs are controlled to follow the synthesized trajectories using the ROS package developed in \cite{Furrer2016}. 

\subsubsection{Perception System \eqref{eq:measModelR}-\eqref{eq:measModelC}}\label{sec:objRecDrone}
We assume that the drones are equipped with a downward facing sensor with rectangular field of view with dimensions $24\times16$ m that can take noisy positional measurements of landmarks falling inside its field-of-view (see e.g., Figure \ref{fig:EnvDrone1} and  \cite{freundlich2018distributed}), i.e., the measurement of landmark $\ell_i$ by robot $j$ is:
\begin{equation}\label{eq:measPosDrone}
    \bby_{j,i}=\bbx_j + \bbv,
\end{equation}
where $\bbv$ is Gaussian noise with covariance matrix with diagonal entries equal to $2$.

As for the object recognition method \eqref{eq:measModelC}, we assume that the robot is equipped with a neural network that is capable of detecting  classes of interest which are, for this case study, `car', `pole', and `person' and returning a discrete distribution over the available classes (see e.g. \cite{guo2017calibration}). Also, we model the confusion matrix of the object recognition method as follows:
\[ \left( \begin{array}{ccc}
0.8 & 0.23 & 0.06  \\
0.18  & 0.75 & 0.04 \\
0.02 & 0.02 & 0.9 \end{array} \right),\] 
where the first, second, and third row and column correspond to the classes `person', `car', and `pole', respectively. The entries in this confusion matrix capture the conditional probabilities $\mathbb{P}(\text{predicted class (row)}|\text{actual class (column)})$.
For instance, if the drone sees a person then with probability $0.8$, $0.18$, and $0.02$, it will generate a discrete distribution where the most likely class is `person', `car', and `pole', respectively. The confusion matrix can be learned offline based on training data. We note that any other object recognition/classification method can be used in place of the simulated one as this is used only to update the environmental belief.

\subsubsection{Mission Specifications \& Reacting to the Learned Map}

First, we consider a single drone that is responsible for accomplishing the task \eqref{eq:taskEx} described in Example \ref{ex:LTL} with parameters $r_1=r_2=0.2$ m, and $\delta_1=\delta_2=0.2$. The prior semantic semantic map that is available to the drone is illustrated in Figure \ref{fig:EnvDrone1}. \textcolor{black}{As per the prior Gaussian distribution, it holds that $||\hat{\bbx}_1(0)-\bbx_i||=4.9$m, $||\hat{\bbx}_2(0)-\bbx_i||= 5$m where $\ell_1$ and $\ell_2$ refer to the landmarks with true classes 'person' and 'pole', respectively. Also, the semantic prior is inaccurate as the most likely class for landmark $\ell_2$, i.e., the pole (that has to be avoided), is `person'  (who has to be approached). 
Also, the most likely class for landmark $\ell_1$, i.e., the person, is 'person'. 
}
Due to this (partially) wrong semantic prior, the  drone initially heads towards the `pole' (see Fig. \ref{fig:2Pole}). As the drone navigates the environment, it updates the semantic map and it accordingly re-plans so that eventually the person is approached and the pole is avoided as per the imposed probabilistic requirements (see Figs. \ref{fig:t2drone1}-\ref{fig:trajDrone1}).\footnote{In our implementation, when re-planning is triggered a new tree is constructed from scratch. A more efficient approach would be to exploit the previously constructed tree structure. For instance, given the new online learned map, the a-posteriori covariance matrices for all tree nodes that are reachable from the current one (where re-planning was triggered) can be computed along with the corresponding DFA states. Then re-planning may start using this sub-tree. This way previously sampled obstacle-free multi-robot states are not discarded.} In total re-planning was triggered $5$ times. Snapshots of the robot trajectory along with instances where re-planning was triggered are provided and discussed in Figure \ref{fig:Drone1}. Additional simulations for various initial semantic maps can be found in \cite{SimSemMaps} where it can be observed that the more inaccurate the initial semantic map is, the more often re-planning is triggered. \textcolor{black}{For instance, re-planning was never triggered when the prior means of the two landmarks (i.e., person and pole) were close enough to the actual positions, i.e., $||\hat{\bbx}_1(0)-\bbx_1||= 0.7$m and $||\hat{\bbx}_2(0)-\bbx_2||=1$m, and the most likely class, as per the prior distribution $d(0)$, for each landmark, is the true one. Similar observations were made for case studies involving a team of five drones; see Figure \ref{fig:multiDrone} and \cite{SimSemMaps}.}

\begin{figure}[t]
 \centering
  \includegraphics[width=1\linewidth]{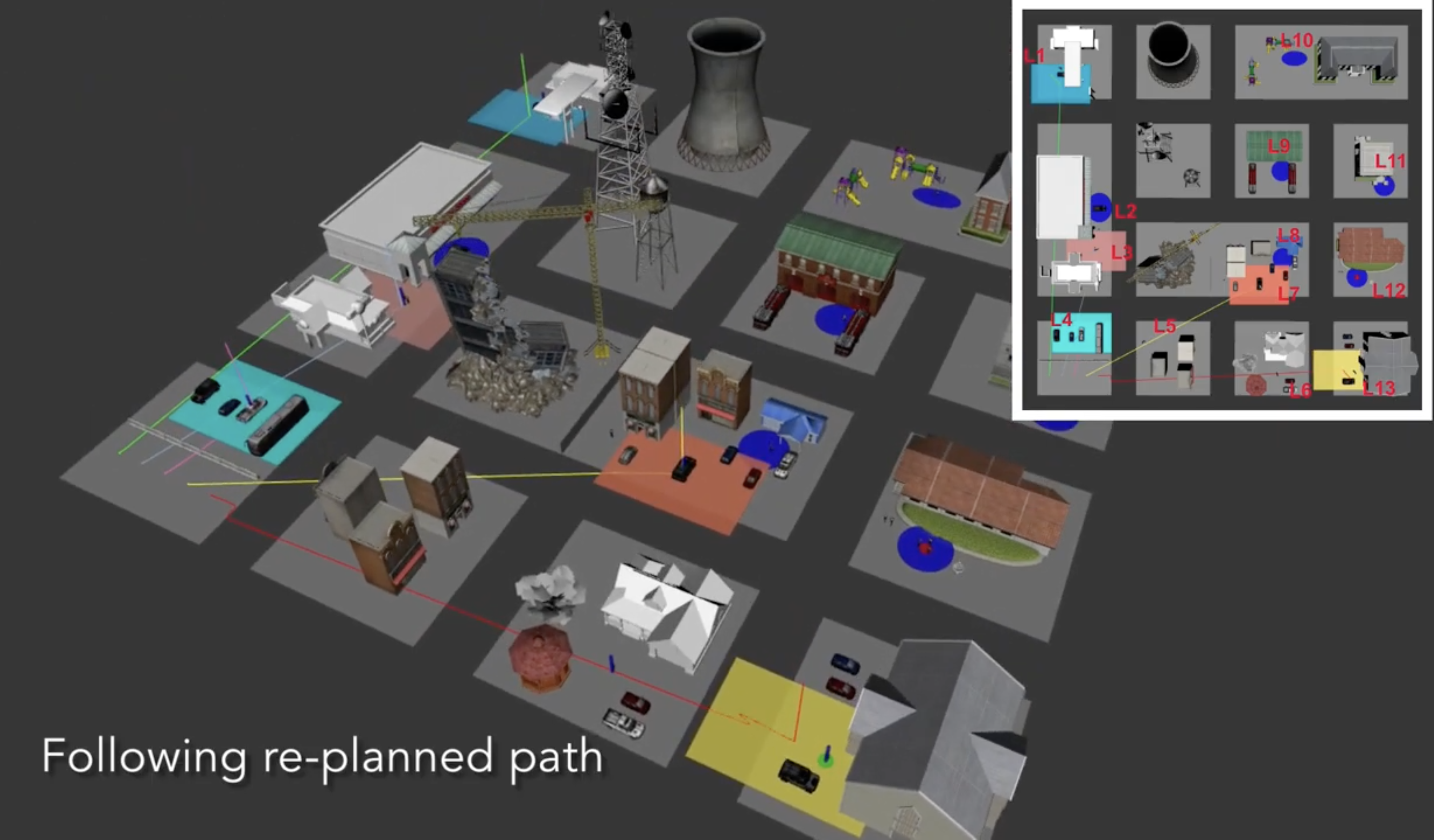}
\caption{A team of drones equipped with downward facing cameras (colored rectangles) is responsible for accomplishing a collaborative mission.}
  \label{fig:multiDrone}
\end{figure}

\textcolor{black}{\subsection{Effect of Map Density on Re-planning Frequency}\label{sec:ObsDensity}
In this section, we build upon the case study considered in Section \ref{sec:replanSim} and we examine how the density of uncertain obstacles/poles affect the re-planning frequency and runtimes. Specifically, we consider a single-drone with the dynamics, perceptual capabilities, and the mission specification discussed in Section \ref{sec:replanSim}. Recall that the assigned mission \eqref{eq:taskEx} requires the drone to eventually reach a person while always avoiding all poles with a user-specified probability; see also Figure \ref{fig:FourObstacles}. In Figure \ref{fig:replanRuntime}, we report the average (re-)planning runtimes as the number of the poles increases. Specifically, to design the initial or revised paths, we run Algorithm \ref{alg:RRT} five times and we report the average runtime. Observe in this figure, that as the number of uncertain poles increases, the re-planning frequency increases. Additionally, notice that as the mission progresses (i.e., as the re-planning counter increases), the runtimes tend to decrease, as the robot gets 'closer' to accomplishing the assigned task.}

\begin{figure}[t]
 \centering
   \subfigure[Initial Map]{
    \label{fig:InitMap4}
  \includegraphics[width=0.45\linewidth]{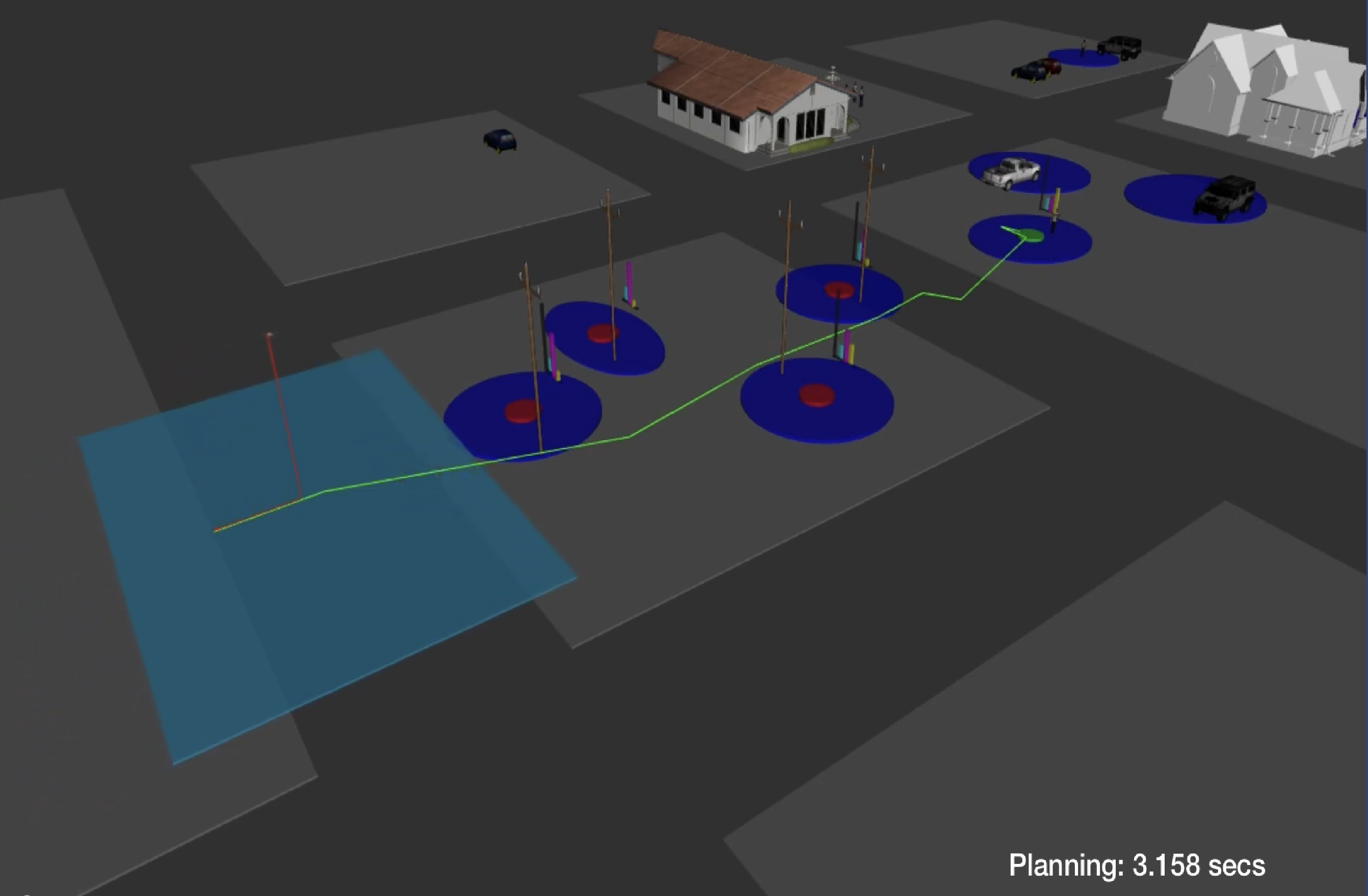}
  }
  \subfigure[Approaching the person ]{
    \label{fig:FinalMap44}
  \includegraphics[width=0.45\linewidth]{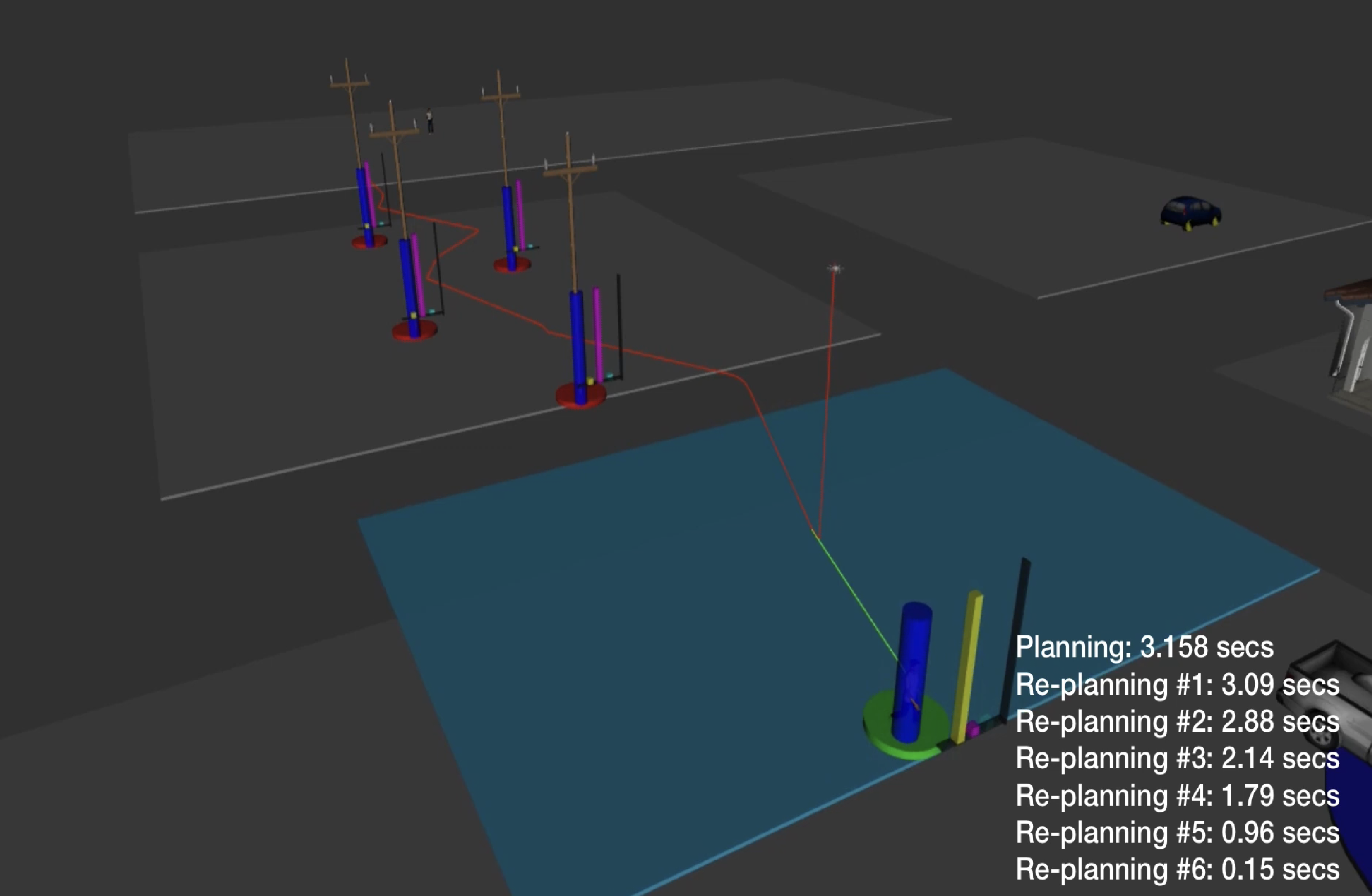}
  }
\caption{\textcolor{black}{A drone tasked with reaching a person while avoiding four poles.}}
  \label{fig:FourObstacles}
\end{figure}

\begin{figure}[t]
 \centering
  \includegraphics[width=0.7\linewidth]{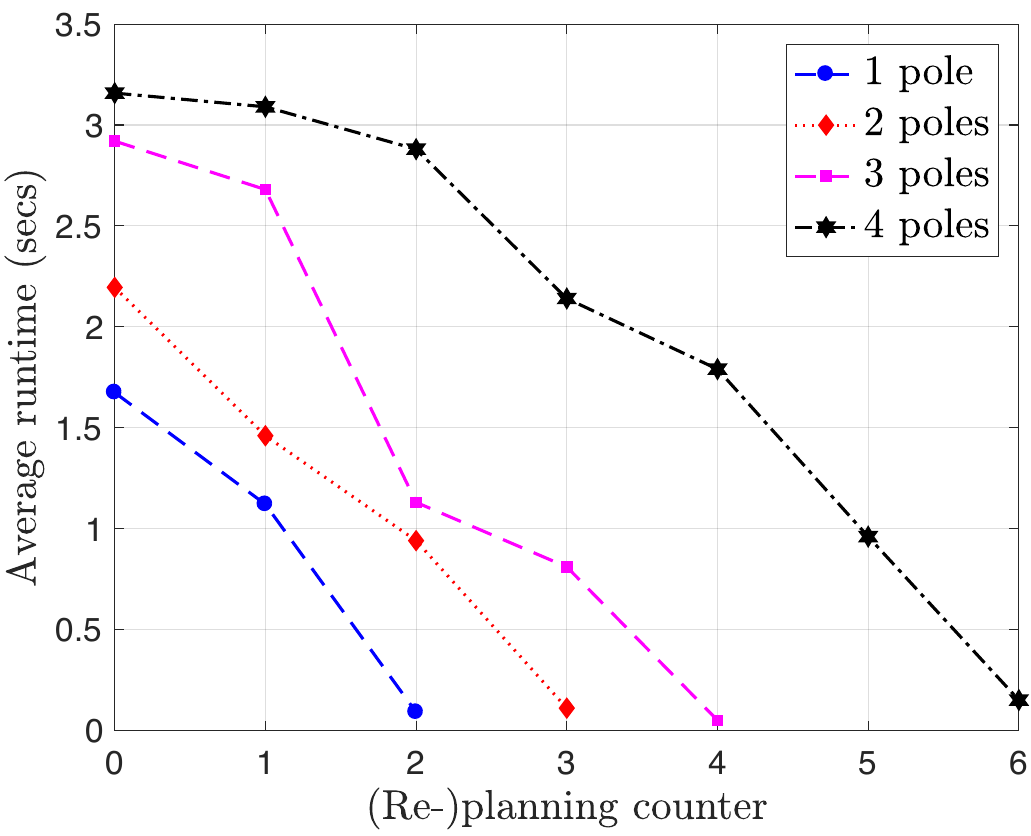}
\caption{\textcolor{black}{Graphical depiction of replanning frequency and average (re-)planning runtime. The runtime when the re-planning counter is zero refers to the runtime to design the initial paths. Each curve corresponds to a different number of poles. The markers along the curves correspond to re-planning events. The average total planning runtime for $1, 2, 3,$ and $4$ poles is $0.961, 1.176, 1.518,$ and $2.024$ secs. Note that these average runtimes are significantly smaller (between $0.1$ and $0.3$ secs) when Algorithm \ref{alg:RRT} is not executed in the Gazebo simulation environment.}}
  \label{fig:replanRuntime}
\end{figure}

\textcolor{black}{\subsection{Benefit of Predicting Offline the Metric Uncertainty}\label{sec:PredUnc}
In this section, we compare the proposed method against \cite{kantaros2019optimal}. Recall that Algorithm \ref{alg:RRT} updates the covariance matrices of the semantic map as per \eqref{constr7c} as opposed to \cite{kantaros2019optimal} where the map remains fixed during the synthesis phase. Particularly, we consider the case study discussed in Section \ref{sec:replanSim} evaluated on two semantic maps shown in Figures \ref{fig:InfoMap}-\ref{fig:UnInfoMap}. The only difference between these two maps is that the prior Gaussian distributions in Figure \ref{fig:InfoMap} are more informative (or less uncertain) than the ones in Figure \ref{fig:UnInfoMap}. When the informative map in Figure \ref{fig:InfoMap} is considered both \cite{kantaros2019optimal} and Alg. \ref{alg:RRT} generate feasible initial paths after $1.152$ and $1.548$ seconds, with costs $47$m and $46$m and terminal horizons for both equal to $13$, respectively, while re-planning was triggered three times for both algorithms. When the less informative map in Figure \ref{fig:UnInfoMap} is considered, Alg. \ref{alg:RRT} found a feasible path after $4.824$ secs. To the contrary, \cite{kantaros2019optimal} failed to design feasible initial paths as the imposed probabilistic requirements cannot be satisfied under the considered prior metric uncertainty. In other words, due to updating offline the covariance matrices, Alg. \ref{alg:RRT} can find feasible paths in prior maps with high metric uncertainty.}

\begin{figure}[t]
 \centering
    \subfigure[Informative Prior Map]{
    \label{fig:InfoMap}
  \includegraphics[width=0.46\linewidth]{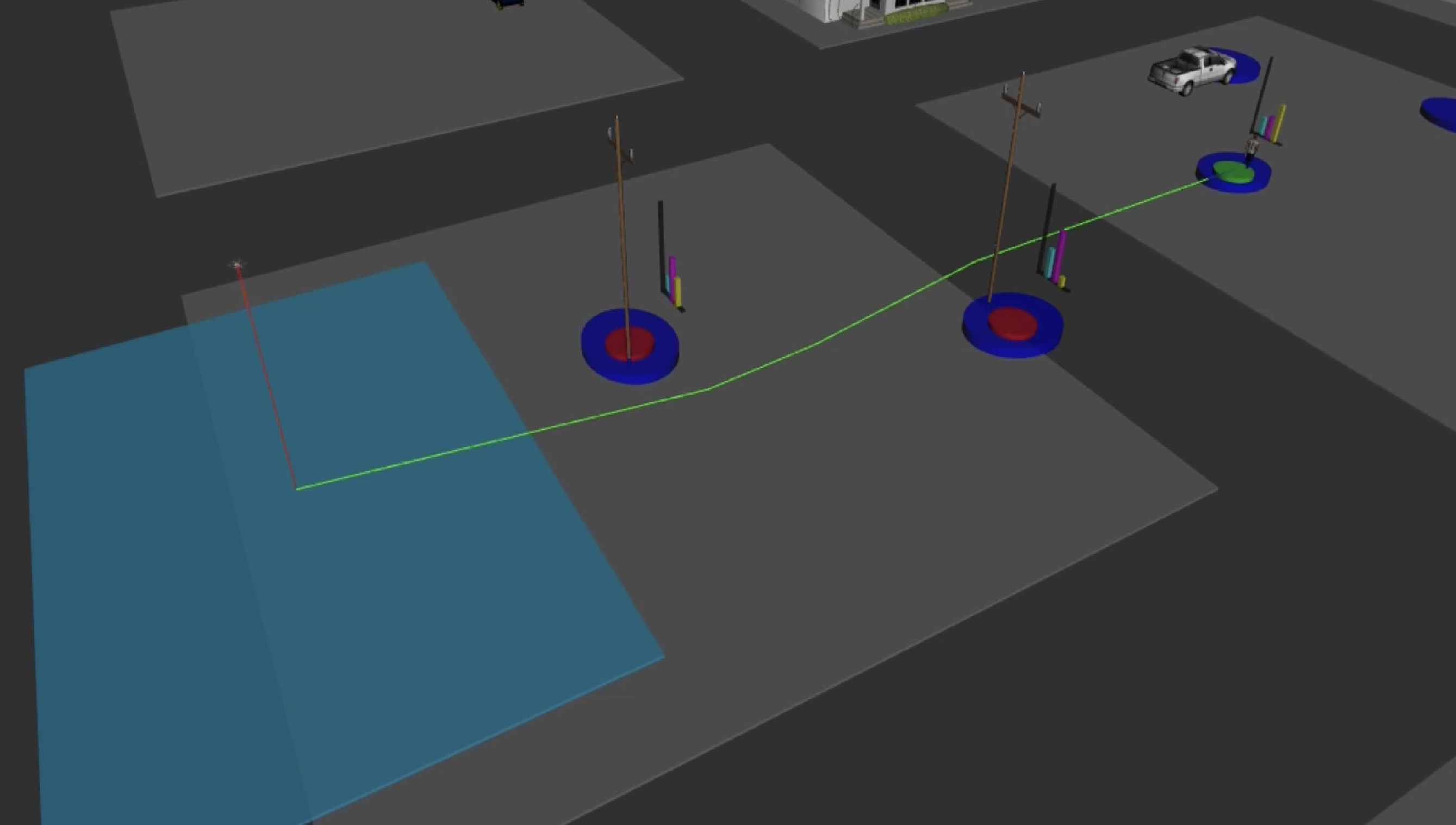}
  }
    \subfigure[Uninformative Prior Map ]{
    \label{fig:UnInfoMap}
  \includegraphics[width=0.46\linewidth]{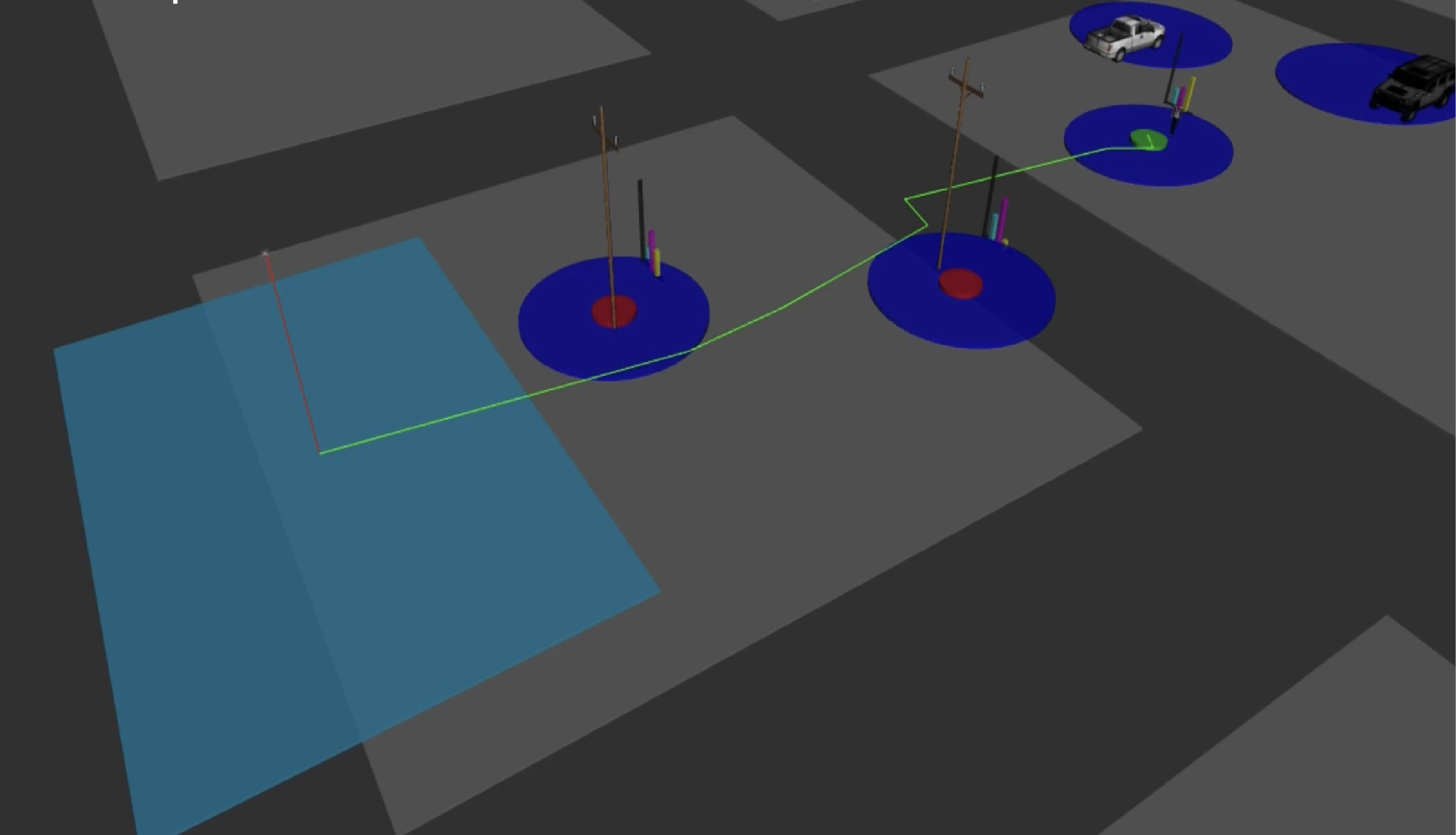}
  }
\caption{\textcolor{black}{Semantic maps of an environment with two poles and a person. Figure \ref{fig:InfoMap} shows a semantic map with an informative metric prior distribution while Figure \ref{fig:UnInfoMap} shows a less informative map; observe the flatter Gaussian distributions.}}
  \label{fig:OfflineUpd}
\end{figure}

\subsection{Safe Planning in Semantically Unexplored Environments }\label{sec:noMap}
In this section, we evaluate Algorithm \ref{alg:RRT} in environments with known geometric structure but with no prior semantic map, i.e., when Assumption \ref{as:prior} does not hold. Specifically, in this case, the number $M$ of landmarks is unknown while a prior uncertain semantic map $\ccalM(0)$ is not available. In this case, Algorithm \ref{alg:RRT} is applied as discussed in Section \ref{sec:extensions}.

In what follows, we consider a single ground robot with the dynamics defined in \eqref{eq:nonlinRbt} equipped with (i) a camera capable of taking noisy positional measurements (as in \eqref{eq:measPosDrone}) of all landmarks that are within range of $1$ m where the covariance matrix of the Guassian measurement noise has diagonal entries equal to $0.3$ and (ii) an object recognition method as defined in Section \ref{sec:objRecDrone}. The robot operates in an $10\times10$ m environment with the semantic and geometric structure presented in Figure \ref{fig:environmentUnexpl}. The robot is responsible for eventually finding a landmark with class `resource center' to pick up supplies, and then eventually find an `injured person' where the collected supplies will be delivered, while always avoiding known obstacles. Formally, this sequencing mission can be expressed by the following LTL formula
\begin{align}\label{eq:task1}
\phi = &\Diamond \{\pi_p(\bbp(t),\ccalM(t),\{1,0.2,0.1, \text{'resource center'}\})\wedge \nonumber\\ &\Diamond [\pi_p(\bbp(t),\ccalM(t),\{1,0.2,0.1, \text{'injured person'}\})]\},
\end{align}
where $\pi_p(\bbp(t),\ccalM(t),\{1,0.2,0.1, \text{'resource center'}\})$ and $\pi_p(\bbp(t),\ccalM(t),\{1,0.2,0.1, \text{'injured person'}\})$ are atomic predicates defined as per \eqref{apMS}. For instance, the latter requires the robot, with probability greater than $0.9$ to be within distance of $0.2$m from a landmark with class `injured person'. In this case study, we assume that landmarks $\ell_9$ and $\ell_{14}$ are the landmarks with classes `injured person' and 'resource center', respectively; see also Fig. \ref{fig:environmentUnexpl}.

We emphasize that initially the robot does not have any prior information about where the landmarks of interest are. As a result, the robot initially adopts an exploration mode aiming to detect landmarks of interest; see Fig. \ref{fig:t0}. Here, since the mission involves two distinct landmarks, we require the robot to switch from exploration to exploitation (i.e., application of Algorithm \ref{alg:RRT} over the currently learned semantic map) for the first time when the first two landmarks are detected. This happens at the time instant $t=225$ (see Fig. \ref{fig:t225}). Algorithm \ref{alg:RRT} fails to design a path satisfying the assigned task based on the currently learned map within $1000$ iterations. Thus, the robot keeps being in exploration mode and switches to exploitation mode every time a new landmark is detected. At $t=335$, both landmark $\ell_9$ and $\ell_5$ are detected and then Algorithm \ref{alg:RRT} is triggered generating feasible paths over the currently learned semantic map. Then the robot switches to an exploitation mode following the synthesized paths; see Fig. \ref{fig:ExploitationMode}. Observe in Figure \ref{fig:ExploitationMode} that the robot does not explore the whole environment but, instead, it only partially learns/explores the environment to accomplish the assigned mission. This simulation required in total $51.2$ \text{seconds}.

\begin{figure}[t]
 \centering
  \includegraphics[width=0.60\linewidth]{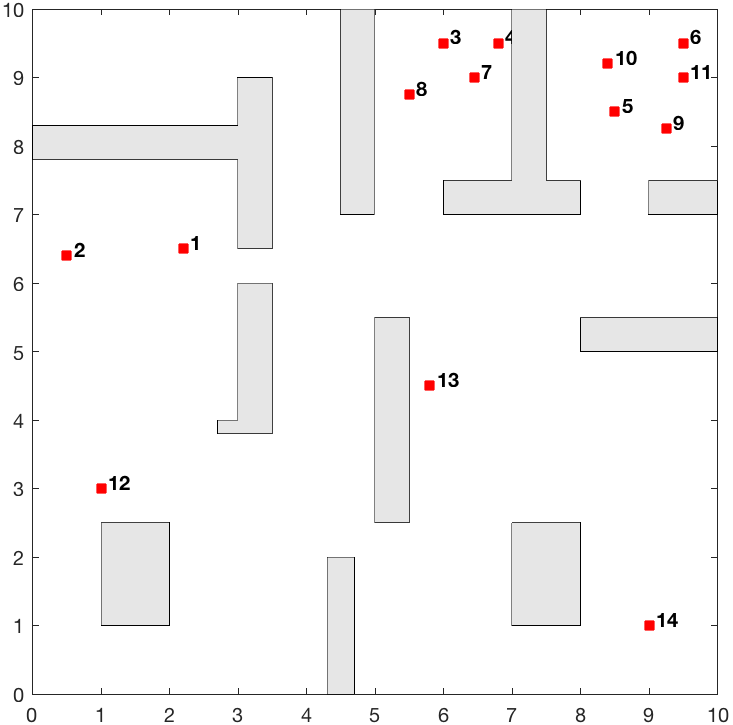}
\caption{Graphical depiction of the environment considered in Section \ref{sec:noMap}. The gray polygons represent known obstacles while the red squares denote the true positions of the landmarks. The numbers next to the landmarks denote their respective indices. Each index correspond to a different label, i.e., there are $|\ccalC|=14$ available classes. For instance, $14$ (bottom right corner) corresponds to 'resource center' while $9$ (top right corner) corresponds to 'injured person'. The robot does not know either the positions or the classes of the landmarks.}
  \label{fig:environmentUnexpl}
\end{figure}

\begin{figure}[t]
  \centering
%
\subfigure[$t=0$]{
    \label{fig:t0}
  \includegraphics[width=0.48\linewidth]{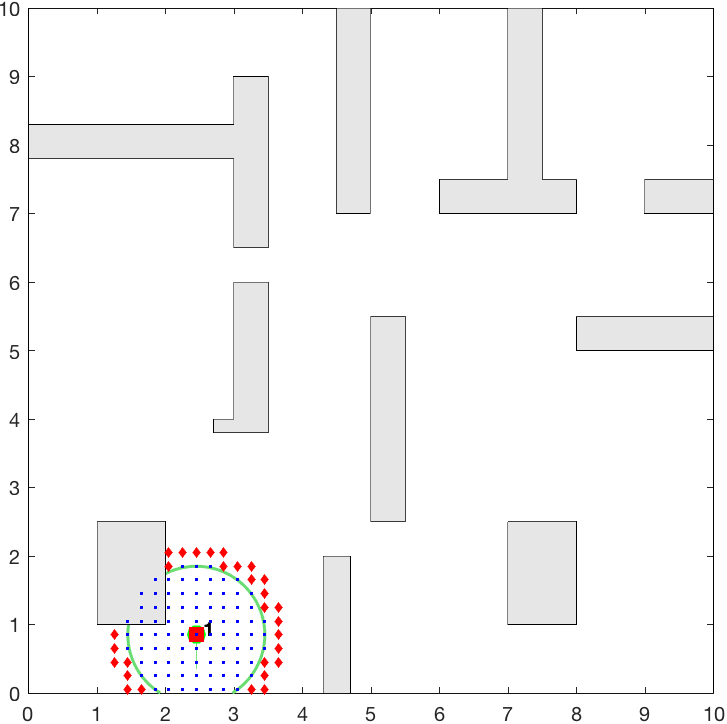}}
  \subfigure[$t=40$]{
    \label{fig:t40}
  \includegraphics[width=0.48\linewidth]{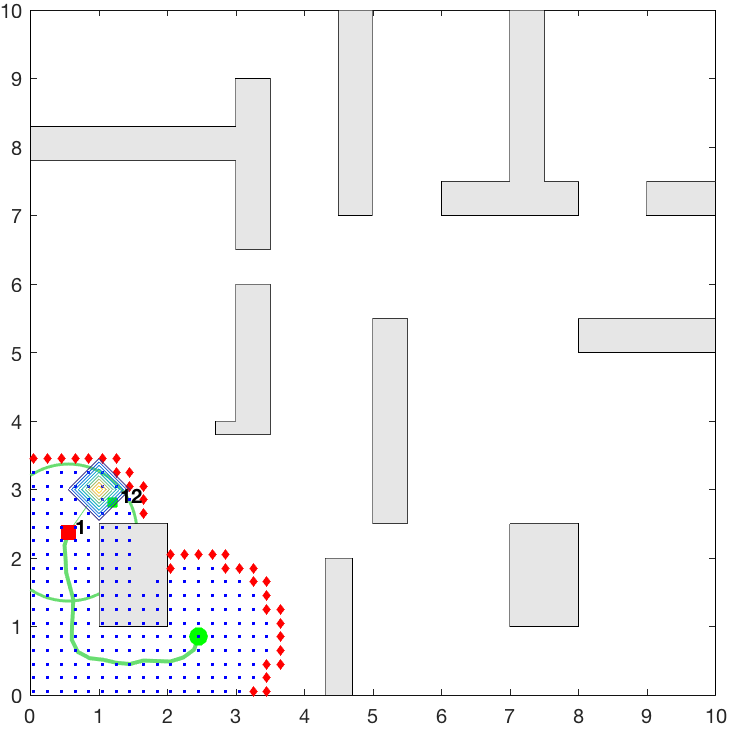}}
          \subfigure[$t=225$]{
    \label{fig:t225}
  \includegraphics[width=0.48\linewidth]{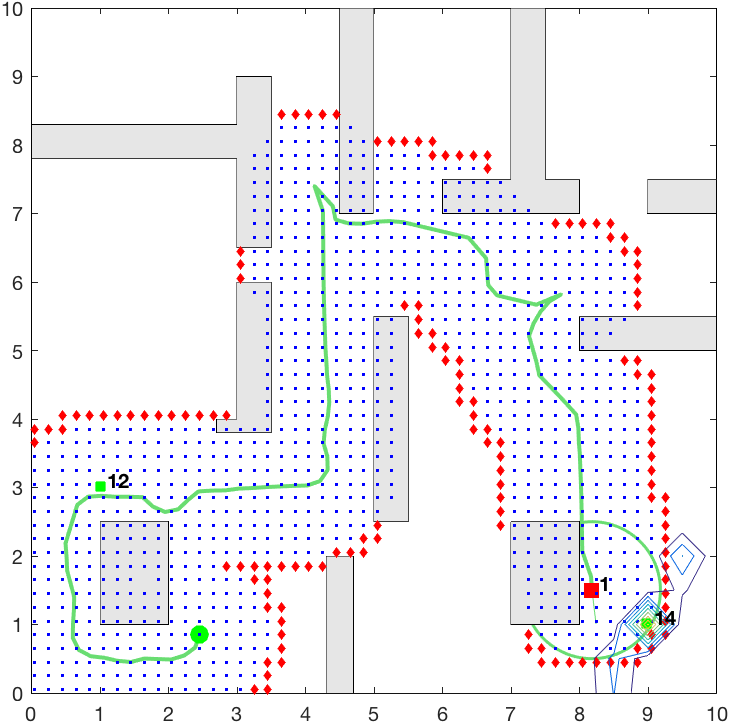}}
   \subfigure[$t=335$]{
    \label{fig:t335}
  \includegraphics[width=0.48\linewidth]{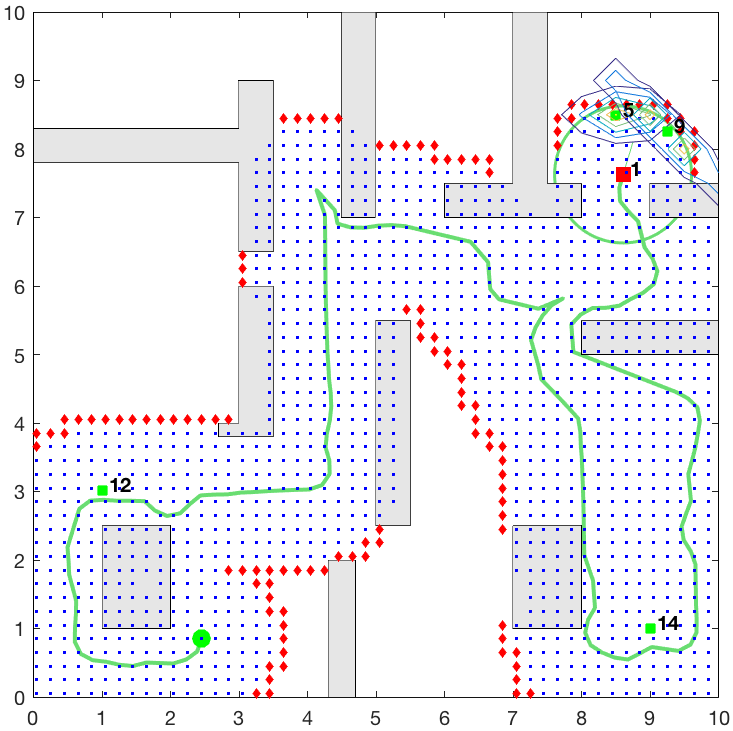}}  
 %
 \caption{Exploration Mode: Figures \ref{fig:t0}-\ref{fig:t335} illustrate the robot path towards exploring the environment. The green disk, red square, and the green solid line represent the initial robot location, the current robot location, and the followed trajectory, respectively. The green circle centered at the current robot position denotes the sensing range. The green squares denote the actual position of the detected landmarks. Gaussian distributions for landmarks $\ell_i$ with $\det\Sigma_i<10^{-5}$ are not shown. The numbers next to the landmarks show the most likely class. The blue dots represent the cells of the grid map that have been explored while the red diamonds show the `fake' exploration landmarks. At time instants $t=40$,  $t=225$, and $t=335$, the landmarks $\ell_{12}$, $\ell_{14}$ ('resource center'), and $\ell_5$ and $\ell_9$ ('injured person') are detected, respectively. Mission-based planning (Algorithm \ref{alg:RRT}) was triggered at the time instants $t=225$, and $t=335$. Feasible paths were found only at $t=335$ illustrated in Figure \ref{fig:ExploitationMode}.}
  \label{fig:ExplorationMode}
\end{figure}

\begin{figure}[t]
  \centering
%
\subfigure[$t=350$]{
    \label{fig:t350}
  \includegraphics[width=0.48\linewidth]{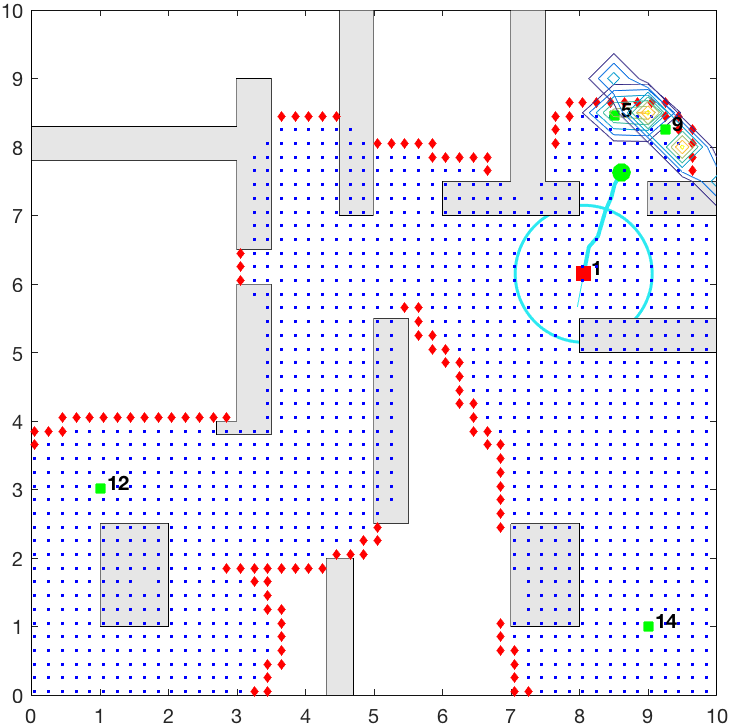}}
  \subfigure[$t=380$]{
    \label{fig:t380}
  \includegraphics[width=0.48\linewidth]{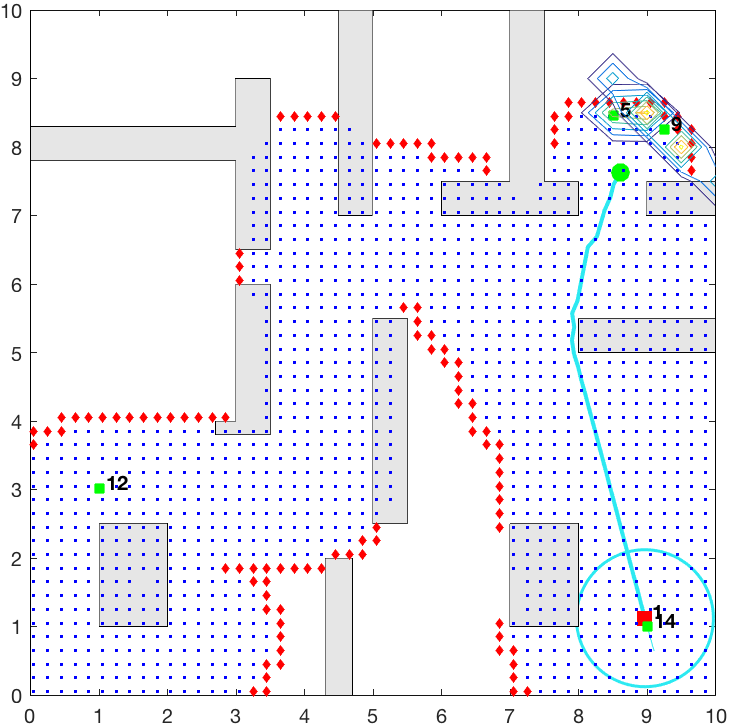}}
      \subfigure[$t=410$]{
    \label{fig:t410}
  \includegraphics[width=0.48\linewidth]{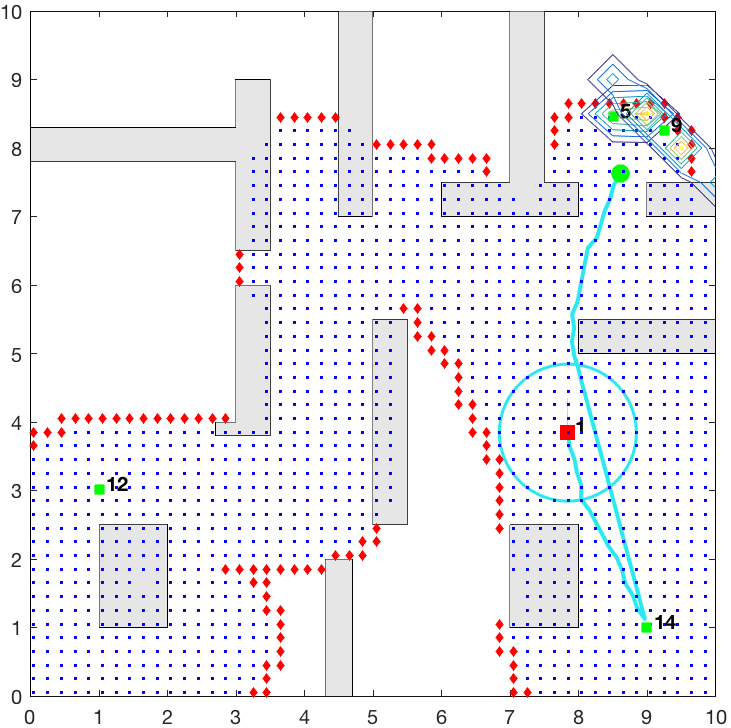}}
\subfigure[$t=435$]{
    \label{fig:t435}
  \includegraphics[width=0.48\linewidth]{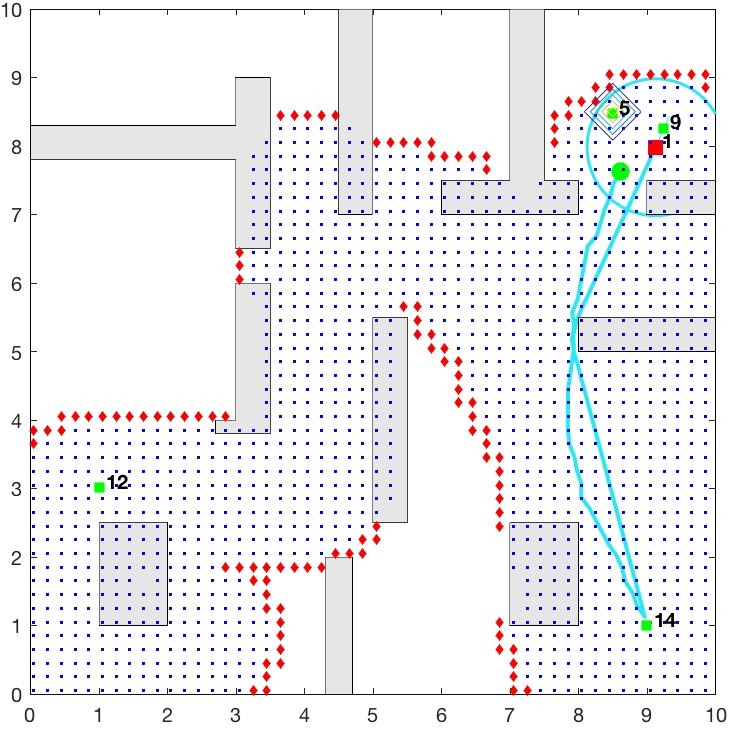}}
 \caption{Exploitation Mode: Figures \ref{fig:t350}-\ref{fig:t435} illustrate the robot path (cyan solid line) to accomplish the task \eqref{eq:task1} designed at time $t=335$ (green disk) when Algorithm \ref{alg:RRT} successfully designed feasible paths. 
 From $t=335$ until $t=380$, the robot moves towards landmark $\ell_{14}$ ('resource center'). Once this landmark is visited as per the imposed probabilistic requirement, the robot heads towards landmark $\ell_9$ ('injured person') which is visited at $t=435$ and, then, the mission terminates.} 
  \label{fig:ExploitationMode}
\end{figure}

\section{Conclusion} \label{sec:Concl}
In this paper, we addressed a multi-robot mission planning problem in \textcolor{black}{environments with partially unknown semantics}. In particular, the semantic environment was modeled as static labeled landmarks with uncertain positions and classes. We considered robots equipped with noisy perception systems that are tasked with accomplishing collaborative tasks with probabilistic requirements captured by co-safe LTL. We proposed a sampling-based algorithm to design  open-loop sensor-based control policies that can accomplish the assigned mission. The synthesized policies are updated online to adapt to the semantic map that is continuously learned using semantic SLAM methods. Theoretical and experimental analysis were provided to support the proposed method. 
Our future work will focus on designing (i) distributionally robust motion planning algorithms to account for un-modeled sensor noise; (ii) fast reactive planning frameworks that can quickly adapt to metric/semantic uncertainty and the potentially dynamic geometric structure of the environment.

\appendices
\section{Proof of Completeness, Optimality, \& Complexity}\label{sec:prop}

In what follows, we denote by $\mathcal{G}^n=\{\mathcal{V}^{n}, \mathcal{E}^{n}, \ccalJ_{\ccalG}\}$ the tree that has been built by Algorithm \ref{alg:RRT} by the $n$-th iteration. The same notation also extends to $f_{\ccalV}$, $f_{\ccalU}$, and $\bbu_{\text{new}}$. 
To prove Theorems \ref{thm:probCompl} and \ref{thm:asOpt}, it suffices to show that Algorithm \ref{alg:RRT} can find all possible paths $\bbq_{0:H}$, for any $H>0$, whether they are feasible or not, with probability $1$ as $n\to\infty$. Showing this result means that if there exists a feasible (including the optimal) path $\bbq_{0:H}$, then Algorithm \ref{alg:RRT} will find it with probability $1$ as $n\to\infty$. The key idea is to prove that once any node is added to the tree, then eventually this node will be extended towards all possible feasible directions as per the set of control inputs $\ccalU$. A key assumption to show this is that of finite control space $\ccalU$. 
To show this result, we first need to state the following results. 


\begin{lem}[Sampling $\ccalV_{k_{\text{rand}}}^n$]
Consider any subset $\ccalV_{k}^n$ and any fixed iteration index $n$ and any fixed $k\in\set{1,\dots,K_n}$. Then, there exists an infinite number of subsequent iterations $n+w$,  where $w\in\mathcal{W}\subseteq\mathbb{N}$ and $\ccalW$ is a subsequence of $\mathbb{N}$, at which the subset $\ccalV_{k}^n$ is selected to be the set $\ccalV_{k_{\text{rand}}}^{n+w}$. 
\label{lem:qrand}
\end{lem}

\begin{proof}
Let $A^{\text{rand},n+w}(k)=\{\ccalV_{k_{\text{rand}}}^{n+w}=\ccalV_{k}^n\}$, with $w\in\mathbb{N}$, denote the event that at iteration $n+w$ of Algorithm \ref{alg:RRT} the subset $\ccalV_{k}^n\subseteq\ccalV^n$ is selected by the sampling operation to be the set $\ccalV_{k_{\text{rand}}}^{n+w}$ [line \ref{rrt:samplekrand}, Alg. \ref{alg:RRT}]. Also, let $\mathbb{P}(A^{\text{rand},n+w}(k))$ denote the probability of this event, i.e., $\mathbb{P}(A^{\text{rand},n+w}(k))=f_{\ccalV}^{n+w}(k)$. 

Next, define the infinite sequence of events $A^{\text{rand}}=\{A^{\text{rand},n+w}(k)\}_{w=0}^{\infty}$, for a given subset $\ccalV_{k}^n\subseteq\ccalV^n$. In what follows, we show that the series $\sum_{w=0}^{\infty}\mathbb{P}(A^{\text{rand},n+w}(k))$ diverges and then we complete the proof by applying the Borel-Cantelli lemma \cite{grimmett2001probability}. 
%

Recall that by Assumption \ref{frand}(i) we have that  given any subset $\ccalV_{k}^n\subseteq\ccalV^n$, the probability $f_{\ccalV}^n(k|\ccalV^n)$ satisfies $f_{\ccalV}^n(k|\ccalV^n)\geq \epsilon$, for any iteration $n$. Thus we have that $\mathbb{P}(A^{\text{rand},n+w}(k))=f_{\ccalV}^{n+w}(k|\ccalV^{n+w})\geq \epsilon>0, 
$
for all $w\in\mathbb{N}$. Note that this result holds for any $k\in\set{1,\dots,K_{n+w}}$ due to Assumption \ref{frand}(i). 
%
Therefore, we have that $\sum_{w=0}^{\infty}\mathbb{P}(A^{\text{rand},n+w}(k))\geq\sum_{w=0}^{\infty}\epsilon.$
Since $\epsilon$ is a strictly positive constant, we have that $\sum_{w=0}^{\infty}\epsilon$ diverges. Then, 
we conclude that $\sum_{w=0}^{\infty}\mathbb{P}(A^{\text{rand},n+w}(k))=\infty.$
%
Combining this result with the fact that the events $A^{\text{rand},n+w}(k)$ are independent by Assumption \ref{frand}(ii), we get that $\mathbb{P}(\limsup_{k\to\infty} A^{\text{rand},n+w}(k))=1,$ by the Borel-Cantelli lemma. In other words, the events $A^{\text{rand},n+w}(k)$ occur infinitely often, for all $k\in\{1,\dots,K_n\}$. This equivalently means that for every subset $\ccalV_{k}^n\subseteq\ccalV^n$, for all $n\in\mathbb{N}_{+}$, there exists an infinite subsequence $\mathcal{W}\subseteq \mathbb{N}$ so that for all $w\in\mathcal{W}$ it holds $\ccalV_{k_{\text{rand}}}^{n+w}=\ccalV^{n}$, completing the proof.
\end{proof}

\begin{lem}[Sampling $\bbu_{\text{new}}$]
Consider any  subset $\ccalV_{k_{\text{rand}}}^n$ selected by $f_{\ccalV}$ and any fixed iteration index $n$. Then, for any given control input $\bbu\in\ccalU$, there exists an infinite number of subsequent iterations $n+w$, where $w\in\mathcal{W}'$ and $\mathcal{W}'\subseteq\mathcal{W}$ is a subsequence of the sequence of $\mathcal{W}$ defined in Lemma \ref{lem:qrand}, at which the control input $\bbu\in\ccalU$ is selected to be $\bbu_{\text{new}}^{n+w}$.
\label{lem:qnew}
\end{lem}

\begin{proof}
Define the infinite sequence of events $A^{\text{new}}=\{A^{\text{new},n+w}(\bbu)\}_{w=0}^{\infty}$, for $\bbu\in\ccalU$, where $A^{\text{new},n+w}(\bbu)=\{\bbu_{\text{new}}^{n+w}=\bbu\}$, for $w\in\mathbb{N}$, denotes the event that at iteration $n+w$ of Algorithm \ref{alg:RRT} the control input $\bbu\in\ccalU$ is selected by the sampling function to be the input $\bbu_{\text{new}}^{n+w}$, given the subset $\ccalV_{k_{\text{rand}}}^n\in\ccalV_{k_{\text{rand}}}^{n+w}$. Moreover, let $\mathbb{P}(A^{\text{new},n+w}(\bbu))$ denote the probability of this event, i.e., $\mathbb{P}(A^{\text{new},n+w}(\bbu))=f_{\ccalU}^{n+w}(\bbu| \ccalV_{k_{\text{rand}}}^{n+w})$. 
Now, consider those iterations $n+w$ with $w\in\mathcal{W}$ such that $k_{\text{rand}}^{n+w}=k_{\text{rand}}^n$ by Lemma \ref{lem:qrand}. We will show that the series $\sum_{w\in\ccalW}\mathbb{P}(A^{\text{new},n+w}(\bbu))$ diverges and then we will use the Borel-Cantelli lemma to show that any given $\bbu\in\ccalU$ will be selected infinitely often to be control input $\bbu_{\text{new}}^{n+w}$.
By Assumption \ref{fnew}(i)  we have that $\mathbb{P}(A^{\text{new},n+w}(\bbu))=f_{\ccalU}^{n+w}(\bbu| \ccalV_{k_{\text{rand}}}^{n+w})$ is bounded below by a strictly positive constant $\zeta>0$ for all $w\in\ccalW$. 
Therefore, we have that $\sum_{w\in\ccalW}\mathbb{P}(A^{\text{new},n+w}(\bbu))$ diverges, since it is an infinite sum of a strictly positive constant term. Using this result along with the fact that the events $A^{\text{new},n+w}(\bbu)$ are independent, by Assumption \ref{fnew}(ii), we get that $\mathbb{P}(\limsup_{w\to\infty} A^{\text{new},n+w}(\bbu))=1,$ due to the Borel-Cantelli lemma. In words, this means that the events $A^{\text{new},n+w}(\bbu)$ for $w\in\mathcal{W}$ occur infinitely often. Thus, given any subset $\ccalV_{k_{\text{rand}}}^n$, for every control input $\bbu$ and for all $n\in\mathbb{N}_{+}$, there exists an infinite subsequence $\mathcal{W}' \subseteq \mathcal{W}$ so that for all $w\in\mathcal{W}'$ it holds $\bbu_{\text{new}}^{n+w}=\bbu$. 
\end{proof}
Before stating the next result we first define the \textit{reachable} state-space of a state $\bbq(t)=[\bbp(t),\ccalM(t), q_D(t)]\in\ccalV_{k}^n$, denoted by $\ccalR(\bbq(t))$, that collects all states $\bbq(t+1)=[\bbp(t+1),\ccalM(t+1), q_D(t+1)]$ that can be reached within one time step from $\bbq(t)$.


\begin{cor}[Reachable set $\ccalR(\bbq)$]\label{cor:reach}
Given any state $\bbq=[\bbp,\ccalM,q_D]\in\ccalV_{k}^n$, for any $k\in\{1,\dots,K_n\}$, all states that belong to the reachable set $\ccalR(\bbq)$ will be added to $\ccalV^{n+w}$, with probability 1, as $w\to\infty$, i.e., $\lim_{w\rightarrow\infty} \mathbb{P}\left(\{\ccalR(\bbq)\subseteq\mathcal{V}^{n+w}\}\right)=1.$
Also, edges from $\bbq(t)$ to all reachable states $\bbq'\in\ccalR(\bbq)$ will  be added to $\ccalE^{n+w}$, with probability 1, as $w\to\infty$, i.e., $\lim_{w\rightarrow\infty} \mathbb{P}\left(\{\cup_{\bbq'\in\ccalR(\bbq)}(\bbq,\bbq')\subseteq\mathcal{E}^{n+w}\}\right)=1.$
\end{cor}

\begin{proof}
We show this result by contradiction. Let $\bbq\in\ccalV_{k}^n$ and $\bbq'\in\ccalR(\bbq)$. Since $\bbq'\in\ccalR(\bbq)$ we have that there exists a control input $\bbu$ so that $\bbq'$ can be reached from $\bbq$.
Assume that the edge from $\bbq$ to $\bbq'$ will never be added to the tree. This can happen in the following two cases. (i) There exists no subsequent iteration $n+w$, for some $w>0$, where the set $\ccalV_{k}^n$ will be selected to be $\ccalV_{{k_{\text{rand}}}}^{n+w}$ and, therefore, no edges starting from $\bbq$ will ever be constructed. However, this contradicts Lemma \ref{lem:qrand}. (ii) There exists no subsequent iteration $n+w$, for some $w>0$, where the control input $u$ will be selected to be $\bbu_{\text{new}}^{n+w}$ and, therefore, the edge from $\bbq$ to $\bbq'$ will never be considered to be added to the tree. However, this contradicts Lemma \ref{lem:qnew}. Therefore, we conclude that all states that belong to the reachable set $\ccalR(\bbq)$ along with edges from $\bbq$ to the states in $\ccalR(\bbq)$ will eventually be added to the tree structure. In mathematical terms this can be written as $\lim_{w\rightarrow\infty} \mathbb{P}\left(\{\ccalR(\bbq(t))\subseteq\mathcal{V}^{n+w}\}\right)=1$ and 
$\lim_{w\rightarrow\infty} \mathbb{P}\left(\{\cup_{\bbq'\in\ccalR(\bbq)}(\bbq,\bbq')\subseteq\mathcal{E}^{n+w}\}\right)=1$, completing the proof.
\end{proof}



\begin{cor}\label{cor:paths}
Consider any (feasible or not) path $\bbq_{0:H}=\bbq(0),\bbq(1),\dots,\bbq(H)$, where $\bbq(h)\in\ccalR(\bbq(h-1))$, for all $h\in\{1,\dots,H\}$. Then Algorithm \ref{alg:RRT} will find $\bbq_{0:H}$ with probability $1$ as $n\to\infty$.
\end{cor}
\begin{proof}
By construction of the path $\bbq_{0:H}$, it holds that $\bbq(h)\in\ccalR(\bbq(h-1))$, for all $h\in\{1,\dots,H\}$. Since $\bbq(0)\in\ccalV^1$, it holds that all states $\bbq\in\ccalR(\bbq(0))$, including the state $\bbq(1)$ (i.e., the second state in the path $\bbq_{0:H}$), will be added to $\ccalV^n$ with probability $1$, as $n\to\infty$, due to Corollary \ref{cor:reach}. 
Once this happens, the edge $(\bbq(0),\bbq(1))$ will be added to set of edges $\ccalE^n$ due to Corollary \ref{cor:reach}. 
Applying Corollary \ref{cor:reach} inductively, we get that $\lim_{n\rightarrow\infty} \mathbb{P}\left(\{\bbq(h)\in\mathcal{V}^{n}\}\right)=1$ and $\lim_{n\rightarrow\infty} \mathbb{P}\left(\{(\bbq(h-1), \bbq(h))\in\mathcal{E}^{n}\}\right)=1$, for all $h\in\{1,\dots,H\}$ meaning that the path $\bbq_{0:H}$ will be added to the tree $\ccalG^n$ with probability $1$ as $n\to\infty$.
\end{proof}

{\bf{Proof of Theorem \ref{thm:probCompl}}:}  
Probabilistic completeness directly follows from Corollary \ref{cor:paths}. Specifically, since Algorithm \ref{alg:RRT} can find all possible paths $\bbq_{0:H}$, it will also find any feasible path (if they exist), that satisfy $\phi$. 

{\bf{Proof of Theorem \ref{thm:asOpt}}:} The proof of this result follows from Theorem \ref{thm:probCompl}. Specifically, recall from Theorem \ref{thm:probCompl} that Algorithm \ref{alg:RRT} can find any feasible path and, therefore, the optimal path, as well, with probability $1$, as $n\to\infty$. 


\section{Designing Biased Sampling Functions}\label{sec:biasedSampling}
As discussed in Section \ref{sec:complOpt}, any mass functions $f_{\ccalV}$ and $f_{\ccalU}$ that satisfy Assumptions \ref{frand} and \ref{fnew}, respectively, can be employed in Algorithm \ref{alg:RRT} without sacrificing its completeness and optimality guarantees. 
In what follows, we describe a non-uniform/biased sampling strategy, originally developed in \cite{kantaros2019optimal}, allowing us to address large-scale planning tasks. 


\subsection{Pruning the DFA}\label{sec:prune}

We first prune the DFA by removing transitions that are unlikely to be enabled. Particularly, these are DFA transitions that are enabled when a robot $j$ has to satisfy atomic propositions that require it to be present in more than one location \textit{simultaneously}. Hereafter, these DFA transitions are called \textit{infeasible}. 
For instance, a DFA transition that is enabled if a robot is required to be simultaneously in more than one known disjoint region of interest, is classified as infeasible; see \cite{kantaros2018text} for more details. 
%
Determining infeasible transitions becomes more challenging for transitions that depend on probabilistic predicates as e.g., in \eqref{ap1}. The reason is that search over the space of distributions is required to determine if there exist robot positions that satisfy a Boolean formula defined over such predicates; recall that the map distribution is updated during the execution of Algorithm \ref{alg:RRT}. 
%
%
Hereafter, for simplicity, we classify such DFA transitions that require a robot to be close to more than one landmark at the same time as \textit{infeasible}.\footnote{Note that this is a conservative approach as transitions that are in fact feasible may be classified as infeasible transitions. Nevertheless, this does not affect the correctness of the proposed algorithm, since the biased sampling functions that will be constructed next still satisfy Assumptions \ref{frand} and \ref{fnew}. If in the pruned DFA, there is no path from the initial to the final state then this means that either the specification is infeasible or the task is feasible but the pruning was very conservative. In this case, a warning message can be returned to the user, and either uniform distributions  or the proposed biased distributions can be employed but using the original (and not the pruned) DFA.} 

Second, along the lines of \cite{kantaros2018text} we define a distance function $d: \ccalQ_D \times \ccalQ_D \rightarrow \mathbb{N}$ between any two DFA states, which captures the minimum number of transitions in the pruned DFA to reach a state $q_D'$ from a state $q_D$. This function is defined as follows
\begin{equation}
d(q_D,q_D')=\left\{
                \begin{array}{ll}
                  |SP_{q_D,q_D'}|, \mbox{if $SP_{q_D,q_D'}$ exists,}\\
                  \infty, ~~~~~~~~~\mbox{otherwise},
                \end{array}
              \right.
\end{equation}
where $SP_{q_D,q_D'}$ denotes the shortest path (in terms of hops) in the  pruned DFA from $q_D$ to $q_D'$ and $|SP_{q_D,q_D'}|$ stands for its cost (number of hops).
%
\subsection{Mass function $f_{\ccalV}$}\label{sec:fv}
%
To construct a non-uniform mass function $f_{\ccalV}$, first, we define the set $\ccalD_{\text{min}}$ that collects all tree nodes $\bbq=(\bbp,\ccalM,q_D)$ with the smallest distance $d(q_D,q_F)$, denoted by $d_{\text{min}}$, among all nodes in the tree, i.e., $\ccalD_{\text{min}}=\{\bbq=(\bbp,\ccalM,q_D)\in\ccalV~|~d(q_D,q_F)=d_{\text{min}} \}$. The set $\ccalD_{\text{min}}$ initially collects only the root and is updated (along with $d_{\text{min}}$) as new states are added to the tree. Given the set $\ccalD_{\text{min}}$, we define the set $\ccalK_{\text{min}}$ that collects the indices $k$ of the subsets $\ccalV_k$ that satisfy $\ccalV_k\cap\ccalD_{\text{min}}\neq\emptyset$.
Given the set $\ccalK_{\text{min}}$, the probability mass function $f_{\text{rand}}(k|\ccalV)$ is defined so that it is biased to select more often subsets $\ccalV_k\subseteq\ccalV$ that satisfy $k\in\ccalK_{\text{min}}$. Specifically, $f_{\text{rand}}(k|\ccalV)$ is defined as follows:
\begin{equation}\label{eq:frand}
f_{\text{rand}}(k|\ccalV)=\left\{
                \begin{array}{ll}
                  p_{\text{rand}}\frac{1}{|\ccalK_{\text{min}}|}, ~~~~~~~~~~~~\mbox{if}~k\in\ccalK_{\text{min}}\\
                  (1-p_{\text{rand}})\frac{1}{|\ccalV\setminus\ccalK_{\text{min}}|},~\mbox{otherwise},
                \end{array}
              \right.
\end{equation}
where $p_{\text{rand}}\in(0.5,1)$ stands for the probability of selecting \textit{any} subset $\ccalV_k$ that satisfies $k\in\ccalK_{\text{min}}$. Note that $p_{\text{rand}}$ can change with iterations $n$ but it should always satisfy $p_{\text{rand}}\in(0.5,1)$ to ensure that subsets $\ccalV_k$ with $k\in\ccalK_{\text{min}}$ are selected more often. 
\subsection{Mass function $f_{\ccalU}$}\label{sec:fu}
The sampling function $f_{\ccalU}(\bbu|k_{\text{rand}})$ is designed so that a state $\bbq=(\bbp,\ccalM,q_F)$ is reached by following the shortest path (in terms of hops) in the pruned DFA that connects $q_D^0$ to $q_F$.  
%
%
First, given  a state $\bbq_{\text{rand}}=(\bbp_{\text{rand}},\ccalM_{\text{rand}},q_D^{\text{rand}})$, we compute the next DFA state defined as $q_D^{\text{next}}=\delta_D(q_D^{\text{rand}},L(\bbp_{\text{rand}},\ccalM_{\text{rand}}))$. 
Next, we construct the reachable set $\ccalR_D(q_D^{\text{next}})$ that collects all states $q_D\in\ccalQ_D$ that can be reached in one hop in the \textit{pruned} DFA from $q_D^{\text{next}}$, for any observation/symbol $\sigma\in 2^{\mathcal{AP}}$, defined as $\ccalR_D(q_D^{\text{next}})=\{q_D\in\ccalQ_D|\exists  \sigma\in 2^{\mathcal{AP}}~\text{s.t.}~\delta_D(q_D^{\text{next}},\sigma)=q_D\}.$
%
Among all states in $\ccalR_D(q_D^{\text{next}})$ we select the state, denoted by $q_D^{\text{min}}$, with the minimum distance from $q_F$. 

\normalsize{Given $q_D^{\text{next}}$ and $q_D^{\text{min}}$, we select a symbol $\sigma$ that enables a transition from $q_D^{\text{next}}$ to $q_D^{\text{min}}$ in the pruned DFA. 
} 
Based on the selected symbol $\sigma$, we select locations that if every robot visits, then $\sigma$ is generated and transition from $q_D^{\text{next}}$ to $q_D^{\text{min}}$ is achieved.
To this end, first we select the atomic proposition that appears in $\sigma$ and is associated with robot $j$. Note that by definition of the feasible transitions and by construction of the pruned DFA, there is only one (if any) atomic proposition in $\sigma$ related to robot $j$. With slight abuse of notation, we denote this atomic proposition by $\pi_j\in\mathcal{AP}$.
%
For instance, the atomic proposition $\pi_j$ corresponding to the word $\sigma=\pi_p(\bbp,\ccalM,\{j, r_1,\delta_1,c_1\})\pi_p(\bbp,\ccalM,\{z,r_2,\delta_2,c_2\})$ is $\pi_j=\pi_p(\bbp,\ccalM,\{j,r_1,\delta_1,c_1\})$. Also, observe that there is no atomic proposition $\pi_s$ associated with any robot $s\neq j,z$.  
Next, given $\pi_j$, we select a landmark that if robot $j$ approaches/localizes, then $\pi_j$ may be satisfied. The selected landmark is denoted by $L_j$ and is used to design $f_{\text{new},i}(\bbu_i|k)$ so that control inputs that drive robot $j$ toward $L_j$ are selected more often than others. Note that for robots $j$ for which there is no predicate $\pi_j$ associated with them in $\sigma$, we have that $L_j=\varnothing$. In other words, the location of such robots $j$ does not play any role in generating the word $\sigma$. For instance, following the previous example, $L_j$ is selected to be the landmark that has the highest probability of having the label $c_1$, based on the discrete distribution $\bbd$, while for all other robots $s\neq j,z$ it holds $L_s=\varnothing$.

%
%

Given a landmark $L_j$, we construct the mass function $f_{\text{new},j}(\bbu_j|k)$ from which we sample a control input $\bbu_j$ as follows: 
%
\begin{equation}\label{eq:fnew}
f_{\text{new},j}(\bbu_j|k)=\left\{
                \begin{array}{ll}
                \frac{1}{|\ccalU_j|},~\mbox{if}~L_j=\varnothing,\\
                 p_{\text{new}},\mbox{if}~(L_j\neq\varnothing)\wedge
                   (\bbu_j=\bbu_j^*)\wedge (d_j\geq R_j)\\
                  (1-p_{\text{new}})\frac{1}{|\ccalU_j\setminus\set{\bbu_j^*}|},\mbox{if}~(L_j\neq\varnothing)\\ 
                  ~~~~~~~~~~~~~\wedge(\bbu_j\neq \bbu_j^*)\wedge (d_j\leq R_j),
                \end{array}
              \right.
\end{equation}
\normalsize
where (i) $d_j$ denotes the \textit{geodesic} distance between robot $j$ and the estimated position of landmark $L_j$, i.e., $d_j=\left\lVert\hat{\bbx}_{L_j}(t+1) - \bbp_j(t+1) \right\rVert_g$ \cite{kantaros2016global}, (ii) $R_j$ denotes the sensing range of robot $j$, and (iii) $\bbu_j^{*}\in\ccalU_j$ is the control input that minimizes the geodesic distance between $\bbp_j(t+1)$ and $\hat{\bbx}_i(t+1)$, i.e., $\bbu_j^*=\argmin_{\bbu_j\in\ccalU_j} \left\lVert \hat{\bbx}_i(t+1) - \bbp_j(t+1) \right\rVert_g,$
where $\left\lVert \cdot \right\rVert_g$ denotes the geodesic norm/distance. In words, in the first case of \eqref{eq:fnew}, a control input $\bbu_j$ is drawn from a uniform distribution if $L_j=\varnothing$. In the second case, the input $\bbu_j^{*}\in\ccalU_j$ is selected, if the (predicted) distance between robot $j$ and the assigned landmark $L_j$ is greater than the sensing range. In the third case, if the assigned landmark is  within the sensing range, random inputs are selected to get different views of the assigned landmark, which will contribute to decreasing the uncertainty of the assigned landmark. 
 
Finally, to compute the geodesic distance between $\bbp_j(t+1)$  and the estimated location of $L_j$, we first treat as `virtual/temporal' obstacles all landmarks that if visited, the Boolean condition under which transition from $q_D^{\text{next}}$ to $q_D^{\text{min}}$ is enabled, is false. For example, assume that transition from  $q_D^{\text{next}}$ to $q_D^{\text{min}}$ is possible if the following Boolean condition is true 
$\pi_p(\bbp,\ccalM,\{j, r_1,\delta_1,c_1\})\wedge\pi_p(\bbp,\ccalM,\{z,r_2,\delta_2,c_2\})\wedge[\neg \pi_p(\bbp,\ccalM,\{j, \ell_1, r_1,\delta_1\})]$
which is satisfied by the following symbol $\sigma=\pi_p(\bbp,\ccalM,\{j, r_1,\delta_1,c_1\})\pi_p(\bbp,\ccalM,\{z,r_2,\delta_2,c_2\})$ and is violated if $\pi_p(\bbp,\ccalM,\{j, \ell_1, r_1,\delta_1\}$ is observed, i.e., if robot $j$ approaches the landmark $\ell_1$ as per the probabilistic requirements captured by $\pi_p(\bbp,\ccalM,\{j, \ell_1, r_1,\delta_1\}$. For such `virtual' obstacles, we first compute a $\epsilon$-confidence interval around the mean position of landmark $\ell_1$, for some $\epsilon>0$ (e.g., $\epsilon=90\%$). This $\epsilon$-confidence interval is then treated as a (known) physical obstacle in the workspace. Then, we compute the \textit{geodesic} distance between $\bbp_j(t+1)$ and  such virtual obstacles. 
Essentially, this empirical approach minimizes the probability of generating an observation that violates the assigned task specification. Observe that the `virtual' obstacles are different for each robot while they may also depend on the global state $\bbp(t)$ (and not solely on individual robot states $\bbp_j(t)$). For instance, consider the Boolean condition 
$\pi_p(\bbp,\ccalM,\{j, r_1,\delta_1,c_1\})\wedge\pi_p(\bbp,\ccalM,\{z,r_2,\delta_2,c_2\})\wedge\{[\neg \pi_p(\bbp,\ccalM,\{j, \ell_1, r_1,\delta_1\})] \wedge [\neg \pi_p(\bbp,\ccalM,\{z, \ell_2, r_2,\delta_2\})]\}$
%
This is violated only if \textit{both} robots $j$ and $z$ are sufficiently close to landmarks $\ell_1$ and $\ell_2$, respectively. Such `virtual' obstacles are ignored, i.e., in this case, $\ell_1$ is not treated as an obstacle for robot $j$.

\bibliographystyle{IEEEtran}
\bibliography{YK_bib.bib}

\end{document}